\documentclass[11pt,a4paper]{article}
\usepackage[top=20mm, left=25mm, right=25mm, bottom=20mm]{geometry}
\usepackage{graphicx} 
\usepackage[colorlinks = true]{hyperref}
\usepackage[utf8]{inputenc}
\usepackage{xcolor}
\usepackage{bm}
\usepackage{amsmath}
\usepackage{amsfonts}
\usepackage{amssymb}
\usepackage{amsthm}

\usepackage{tikz}
\usetikzlibrary{shapes,arrows,positioning,calc}
\usepackage{url}   

\usepackage{svg} 

\newtheorem{theorem}{Theorem}[section]
\newtheorem{lemma}[theorem]{Lemma}
\newtheorem{corollary}[theorem]{Corollary}

\theoremstyle{definition}

\usepackage{enumitem}
\usepackage{mathtools}

\usepackage{booktabs}
\newtheorem{observation}{Observation}

\usepackage{algorithm}
\usepackage{algpseudocode}
\newcommand{\bx}{\mathbf{x}}
\newcommand{\by}{\mathbf{y}}
\newcommand{\bo}{\mathbf{o}}
\newcommand{\ba}{\mathbf{a}}
\newcommand{\bz}{\mathbf{z}}
\newcommand{\be}{\mathbf{e}}
\newcommand{\bw}{\mathbf{w}}
\newcommand{\bv}{\mathbf{v}}
\newcommand{\bu}{\mathbf{u}}
\newcommand{\bp}{\mathbf{p}}
\newcommand{\bq}{\mathbf{q}}

\newcommand{\bone}{\mathbf{1}}
\newcommand{\bzero}{\mathbf{0}}
\usepackage{thmtools,thm-restate}

\title{Stronger Approximation Guarantees for Non-Monotone $\gamma$-Weakly DR-Submodular Maximization}

\author{
  Hareshkumar Jadav\textsuperscript{1},
  Ranveer Singh\textsuperscript{1},
  Vaneet Aggarwal\textsuperscript{2}\thanks{Emails: \href{mailto:hareshjadavcse@gmail.com}{hareshjadavcse@gmail.com}, \href{mailto:ranveer@iiti.ac.in}{ranveer@iiti.ac.in}, \href{mailto:vaneet@purdue.edu}{vaneet@purdue.edu}. This work will be presented in part in International Conference on Autonomous Agents and Multiagent Systems (AAMAS), May 2026.}\\
  \textsuperscript{1}IIT Indore, India, 
  \textsuperscript{2}Purdue University, USA
}

\date{}

\begin{document}

\maketitle

\begin{abstract}
    Maximizing submodular objectives under constraints is a fundamental problem in machine learning and optimization. We study the maximization of a nonnegative, non-monotone $\gamma$-weakly DR-submodular function over a down-closed convex body. Our main result is an approximation algorithm whose guarantee depends smoothly on \(\gamma\); in particular, when \(\gamma=1\) (the DR-submodular case) our bound recovers the \(0.401\) approximation factor, while for \(\gamma<1\) the guarantee degrades gracefully and, it improves upon previously reported bounds for \(\gamma\)-weakly DR-submodular maximization under the same constraints. Our approach combines a Frank–Wolfe–guided continuous-greedy framework with a \(\gamma\)-aware double-greedy step, yielding a simple yet effective procedure for handling non-monotonicity. This results in state-of-the-art guarantees for non-monotone \(\gamma\)-weakly DR-submodular maximization over down-closed convex bodies.
\end{abstract}


\section{Introduction}

Submodular maximization under various constraints is a central problem in optimization and theoretical computer science. Foundational work established much of the modern toolkit and sparked a rich line of research~\cite{conforti1984submodular,fisher2009analysis,korte1978analysis,nemhauser1978best,nemhauser1978analysis}. Informally, a set function is submodular if marginal gains decrease as the chosen set grows (a discrete “diminishing returns’’ property). One key reason the problem remains important is its wide scope: many classic combinatorial tasks can be cast as maximizing a submodular objective, including \emph{Max-Cut}, the \emph{assignment problem}, \emph{facility location}, and \emph{max bisection}~\cite{haastad2001some,karp2009reducibility,chekuri2005polynomial,feige2006approximation,cornuejols1977uncapacitated,austrin2016better}.

On the continuous side, Diminishing-Returns (DR) submodular models support Maximum A Posteriori (MAP) inference for determinantal point processes~\cite{Bian2017NeurIPS,Bian2019ICML}. Informally, a DR-submodular function is a continuous analogue of submodularity in which marginal gains along each coordinate decrease as the current point increases (a “diminishing returns’’ property in $\mathbb{R}^n$). We give the formal definition in the preliminaries. The problem also arise in online allocation and learning~\cite{chen2018online,Zhang2019NeurIPS}.

Classical results for continuous DR-submodular maximization are based on projection-free, first-order methods with provable guarantees. 
Foundational work developed the geometry and algorithms for continuous DR-submodular maximization~\cite{Bian2017NeurIPS}, as well as optimal algorithms for continuous non-monotone DR-submodular objectives~\cite{niazadeh2020jmlr}. 
This paradigm extends to online and stochastic information models~\cite{chen2018online,zhang2022stochastic}, and recent progress sharpens bounds and constraint handling via DR-based analyses~\cite{buchbinder2024constrained}. 
Within this line, the \emph{weakly} DR framework broadens modeling reach by relaxing diminishing returns to a factor \(\gamma\in(0,1]\): roughly, marginal gains are still decreasing, but only up to a controlled multiplicative slack $\gamma$. This provides unified algorithms and guarantees~\cite{hassani2017gradient,pedramfar2023unified,pedramfar2024unified,pedramfar2024linear,nie2024gradient}.
These developments motivate our focus on continuous DR and weakly DR objectives over down-closed convex bodies.

A range of first-order methods are known for continuous DR and weakly DR maximization over down-closed convex bodies. 
\emph{Projected gradient} methods are the most basic scheme: they take a step in the direction of the gradient and then project back to the feasible region \(P\)~\cite{hassani2017gradient,zhang2022stochastic}. 
\emph{Frank--Wolfe (projection-free)} methods, also known as conditional gradient methods, avoid projections by instead solving a linear subproblem \(\max_{y\in P} \langle y,\nabla F(x)\rangle\) and moving towards this direction; they leverage DR/weakly-DR restricted concavity to certify progress~\cite{jaggi2013fw,lacoste2015linearFW,Bian2017NeurIPS,niazadeh2020jmlr,pedramfar2023unified}. 
\emph{Double--greedy} style methods adapt the discrete bracketing idea to the continuous setting. They keep two solutions, a “lower’’ and an “upper’’ one, and repeatedly adjust their coordinates in opposite directions so that the two solutions move closer together~\cite{buchbinder2015tight,niazadeh2020jmlr,pedramfar2023unified}. 
These techniques have also been extended to online, bandit, and stochastic models~\cite{chen2018online,zhang2022stochastic,pedramfar2023unified}.

At a high level, our algorithm has two interacting phases. 
We combine a $\gamma$-aware Frank--Wolfe--guided measured continuous greedy step with a $\gamma$-aware double--greedy step. 
Intuitively, the Frank--Wolfe component drives global progress by following promising directions inside the convex body, while the double--greedy component resolves local conflicts between including or excluding mass on each coordinate. 
Together, these phases produce a single solution whose quality we certify via a combined performance guarantee.
In our method, we design a $\gamma$-aware \emph{Frank-Wolfe guided measured continuous Greedy ($\gamma$-FWG)} algorithm with $\gamma$-dependent thresholds and progress certificates, pair it with a $\gamma$-aware \emph{double--greedy}, and then \emph{optimize a convex mixture of their certificates} to obtain a performance curve $\Phi_\gamma$ that strictly improves the baseline \cite{Bian2017NeurIPS,PedramfarQuinnAggarwal2024} for all $\gamma\in(0,1)$ while matching the DR boundary at $\gamma=1$ \cite{buchbinder2024constrained}, yielding state-of-the-art $\gamma$-dependent guarantees.

 Beyond establishing new approximation guarantees and improving the best-known bounds, our approach introduces the following \emph{technical novelties}:

\begin{itemize}
    \item We design a novel, $\gamma$-aware Frank--Wolfe–guided measured continuous greedy algorithm for the non-monotone $\gamma$-weakly DR setting. Our method introduces $\gamma$-dependent threshold schedules and progress certificates that balance ascent along Frank--Wolfe directions with measured updates, while preserving feasibility and ensuring a monotone decay of the residual gap.

\item Weakly-DR functions behave asymmetrically. In the DR case, one has the `naive' inequality
\begin{equation}
F(\mathbf{x})\ \ge\ \frac{F(\mathbf{x} \vee \mathbf{y}) + F(\mathbf{x} \wedge \mathbf{y})}{2}.
\end{equation}
In contrast, in the $\gamma$-weakly DR setting, Lemma~3.1 yields
\begin{equation}
F(\mathbf{x})\ \ge\ \frac{\gamma^{2}\,F(\mathbf{x} \vee \mathbf{y}) + F(\mathbf{x} \wedge \mathbf{y})}{1+\gamma^{2}}.
\end{equation}
Here only the $F(\mathbf{x}\vee\mathbf{y})$ term is scaled by $\gamma^{2}$, so the two sides of the inequality are no longer treated symmetrically. Similar one-sided $\gamma$-dependence appears in Lemmas~2.1, 2.2, and several auxiliary results in the appendix. This loss of symmetry breaks the standard potential-based arguments used in classical DR analyses. Our approach therefore uses a case-based progress analysis that explicitly tracks one-sided marginal decay and introduces $\gamma$-aware thresholds that couple Frank--Wolfe steps with measured updates, allowing us to certify progress despite this asymmetry.
    
    \item In parallel, we adapt the classical double–greedy potential to a $\gamma$-weighted variant that explicitly balances asymmetric gains and losses, yielding tight progress guarantees across the weakly-DR regime.
\end{itemize}

\subsection{Our Contribution}
In this paper, we consider \emph{non-monotone} $\gamma$-weakly DR-submodular objectives over a down-closed convex body $P \subseteq [0,1]^n$ with $0 < \gamma \le 1$.
Informally, $\gamma$-weak DR means that marginal gains decrease as coordinates increase, but only up to a factor $\gamma \in (0,1]$ capturing objectives that exhibit \emph{partial}, rather than full, diminishing returns.
In this regime, the canonical approximation envelope is $\kappa(\gamma) = \gamma e^{-\gamma}$ (which recovers $e^{-1}$ at $\gamma = 1$) \cite{Bian2017NeurIPS,PedramfarQuinnAggarwal2024}.
In \cite{Bian2017NeurIPS}, this approximation guarantee is achieved with time complexity $\mathcal{O}(1/\varepsilon)$, where $\varepsilon > 0$ is an accuracy parameter, while \cite{PedramfarQuinnAggarwal2024} attains the same guarantee with running time $\mathcal{O}(1/\varepsilon^3)$.
In contrast, our algorithm achieves the $\Phi_\gamma$-approximation in time $\mathrm{Poly}(n, \delta^{-1})$.
Recently, Buchbinder and Feldman introduced a novel technique that yields a $0.401$-approximation for the (fully) DR-submodular case \cite{buchbinder2024constrained}; their algorithm also runs in time $\mathrm{Poly}(n, \delta^{-1})$.

\emph{This paper aims to close the gap between weakly and full DR.}
We develop $\gamma$-aware algorithms and analyses that (i) recover the classical DR constant at $\gamma=1$ \cite{buchbinder2024constrained} and (ii) strictly improve upon the baseline $\kappa(\gamma)=\gamma e^{-\gamma}$ throughout the weakly-DR regime   \cite{Bian2017NeurIPS,PedramfarQuinnAggarwal2024}.
Our approach combines a $\gamma$-aware Frank--Wolfe--guided measured continuous greedy subroutine with a $\gamma$-aware double--greedy, and then optimizes a convex mixture of their certificates.
The resulting guarantee $\Phi_\gamma$ is an \emph{objective value} determined by three tunable parameters $(\alpha,r,t_s)$—where $\alpha$ is the mixing weight between certificates and $(r,t_s)$ govern schedule/tuning in the FW-guided and double--greedy components—and we choose these to maximize $\Phi_\gamma$ at each given $\gamma$.
A formal statement of our main result appears as Theorem~\ref{thm:main}.
To contextualize our guarantees, Figure~\ref{fig:phi-gamma} plots the optimized curve $\Phi_\gamma$, and Table~\ref{tab:phi-params} reports representative values and the associated parameter choices $(\alpha,r,t_s)$.

In this direction, our \emph{key contributions} are summarized below—each item highlights a distinct component of our algorithmic framework and its guarantee.
\begin{itemize}
  \item 
  We present a \(\gamma\)-aware Frank--Wolfe--guided measured continuous greedy and a \(\gamma\)-aware double--greedy,
  each delivering explicit constant-factor guarantees for \(\gamma\)-weakly DR objectives over down-closed convex bodies.

  \item 
  We derive a parameter-optimized convex mixture of the two certificates producing a performance curve \(\Phi_\gamma\)
  that strictly improves the prior baseline \(\kappa(\gamma)=\gamma e^{-\gamma}\) for all \(\gamma\in(0,1)\) and
  matches the classical DR constant at \(\gamma=1\).

  \item 
  Our proofs are modular and avoid curvature assumptions; they recover the DR boundary as a special case
  and extend smoothly across the weakly-DR spectrum.

  \item 
  The methods use only first-order information and linear optimization over \(P\) (Frank--Wolfe oracles),
  making them projection-free and suitable for large-scale instances.
\end{itemize}

Non-monotone $\gamma$-weakly DR-submodular objectives naturally arise when full DR-submodular models, such as continuous budget allocation, DPP-based diversity objectives, and mean-field inference for probabilistic submodular models, are augmented with practical penalties including over-exposure costs, risk or variance regularization, and intensity constraints. These augmentations retain the weak DR structure while inducing non-monotonicity, which is exactly the regime targeted by our algorithm.

\begin{figure}[t]
  \centering
\includegraphics[width=0.77\linewidth]{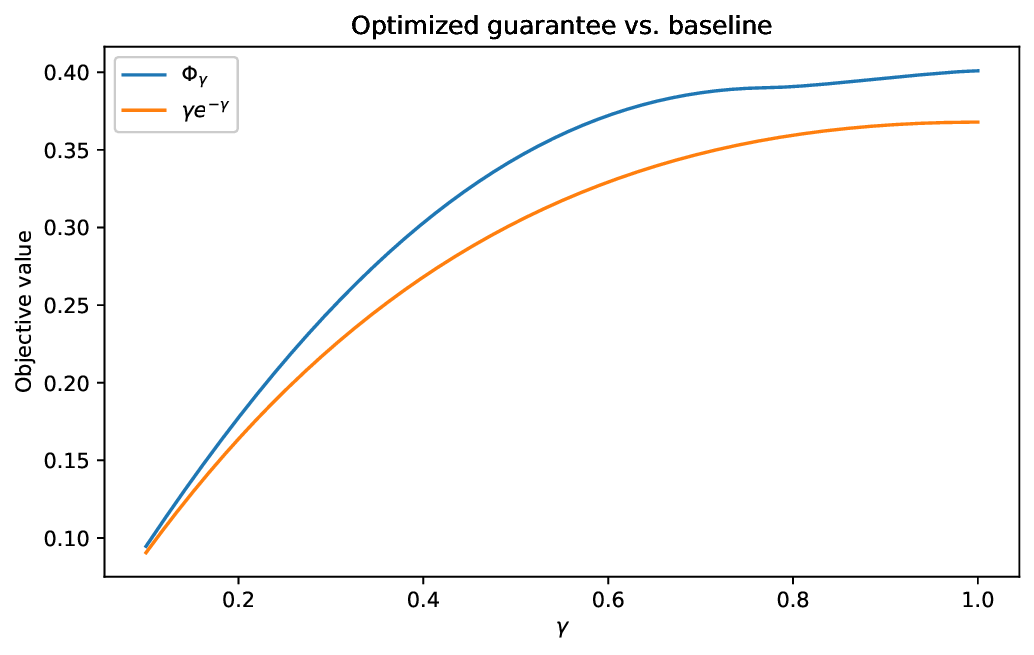}%
  \caption{Approximation guarantee versus weakly-DR parameter. 
  The horizontal axis is the weakly-DR parameter $\gamma\in(0,1]$ and the vertical axis is the approximation factor. 
  We plot our optimized guarantee $\Phi_\gamma$ (blue curve) alongside the non-monotone weakly-DR baseline $\kappa(\gamma)=\gamma e^{-\gamma}$ (orange curve). 
  Across the entire regime $\gamma\in(0,1)$, $\Phi_\gamma$ strictly exceeds $\kappa(\gamma)$, and at $\gamma=1$ (full DR) our curve reaches $0.401$, matching the current best bound. 
  Selected parameter choices $(\alpha,r,t_s)$ used to construct $\Phi_\gamma$ are reported in Table~\ref{tab:phi-params}.}
  \label{fig:phi-gamma}
\end{figure}

\begin{table}[t]
\centering
\begin{tabular}{@{}rrrrrr@{}}
\toprule
$\gamma$ & $\Phi_\gamma$ &  $\gamma e^{-\gamma}$ & $\alpha$ & $r$ & $t_s$ \\
\midrule
0.1 & 0.095 & 0.090 & 0.001 & 3.750 & 0.000 \\
0.2 & 0.178 & 0.164 & 0.000 & 3.083 & 0.000 \\
0.3 & 0.247 & 0.222 & 0.000 & 2.833 & 0.000 \\
0.4 & 0.303 & 0.268 & 0.000 & 2.667 & 0.000 \\
0.5 & 0.345 & 0.303 & 0.000 & 2.583 & 0.000 \\
0.6 & 0.372 & 0.329 & 0.000 & 2.417 & 0.000 \\
0.7 & 0.387 & 0.348 & 0.000 & 2.333 & 0.000 \\
0.8 & 0.391 & 0.359 & 0.054 & 2.250 & 0.075 \\
0.9 & 0.396 & 0.366 & 0.160 & 2.083 & 0.267 \\
1.0 & 0.401 & 0.368 & 0.197 & 2.220 & 0.368 \\
\bottomrule
\end{tabular}
\caption{Numerical comparison of our optimized guarantee $\Phi_\gamma$ against non-monotone weakly-DR baseline $\kappa(\gamma)=\gamma e^{-\gamma}$. Each row corresponds to a choice of the weakly-DR parameter $\gamma\in(0,1]$; the second and third columns report the achieved approximation factor $\Phi_\gamma$ and the baseline, respectively. The final three columns list representative internal parameters $(\alpha,r,t_s)$ used to construct $\Phi_\gamma$: $\alpha$ is the convex mixing weight between the two certificates, while $r$ and $t_s$ are schedule/tuning parameters of the $\gamma$ aware FW-guided measured continuous greedy and $\gamma$-aware double--greedy components (see the algorithm description). All values are rounded to three decimals.}
\label{tab:phi-params}
\end{table}

\subsection{Related Work}
For non-monotone, $\gamma$-weakly DR objectives over down-closed convex bodies, unified analyses establish the \emph{baseline envelope} $\kappa(\gamma)=\gamma e^{-\gamma}$ \cite{PedramfarQuinnAggarwal2024}. 
An earlier line derives the same envelope via continuous/“measured” greedy combined with projection-free (Frank--Wolfe) arguments adapted to one-sided weakly-DR gradients \cite{Bian2017NeurIPS}. 
A complementary \emph{stationary-point baseline} shows that any first-order stationary point obtained by projected or mirror ascent achieves $\gamma^2/(1+\gamma^2)$ of $\mathrm{OPT}$ (recovering $1/2$ at $\gamma=1$), with stochastic/online refinements using non-oblivious surrogates and unbiased gradients \cite{hassani2017gradient,chen2018online,zhang2022stochastic}.

When $\gamma=1$, the problem reduces to continuous non-monotone DR-submodular maximization over down-closed convex bodies.
This line begins with the multilinear/continuous-greedy framework \cite{chekuri2011nonmonotone} and the \emph{Measured Continuous Greedy} guarantee of $1/e\approx0.367$ \cite{FeldmanNaorSchwartzWard2011}, with subsequent improvements culminating in the current best constant $0.401$ \cite{buchbinder2024constrained}.
Hardness results still leave a notable gap for non-monotone objectives under common constraints (with no better than $\approx 0.478$ \cite{gharan2011submodular}), underscoring the importance of tighter algorithms at $\gamma=1$.
Our guarantees match this DR boundary while delivering strict improvements over $\kappa(\gamma)$ for every $\gamma\in(0,1)$.

\section{Preliminaries and Notation}\label{sec:prelim}
In this section, we introduce the basic notation, definitions, and assumptions used throughout the paper. We use boldface letters (e.g., \(\mathbf{x}, \mathbf{y}\)) to denote vectors in \(\mathbb{R}^n\), and write a vector as \(\mathbf{x} = (x_1, \cdots, x_n)\). The all-ones and all-zeros vectors are denoted by \(\mathbf{1}\) and \(\mathbf{0}\), respectively. We use \(\mathbf{e}_i\) to denote the \(i\)-th standard basis vector in \(\mathbb{R}^n\). Let \(N\) be the ground set with \(|N| = n\) elements. The discrete and continuous hypercubes are defined as
\[
\{0,1\}^n \;=\; \bigl\{\mathbf{x}\in\mathbb{R}^n : x_i\in\{0,1\}\ \forall i\bigr\},
\]
and
\[
[0,1]^n \;=\; \bigl\{\mathbf{x}\in\mathbb{R}^n : x_i\in[0,1]\ \forall i\bigr\},
\]
respectively. For a positive integer \(n\), we write \([n] := \{1,2,\hdots,n\}\).

For \(\mathbf{x},\mathbf{y}\in\mathbb{R}^n\), we use the componentwise order:  
\(\mathbf{x}\le \mathbf{y}\) if and only if \(x_i\le y_i\) for all \(i\) (and \(\mathbf{x}<\mathbf{y}\) if and only if \(x_i<y_i\) for all \(i\)).  
The \emph{join} and \emph{meet} are defined as
\[
\mathbf{x}\vee \mathbf{y} := (\max\{x_i,y_i\})_{i=1}^n,
\quad
\mathbf{x}\wedge \mathbf{y} := (\min\{x_i,y_i\})_{i=1}^n.
\]
We also use the elementwise (Hadamard) product \(\mathbf{x}\odot \mathbf{y}\in\mathbb{R}^n\), defined by
\((\mathbf{x}\odot \mathbf{y})_i := x_i\,y_i\) for each \(i\in[n]\), and the standard inner product
\(\langle \mathbf{x},\mathbf{y}\rangle := \sum_{i=1}^n x_i y_i\).  
For vectors with entries in \([0,1]\), the coordinatewise \emph{probabilistic sum} is
\begin{equation}
    \mathbf{x} \oplus \mathbf{y} \;:=\; \mathbf{1}-\bigl(\mathbf{1}-\mathbf{x}\bigr)\odot\bigl(\mathbf{1}-\mathbf{y}\bigr).
\end{equation}
For vectors \(\mathbf{x}^{(1)},\hdots,\mathbf{x}^{(m)}\in[0,1]^n\), we write
\begin{align}
\bigoplus_{j=1}^m \mathbf{x}^{(j)}
&= \mathbf{1}-\bigodot_{j=1}^m\bigl(\mathbf{1}-\mathbf{x}^{(j)}\bigr) = \mathbf{1} - \big((\mathbf{1}-\mathbf{x}^{(1)})\odot\cdots\odot (\mathbf{1}-\mathbf{x}^{(m)})\big).
\label{eq:oplus-odot}
\end{align}

The operators \(\odot\) and \(\oplus\) bind more tightly than vector addition or subtraction, so
\(\mathbf{x} + \mathbf{y}\odot \mathbf{z}\) means \(\mathbf{x} + (\mathbf{y}\odot \mathbf{z})\).  
When a scalar function is applied to a vector, it is interpreted elementwise; for instance, for \(\mathbf{x}\in[0,1]^n\), the vector \(e^{\mathbf{x}}\) has entries \(\bigl(e^{\mathbf{x}}\bigr)_i = e^{x_i}\).

A set \(P\subseteq\mathbb{R}^n\) is \emph{convex} if
\(\lambda \mathbf{x}+(1-\lambda)\mathbf{y}\in P\) for all \(\mathbf{x},\mathbf{y}\in P\) and \(\lambda\in[0,1]\).
A \emph{convex body} is a compact, convex set with nonempty interior.
A polytope \(P\subseteq[0,1]^n\) is \emph{down-closed} if \(\mathbf{y}\in P\) implies
\(\mathbf{x}\in P\) for every \(\mathbf{x}\in\mathbb{R}^n\) with \(\mathbf{0}\le \mathbf{x}\le \mathbf{y}\).
We say that \(P\) is \emph{solvable} if linear optimization over \(P\) can be performed in polynomial time.
The (Euclidean) \emph{diameter} of a set \(P\subseteq\mathbb{R}^n\) is
\[
D := \sup\{\ \|\mathbf{x}-\mathbf{y}\|_2 : \mathbf{x},\mathbf{y}\in P\ \}.
\]

A nonnegative set function \(f:\{0,1\}^n\to\mathbb{R}_{\ge 0}\) is
\emph{submodular} if, for all \(\mathbf{x},\mathbf{y}\in\{0,1\}^n\) with \(\mathbf{x}\le \mathbf{y}\) and for all
\(\mathbf{a}\in\{0,1\}^n\),
\[
f(\mathbf{x}\vee \mathbf{a})-f(\mathbf{x})\;\ge\;f(\mathbf{y}\vee \mathbf{a})-f(\mathbf{y}).
\]
In the continuous case, a nonnegative function \(F:[0,1]^n\to\mathbb{R}_{\ge 0}\) is
\emph{diminishing-returns (DR) submodular} if, for all \(\mathbf{x},\mathbf{y}\in[0,1]^n\)
with \(\mathbf{x}\le \mathbf{y}\), any coordinate \(i\in[n]\), and any \(c>0\) such that
\(\mathbf{x}+c\,\mathbf{e}_i,\ \mathbf{y}+c\,\mathbf{e}_i\in[0,1]^n\),
\[
F(\mathbf{x}+c\,\mathbf{e}_i)-F(\mathbf{x})\;\ge\;F(\mathbf{y}+c\,\mathbf{e}_i)-F(\mathbf{y}).
\]
If \(F\) is differentiable, this is equivalent to \(\nabla F(\mathbf{x}) \ge \nabla F(\mathbf{y})\) for all \(\mathbf{x}\le \mathbf{y}\)~\cite{Bian2019ICML}.

A nonnegative
function \(F:[0,1]^n\to\mathbb{R}_{\ge 0}\) is \emph{\(\gamma\)-weakly DR-submodular}
if, for all \(\mathbf{x},\mathbf{y}\in[0,1]^n\) with \(\mathbf{x}\le \mathbf{y}\), any \(i\in[n]\), and any \(c>0\)
with \(\mathbf{x}+c\,\mathbf{e}_i,\ \mathbf{y}+c\,\mathbf{e}_i\in[0,1]^n\),
\begin{equation}
F(\mathbf{x}+c\,\mathbf{e}_i)-F(\mathbf{x})\;\ge\;\gamma\big(F(\mathbf{y}+c\,\mathbf{e}_i)-F(\mathbf{y})\big).
\end{equation}
When \(F\) is differentiable, this is equivalent to
\(\nabla F(\mathbf{x}) \ge \gamma\,\nabla F(\mathbf{y})\) for all \(\mathbf{x}\le \mathbf{y}\).
This condition holds for some \(\gamma>1\) if and only if \(F\) is constant (and a constant \(F\) satisfies it for any \(\gamma\));
it holds for some \(\gamma\le 0\) exactly when \(F\) is coordinate-wise monotone. Hence, we focus on the nontrivial range \(0<\gamma\le 1\). 

A differentiable function \(F:P \to \mathbb{R}\) is \emph{\(L\)-smooth} if, for all
\(\mathbf{x},\mathbf{y} \in P\), it satisfies
\begin{equation}
    \|\nabla F(\mathbf{x})-\nabla F(\mathbf{y})\|_2 \le L\,\|\mathbf{x}-\mathbf{y}\|_2.
\end{equation}

Now we discuss some properties of \(\gamma\)-weakly DR-submodular functions. 
In the following three lemmas, let \(F:[0,1]^n\to\mathbb{R}_{\ge 0}\) be differentiable and \(\gamma\)-weakly DR-submodular.

\begin{lemma} \label{lemma:simpe2}
For all \(\bx,\by\in[0,1]^n\) and \(\lambda\in[0,1]\), the following hold:
\begin{enumerate}
\item If \(\bx\le \by\), then
\begin{equation}\label{eq:lemma1-convex-combo}
F\bigl(\lambda \bx+(1-\lambda)\by\bigr)
\ \ge \
\frac{\lambda\,F(\bx)+\gamma^{2}(1-\lambda)\,F(\by)}{\lambda+\gamma^{2}(1-\lambda)}.
\end{equation}
Equivalently,
\begin{equation}\label{eq:lemma1-convex-combo-alt}
F\bigl((1-\lambda)\bx+\lambda \by\bigr)
\ \ge\
\frac{(1-\lambda)\,F(\bx)+\gamma^{2}\lambda\,F(\by)}{(1-\lambda)+\gamma^{2}\lambda}.
\end{equation}

\item If \(\bx+\by\in[0,1]^n\), then
\begin{equation}\label{eq:lemma1-increment}
F(\bx+\lambda \by)-F(\bx)
\ \ge\
\frac{\gamma^{2}\lambda}{\,1-\lambda+\gamma^{2}\lambda\,}\,\bigl(F(\bx+\by)-F(\bx)\bigr).
\end{equation}
\end{enumerate}
\end{lemma}

\begin{lemma}\label{lem:grad-ineq}
For every \(\bx,\by \in [0,1]^n\), the following inequalities hold:
\begin{enumerate}
    \item If \(\bx+\by \le \bone\), then
    \begin{equation}\label{eq:grad-ineq-plus}
        \langle \nabla F(\bx),\,\by \rangle \;\ge\; \gamma\big(F(\bx+\by) - F(\bx)\big).
    \end{equation}
    \item If \(\bx-\by \ge \bzero\), then
    \begin{equation}\label{eq:grad-ineq-minus}
        \langle \nabla F(\bx),\,\by \rangle \;\le\; \frac{1}{\gamma}\big(F(\bx) - F(\bx-\by)\big).
    \end{equation}
\end{enumerate}
\end{lemma}

\begin{lemma}\label{lem:weakDR-closure}
For any fixed \(\by\in[0,1]^n\), define
\begin{equation}\label{eq:closure-defs}
G_{\oplus}(\bx)\ :=\ F(\bx\oplus \by)
\quad\text{and}\quad
G_{\odot}(\bx)\ :=\ F(\bx\odot \by).
\end{equation}
Then both \(G_{\oplus}\) and \(G_{\odot}\) are nonnegative and
\(\gamma\)–weakly DR-submodular.
\end{lemma}

Proofs of these lemmas are provided in Appendix~\ref{Proof_A}. 
These statements generalize prior results from \cite{zhang2022stochastic,hassani2017gradient} to the $\gamma$-weakly DR setting and hold for all $\gamma\in(0,1]$; in particular, they coincide exactly with the classical DR statements when $\gamma=1$.

\section{Supporting Results}

In this section, we generalize two standard algorithms to the $\gamma$–weakly DR setting and present their analyses. First, we prove a $\gamma$-weighted Frank–Wolfe certificate over solvable convex bodies (Theorem~\ref{thm:weak-dr-smooth}). Second, we develop a $\gamma$-aware Double–Greedy algorithm and establish an unbalanced lower bound that interpolates smoothly in $\gamma$ (Theorem~\ref{thm:unbalanced-weakDR}).

We generalize the classical DR framework to the $\gamma$–weakly DR setting and obtain guarantees that \emph{continuously interpolate} between weakly and full DR. Our first ingredient (Lemma~\ref{lemma_2}) shows that the local‐optimality certificate $\langle \by-\bx,\nabla F(\bx)\rangle\le 0$ can be translated, under $\gamma$–weakly DR‐submodularity, into a \emph{$\gamma$–weighted} value comparison between $F(\bx)$ and the join/meet values $F(\bx\vee\by)$ and $F(\bx\wedge\by)$, recovering the classical $\tfrac12\!\big(F(\bx\vee\by)+F(\bx\wedge\by)\big)$ bound at $\gamma=1$ (cf.\ \cite{chekuri2011nonmonotone}; see also \cite{Bian2017NeurIPS,pmlr-v54-bian17a,chen2018online}). Our second ingredient (Theorem~\ref{thm:weak-dr-smooth}) “globalizes’’ the comparison over a solvable convex body $P$: a Frank–Wolfe–type routine yields a \emph{uniform} first‐order certificate against every $\by\in P$ without requiring curvature information, a device standard in recent continuous submodular solvers (e.g., \cite{buchbinder2024constrained}) and aligned with unified weakly‐DR analyses (e.g., \cite{pedramfar2023unified}). Combining this certificate with the weakly‐DR property gives a value bound that degrades smoothly with $\gamma$ and exactly matches the DR case at $\gamma=1$; this is formalized in Lemma~\ref{lemma_2} and Theorem~\ref{thm:weak-dr-smooth}, and proof of these results are given in Appendix~\ref{Frank}

\begin{lemma}\label{lemma_2}
Let $F:[0,1]^n\to\mathbb{R}_{\ge0}$ be $\gamma$–weakly DR–submodular.
If $\bx$ is a local optimum with respect to a vector $\by$, i.e.,
\begin{equation}\label{eq:local-opt}
\langle \by-\bx,\nabla F(\bx)\rangle\le 0,
\end{equation}
then
\begin{equation}\label{eq:lemma2-bound}
F(\bx)\ \ge\ \frac{\gamma^{2}\,F(\bx \vee \by) + F(\bx \wedge \by)}{\,1+\gamma^{2}\,}.
\end{equation}
\end{lemma}

\begin{theorem}\label{thm:weak-dr-smooth}
Let $F:[0,1]^n\!\to\mathbb{R}_{\ge0}$ be a nonnegative, $L$-smooth function that is $\gamma$–weakly DR–submodular. 
Let $P\subseteq[0,1]^n$ be a solvable convex body of diameter $D$, and let $\delta\in(0,1)$. 
There is a polynomial-time algorithm that outputs $\bx\in P$ such that, for every $\by\in P$,
\begin{equation}\label{eq:thm-smooth-bound}
\begin{aligned}
F(\bx) \!\ge \!\frac{\gamma^{2}F(\bx\vee \by)+F(\bx\wedge \by)}{1+\gamma^{2}}
\!-\!\frac{\delta\,\gamma}{1+\gamma^{2}} \left[\max_{\bz\in P}F(\bz) + \frac{L D^{2}}{2}\right].
\end{aligned}
\end{equation}
\end{theorem}

In the continuous DR submodular domain, Double–Greedy–type procedures were extended to \emph{DR–submodular} objectives on $[0,1]^n$ (the $\gamma{=}1$ case) by several works \cite{Bian2019ICML,chen2018online,niazadeh2020jmlr}, but these analyses still stated a uniform $1/2$ bound and predated the unbalanced refinement. 
We generalize this line to the \emph{$\gamma$–weakly DR–submodular} regime ($0<\gamma\le1$), leveraging the weakly-DR structure introduced for continuous submodular maximization \cite{pmlr-v54-bian17a,Bian2017NeurIPS} and the recent unified weakly-DR perspective \cite{pedramfar2023unified}.

Our algorithm retains the Double–Greedy structure but augments it with \emph{$\gamma$-aware} smoothing/thresholding, ensuring that it handles any given $\gamma\in(0,1]$ robustly; the resulting guarantee {interpolates continuously} in $\gamma$ and collapses to the classical DR bound at $\gamma=1$. The corresponding lower-bound guarantee is stated in
Theorem~\ref{thm:unbalanced-weakDR}, and a detailed proof is provided in
Appendix~\ref{Double-Greedy}.

\begin{theorem}\label{thm:unbalanced-weakDR}
Let $F:[0,1]^n\to\mathbb{R}_{\ge 0}$ be nonnegative and $\gamma$–weakly DR-submodular for some $\gamma\in(0,1]$, and fix a parameter $\varepsilon\in(0,1)$. 
There exists a polynomial-time algorithm that outputs $\bx\in[0,1]^n$ such that
\begin{equation}\label{eq:unbalanced-bound}
F(\bx)\ \ge\
\max_{r \ge 0}\ 
\frac{\bigl(2\gamma^{3/2}-4\varepsilon\gamma^{9/2}\bigr)\,r\,F(\bo)\;+\;F(\mathbf{0})\;+\;r^2 F(\mathbf{1})}
{\,r^2\;+\;2\gamma^{3/2}r\;+\;1\,}.
\end{equation}
\end{theorem}

When $\gamma=1$ and $r=1$ this recovers the canonical $1/2$ approximation, while for many instances one can choose $r\neq1$ to obtain a strictly better guarantee, in direct analogy to the unbalanced bounds for set functions \cite{mualem2022partial,qi2022isaac}.

The unbalanced Double–Greedy guarantee extends directly to axis-aligned boxes. 
Fix an upper bound \(\bx\in[0,1]^n\) and consider maximizing \(F\) over the box \([\mathbf{0},\bx]\).
Define \(G:[0,1]^n\to\mathbb{R}_{\ge0}\) by \(G(\ba):=F(\bx\odot \ba)\).
By Lemma~\ref{lem:weakDR-closure}, \(G\) remains nonnegative and \(\gamma\)–weakly DR-submodular, so
Theorem~\ref{thm:unbalanced-weakDR} applies to \(G\). 
Translating the output back via \(\by:=\bx\odot\ba'\le \bx\) yields the following corollary. We use following corollary as Box Maximization in Algorithm~\ref{alg:main}.

\begin{corollary}[Box maximization] \label{cor:box-weakDR}
Let $F:[0,1]^n\to\mathbb{R}_{\ge 0}$ be nonnegative and $\gamma$–weakly DR-submodular for some
$\gamma\in(0,1]$, let $\bx\in[0,1]^n$, and fix $\varepsilon\in(0,1)$. There exists a polynomial-time
algorithm that outputs a vector $\by\in[0,1]^n$ with $\by\le \bx$ such that, for every fixed
$\bo\in[0,1]^n$,
\begin{equation}\label{eq_123}
    F(\by) \ge
\max_{r \ge 0}
\frac{\bigl(2\gamma^{3/2}-4\varepsilon\,\gamma^{9/2}\bigr)\,r\,F(\bx\odot \bo)+F(\mathbf{0})+r^2F(\bx)}
{\,r^2\;+\;2\gamma^{3/2}r\;+\;1\,}.
\end{equation}
\end{corollary}

\begin{proof}
Fix $\bx\in[0,1]^n$ and define the restricted objective
\begin{equation}\label{eq:box-defG}
G(\ba)\ :=\ F(\bx\odot \ba)\qquad\text{for all }\ba\in[0,1]^n.
\end{equation}
By Lemma~\ref{lem:weakDR-closure}, $G$ is nonnegative and $\gamma$–weakly DR-submodular.  
Applying Theorem~\ref{thm:unbalanced-weakDR} to $G$ (with the same $\varepsilon\in(0,1)$) yields some $\ba'\in[0,1]^n$ such that
\begin{equation}\label{eq:box-thmG}
G(\ba') \ge
\max_{r\ge 0}
\frac{\bigl(2\gamma^{3/2}-4\varepsilon\,\gamma^{9/2}\bigr)\,r\,G(\bo)+G(\mathbf{0})+r^2\,G(\mathbf{1})}
{\,r^2\;+\;2\gamma^{3/2}r\;+\;1\,}.
\end{equation}
From \eqref{eq:box-defG}, we have
\begin{equation*}
G(\bo)=F(\bx\odot \bo),\qquad
G(\mathbf{0})=F(\mathbf{0}),\qquad
G(\mathbf{1})=F(\bx).
\end{equation*}
Substituting these identities into \eqref{eq:box-thmG} gives
\begin{equation}\label{eq:box-subbed}
G(\ba') \ge
\max_{r\ge 0}
\frac{\bigl(2\gamma^{3/2}-4\varepsilon\ \gamma^{9/2}\bigr)\,r\,F(\bx\odot \bo)+F(\mathbf{0})+r^2\,F(\bx)}
{\,r^2\;+\;2\gamma^{3/2}r\;+\;1\,}.
\end{equation}
Define $\by:=\bx\odot \ba'$. Then $\by\le \bx$ coordinate-wise and, by \eqref{eq:box-defG},
\begin{equation}\label{eq:box-FeqG}
F(\by)\ =\ F(\bx\odot \ba')\ =\ G(\ba').
\end{equation}
Combining \eqref{eq:box-subbed} and \eqref{eq:box-FeqG} yields the claimed bound \eqref{eq_123}.
\end{proof}

\section{Main Algorithm and Results}
In this section, we present our main result together with the algorithm that achieves it.
Our approach is recursive and hinges on a core subroutine invoked at every level of recursion:
the \emph{$\gamma$–Frank--Wolfe Guided Measured Continuous Greedy} (\textbf{$\gamma$-FWG}).
We describe \textbf{$\gamma$-FWG} and establish its guarantees in Section~\ref{sec:main_com}.
Building on this component, Section~\ref{sec:main_algo} introduces the full recursive algorithm,
and Section~\ref{sec:last_proof} proves our main theorem.

\subsection{$\gamma$-FWG Algorithm}\label{sec:main_com}

We develop a measured continuous greedy method, steered by Frank--Wolfe directions and explicitly tuned by $\gamma$, to operate in the $\gamma$–weakly DR setting.
The algorithm is explicitly $\gamma$-parameterized, so it works for any $\gamma\in(0,1]$ and reduces to the classical DR case when $\gamma=1$ \cite{buchbinder2024constrained}.
For clarity, in the description of Algorithm~\ref{alg:fw-guided-mcg} we assume that $\delta^{-1}$ is an
integer and that $\delta\le \varepsilon$ (which lets us set $m=\delta^{-1}$). If these conditions
do not hold, we reduce $\delta$ to $1/\big\lceil 1/\min\{\delta,\varepsilon\}\big\rceil$ without
affecting the analysis. Also we define $\beta := \frac{\gamma^{2}\,\delta}{\,1-\delta+\gamma^{2}\delta\,}$. Since the algorithm does not know the values $F(\bo)$, $F(\bz\odot \bo)$, and
$F(\bz\oplus \bo)$, we rely on the following guessing lemma; its proof uses standard guessing
arguments and proof of this lemma is given in Appendix~\ref{sec:fwg_proof}.

\begin{lemma}\label{lem:guessing-triples}
Let $F:[0,1]^n\to\mathbb{R}_{\ge 0}$ be nonnegative and $\gamma$-weakly DR-submodular for some $0<\gamma\le 1$, and let $P\subseteq[0,1]^n$ be down-closed.
There exists a constant-size (depending only on $\varepsilon$ and $\gamma$) set of triples $\mathcal{G} \subseteq \mathbb{R}_{\ge 0}^3$
such that $\mathcal{G}$ contains a triple $(g,g_\odot,g_\oplus)$ with
\begin{subequations}\label{eq:triple-bounds1}
\begin{align}
(1-\varepsilon)\,F(\bo) &\le g \le F(\bo), 
\label{eq:triple-bounds-g1}\\
F(\bz\odot \bo)-\varepsilon\,g &\le g_\odot \le F(\bz\odot \bo), 
\label{eq:triple-bounds-godot1}\\
F(\bz\oplus \bo)-\varepsilon\,g &\le g_\oplus \le F(\bz\oplus \bo).
\label{eq:triple-bounds-goplus1}
\end{align}
\end{subequations}
\end{lemma}

Therefore, by trying all triples in $\mathcal{G}$, we can act as if
Algorithm~\ref{alg:fw-guided-mcg} is given valid surrogates $g$, $g_{\odot}$, and $g_{\oplus}$ that meet
these bounds. For convenience, we first define the threshold functions.
For $i\in\{0,1,\hdots,\delta^{-1}-1\}$, define
\begin{align}
v_1(i)
&:= \Bigl[(1-\beta)^{\,i}+\tfrac{1-(1-\beta)^{\,i}-2\varepsilon}{\gamma}\Bigr]\,g
   -\tfrac{1}{\gamma}\,g_{\odot} -\tfrac{1-(1-\beta)^{\,i}}{\gamma}\,g_{\oplus}
\label{eq:v1-def}\\[4pt]
v_2(i)
&:= (1-\beta)^{\,i}\Big[
   \Bigl(\tfrac{(1-\beta)^{-i_s}}{\gamma}-\Bigl(1+\tfrac{3}{\gamma}\Bigr)\varepsilon+1-\tfrac{1}{\gamma}\Bigr)\,g
-\Bigl(\tfrac{(1-\beta)^{-i_s}}{\gamma}-\tfrac{1}{\gamma}-\beta\,(i-i_s)\Bigr)\,g_{\oplus}
   \Big]
\label{eq:v2-def-uniq}
\end{align}

\begin{algorithm}[t]
\caption{\ $\mathrm{\gamma\mbox{-}FWG}(F, P, \mathbf{z},\gamma, t_s, \varepsilon, \delta)$}
\label{alg:fw-guided-mcg}
\begin{algorithmic}[1]
\State \textbf{Input:} nonnegative $L$-smooth $\gamma$-weakly DR-submodular $F:[0,1]^n\!\to\mathbb{R}_{\ge0}$; meta-solvable down-closed $P\subseteq[0,1]^n$; $\mathbf{z}\in P$; $\gamma \in (0,1]$; parameters $t_s\in(0,1)$, $\varepsilon\in(0,1/2)$, $\delta\in(0,1)$.
\State $i_s \gets \lceil t_s/\delta \rceil$
\For{$i = 0$ \textbf{to} $\delta^{-1}-1$}
  \State \(
  v{(i)} \;:=\;
  \begin{cases}
    v_1(i), & \text{if } i \le i_s,\\
    v_2(i), & \text{if } i \ge i_s.
  \end{cases}
  \)
  \State $\mathbf{z}(i) \gets
  \begin{cases}
  \mathbf{z}, & \text{if } i < i_s,\\
  \mathbf{0}, & \text{if } i \ge i_s
  \end{cases}$
\EndFor
\State $\mathbf{y}(0) \gets \mathbf{0}$
\For{$i = 1$ \textbf{to} $\delta^{-1}$}
  \State $\mathbf{w}(i) \gets \bigl(\mathbf{1} - \mathbf{y}(i-1) - \mathbf{z}(i-1)\bigr) \odot \nabla F\bigl(\mathbf{y}(i-1)\bigr)$
  \State $Q{(i)} \gets \bigl\{\, \mathbf{x} \in P \ \big|\ \langle \mathbf{w}(i), \mathbf{x} \rangle \ge \gamma \bigl(v{(i-1)} - F(\mathbf{y}(i-1)) \bigr) \,\bigr\}$
  \State Use Theorem~\ref{thm:weak-dr-smooth} to compute an approximate local maximum $\mathbf{x}(i)$ of $Q{(i)}$ {(if $Q{(i)}=\varnothing$, set $\mathbf{x}(i)$ to an arbitrary vector in $P$)}
  \State $\mathbf{y}(i) \gets \mathbf{y}(i-1) + \delta \,\bigl(\mathbf{1} - \mathbf{y}(i-1) - \mathbf{z}(i-1)\bigr) \odot \mathbf{x}(i)$
\EndFor
\State \textbf{return} $\mathbf{y}(\delta^{-1})$ and the sequence $\mathbf{x}(1), \mathbf{x}(2), \hdots, \mathbf{x}(\delta^{-1})$
\end{algorithmic}
\end{algorithm}

{For a fixed $\gamma$, the parameters $t_s$ and $\varepsilon$ are constants, and hence the running time of Algorithm~\ref{alg:fw-guided-mcg} is $\mathrm{Poly}(n, \delta^{-1})$.}
Algorithm~\ref{alg:fw-guided-mcg} therefore guarantees the following performance on its output.
The proof of this theorem is provided in Appendix~\ref{sec:fwg_proof}.

\begin{theorem}\label{thm:fw-guided-mcg}
\textnormal{$\gamma$-FWG} takes as input a nonnegative, $L$-smooth, $\gamma$–weakly DR-submodular function
$F:[0,1]^n \to \mathbb{R}_{\ge0}$, a meta-solvable down-closed convex body
$P \subseteq [0,1]^n$ of diameter $D$, a vector $\bz \in P$, and parameters $t_s \in (0,1)$,
$\varepsilon \in (0,1/2)$, and $\delta \in (0,1)$. Given this input, \textnormal{$\gamma$-FWG} outputs a vector
$\by \in P$ and vectors $\bx{(1)},\hdots,\bx{(m)} \in P$ for some $m = O(\delta^{-1}+\varepsilon^{-1})$,
such that at least one of the following holds.

\begin{enumerate}
\item\begin{equation}\label{eq:item1-ABC}
F(\by) \ \ge\ A_\gamma(t_s)\,F(\bo)\;+\;B_\gamma(t_s)\,F(\bz\odot\bo)\;+\;C_\gamma(t_s)\,F(\bz\oplus\bo)\;-\;\delta L D^{2}.
\end{equation}

\item There exists \(i\in[m]\) such that
\begin{equation}\label{eq:item2-gap}
F\bigl(\bx{(i)} \oplus \bo\bigr)\ \le\ F(\bz \oplus \bo)\;-\;\varepsilon\,F(\bo),
\end{equation}
and the point \(\bx{(i)}\) satisfies the \(\gamma\)–weakly DR local–value bound
\begin{align}\label{eq:item2-local}
F\bigl(\bx{(i)}\bigr)
&\ge\;
\frac{\gamma^{2}\,F\bigl(\bx{(i)}\vee \bo\bigr)+F\bigl(\bx{(i)}\wedge \bo\bigr)}{1+\gamma^{2}}
-\frac{\delta\,\gamma}{1+\gamma^{2}}
\left(\max_{\by'\in Q(i)}F(\by')+\tfrac{1}{2}L D^{2}\right).
\end{align}
i.e., \(\bx{(i)}\) is an approximate local maximum with respect to \(\bo\) under the \(\gamma\)–weakly DR guarantee and the Frank–Wolfe certificate over \(Q(i)\).
\end{enumerate}

Here the $\gamma$–dependent coefficients are
\begin{subequations}\label{eq:ABC-coeffs}
\begin{align}
A_\gamma(t_s)
&:= -\frac{e^{\gamma t_s-\gamma}}{\,1-\gamma\,}
+\frac{e^{-\gamma^2}}{\gamma(1-\gamma)}\Big(e^{\gamma^2 t_s}-(1-\gamma)\Big)\;-\; O(\varepsilon)
\label{eq:A-gamma}\\[2pt]
B_\gamma(t_s)
&:= \frac{e^{-\gamma}-e^{\gamma t_s-\gamma}}{\gamma}
\label{eq:B-gamma}\\[2pt]
C_\gamma(t_s)
&:= \frac{e^{\gamma^2 t_s}-1}{\gamma(1-\gamma)}\Big(e^{-\gamma(1-t_s)-\gamma^2 t_s}-e^{-\gamma^2}\Big)
\label{eq:C-gamma}\\
&\quad+\;\frac{e^{-\gamma(1-t_s)}}{\gamma}\left[
\big(e^{-\gamma t_s}-1\big)
+\frac{e^{-\gamma^2 t_s}-e^{-\gamma t_s}}{1-\gamma}
\right]\nonumber\\
&\quad+\,e^{-\gamma(1-t_s)-\gamma^2 t_s}\Bigg[
\frac{\gamma^2}{1-\gamma}(1-t_s)\,e^{\gamma(1-\gamma)(1-t_s)}\nonumber\\
&\hspace{2cm}+\frac{\gamma}{(1-\gamma)^2}\Big(1-e^{\gamma(1-\gamma)(1-t_s)}\Big)
\Bigg].\nonumber
\end{align}
\end{subequations}

\end{theorem}

\subsection{Main Algorithm} \label{sec:main_algo}

In this section we describe our main algorithm and analyze the recursive framework that establishes our main result (Theorem~\ref{thm:main}). 
The core building block is our new procedure \emph{$\gamma$-weakly Frank--Wolfe Guided Measured Continuous Greedy} ($\gamma$-FWG), introduced in Section~\ref{sec:main_com}.
We classify an execution of $\gamma$-FWG as \emph{successful} if the first outcome holds (i.e., when $F(\mathbf{y})$ attains the ``large value'' case). Otherwise, the execution is deemed \emph{unsuccessful}. With this terminology in place, we now describe the main recursive driver, \emph{Algorithm~\ref{alg:main}}, which we use to prove Theorem~\ref{thm:main}. In addition to the parameters appearing in Theorem~\ref{thm:main}, Algorithm~\ref{alg:main} takes two auxiliary inputs: $\varepsilon\in(0,1/2)$ and $t_s\in(0,1)$. These are forwarded unchanged to every call to $\gamma$-FWG.

Algorithm~\ref{alg:main} runs for \(L=1+\left\lceil\tfrac{1+\gamma}{\varepsilon\gamma}\right\rceil\) recursion levels, indexed by \(i\).
At level \(1\), it finds an approximate local maximizer \(\mathbf{z}(0)\in P\) and, from this seed, runs
\textsc{Box-Maximization} and \(\gamma\)-FWG, producing \(\mathbf{z}'\), \(\mathbf{y}\), and a batch \(\{\mathbf{x}(1),\ldots,\mathbf{x}(m)\}\).
For each subsequent level \(i=2,\ldots,L\), every candidate \(\mathbf{x}(\cdot)\) emitted at level \(i-1\) becomes a new seed \(\mathbf{z}\);
the same two subroutines are applied to each seed, yielding fresh outputs \(\mathbf{z}'\), \(\mathbf{y}\), and
\(\{\mathbf{x}(1),\ldots,\mathbf{x}(m)\}\).
After all levels complete, the algorithm returns the vector with the largest objective value among
all \(\mathbf{z}'\), \(\mathbf{y}\), and \(\mathbf{x}(\cdot)\) produced at any level.

\begin{algorithm}[H]
\caption{Main Algorithm($F, P, \gamma, t_s, \varepsilon, \delta$)}
\label{alg:main}
\begin{algorithmic}[1]
\State Let $\bz(0)$ be an local maximum in $P$ obtained via the Theorem~\ref{thm:unbalanced-weakDR}.
\State Execute MAIN-RECURSIVE$(F, P,\gamma, t_s , \bz(0), \varepsilon, \delta, 1)$.
\Function{MAIN-RECURSIVE}{$F, P, \gamma, t_s , \bz, \varepsilon, \delta, i$}
\State $\bz' = $ Box-Maximization$(\bz)$ 
\State Let $(\by, \bx(1), \hdots, \bx(m)) = $ $\gamma$-FWG$(F, P, \bz,\gamma, t_s , \varepsilon)$ (Algorithm~\ref{alg:fw-guided-mcg})
\If{$i < L$}
\For{$j = 1$ \textbf{to} $m$}
\State $\by(j)\ =\ $MAIN-RECURSIVE$(F, P,\gamma, \bx(j), t_s , \varepsilon, \delta, i + 1)$
\EndFor
\EndIf
\State \textbf{return} the vector maximizing $F$ among $\bz'$, $\by$ and the vectors in $\{\by(j) \mid j \in [m]\}$.
\EndFunction
\end{algorithmic}
\end{algorithm}

{Observe that, for fixed $\gamma$, the number of recursive calls executed by Algorithm~1 is 
$O\bigl(m^{L}\bigr) = (\delta^{-1} + \varepsilon^{-1})^{O(1/\varepsilon)}$. 
For any constant $\varepsilon$, this quantity is polynomial in $\delta^{-1}$. 
Moreover, each individual recursive call runs in time polynomial in $\delta^{-1}$ and $n$. 
Therefore, the overall running time of Algorithm~1 is polynomial in $\delta^{-1}$ and $n$.}

We say that a recursive call of Algorithm~\ref{alg:main} is \emph{successful} if its internal run of $\gamma$-FWG is successful. 
Section~\ref{sec:last_proof} shows that Algorithm~\ref{alg:main} performs sufficiently many recursive invocations to ensure that at least one call is successful and, moreover, that it obtains a vector $\mathbf{z}$ which is an approximate local maximizer with respect to $\bo$. 
From such a call, the analysis further proves that either the accompanying vector $\mathbf{z}'$ or the vector $\mathbf{y}$ satisfies the performance guarantee stated in Theorem~\ref{thm:main}.

\subsection{The Main Result}\label{sec:last_proof}

In the recursion tree of Algorithm~\ref{alg:main}, we focus on one designated path of
\emph{heir} calls. Along this path, the ``fallback'' guarantee of $\gamma$-FWG is passed
forward at each level. A recursive call $\mathcal{C}$ is an \emph{heir} if either
\begin{enumerate}
    \item $\mathcal{C}$ is the unique level-$1$ call (i.e., the first invocation in the recursion), or
    \item $\mathcal{C}$ was invoked by another heir call $\mathcal{C}_p$ that is \emph{unsuccessful}, and its input seed
    $\bz$ equals $\bx(j^\star)$ for some index $j^\star$ that satisfies the second outcome of Theorem~\ref{thm:fw-guided-mcg}
    for the $\gamma$-FWG run inside $\mathcal{C}_p$.
\end{enumerate}

Intuitively, if an heir call \(\mathcal{C}_p\) is \emph{unsuccessful} (its run of \(\gamma\)-FWG does not return a large-value \(\by\)),
then Theorem~\ref{thm:fw-guided-mcg} guarantees an index \(j^\star\in[m]\) for which \(\bx{(j^\star)}\) satisfies a
first-order \(\gamma\)–weakly DR certificate. We then set the seed of the next heir on the designated path to
\(\bz \gets \bx{(j^\star)}\). The following observation records the certificate’s invariant, which we subsequently combine to obtain the final guarantee.




\begin{observation}\label{obs:weakDR-heir}
Fix \(\gamma\in(0,1]\). Every heir recursive call in Algorithm~\ref{alg:main} receives a seed
\(\bz\in P\) satisfying
\begin{equation}\label{eq:heir-invariant}
F(\bz)\ \ge\ \frac{\gamma^{2}\,F(\bz\vee \bo)\;+\;F(\bz\wedge \bo)}{1+\gamma^{2}}
\;-\;O(\varepsilon)\,F(\bo)\;-\;O(\delta\,L D^{2}).
\end{equation}
\end{observation}

\begin{proof}
We argue by cases on the recursion level that produces $\bz$.

\textbf{Case 1:} $\bz$ comes from a later level.
Here $\bz=\bx{(i)}$ for some index $i$ returned by the previous \textnormal{$\gamma$-FWG} call, where the “successful” bound \eqref{eq:item1-ABC} did not apply.
Hence \eqref{eq:item2-local} holds:
\begin{align}
    F(\bx{(i)})\ &\ge\ 
\frac{\gamma^{2}F(\bx{(i)}\vee \bo)+F(\bx{(i)}\wedge \bo)}{1+\gamma^{2}}
-\frac{\delta\,\gamma}{1+\gamma^{2}}\!\left(\max_{\by'\in Q(i)}F(\by')+\tfrac{1}{2}L D^{2}\right).
\label{eq:local-certificate-short}
\end{align}
Since $Q(i)\subseteq P$, we have $\max_{\by'\in Q(i)}F(\by')\le \max_{\by'\in P}F(\by')\le F(\bo)$, and substituting this into \eqref{eq:local-certificate-short} with $\bz=\bx{(i)}$ yields
\begin{equation}\label{eq:heir-case1-mid}
F(\bz)\ \ge\ 
\frac{\gamma^{2}F(\bz\vee \bo)+F(\bz\wedge \bo)}{1+\gamma^{2}}
-\frac{\delta\,\gamma}{1+\gamma^{2}}\!\left(F(\bo)+\tfrac{1}{2}L D^{2}\right),
\end{equation}
which matches the invariant \eqref{eq:heir-invariant} up to the stated $O(\delta)\,F(\bo)$ and $O(\delta\,L D^{2})$ terms.

\textbf{Case 2:} $\bz$ is produced at the first recursion level.
Here $\bz$ is the output of the weakly-DR local-maximization routine from Theorem~\ref{thm:weak-dr-smooth} with accuracy parameter
$\eta\ :=\ \min\{\varepsilon,\delta\}.$
Applying Theorem~\ref{thm:weak-dr-smooth} with $\by=\bo$ and using $\max_{\by'\in P}F(\by')\le F(\bo)$ gives
\begin{equation}\label{eq:first-level}
F(\bz)\ \ge\ 
\frac{\gamma^{2}F(\bz\vee \bo)+F(\bz\wedge \bo)}{1+\gamma^{2}}
\;-\;\frac{\eta\,\gamma}{1+\gamma^{2}}\Bigl(F(\bo)+\tfrac{1}{2}L D^{2}\Bigr).
\end{equation}
Since $\eta\le\varepsilon$ and $\eta\le\delta$ by definition of $\eta$, the error term in \eqref{eq:first-level} is again of the form
$O(\varepsilon)\,F(\bo)+O(\delta\,L D^{2})$, yielding \eqref{eq:heir-invariant}.

Both cases establish \eqref{eq:heir-invariant}, it completes the proof.
\end{proof}

Before proving that a successful heir exists (Corollary~\ref{cor:weakDR-successful-heir}), we note a simple measure that drops at each level of recursion.

\begin{observation}\label{obs:weakDR-descent}
Assume every heir recursive call of Algorithm~\ref{alg:main} is unsuccessful (in the sense of
Theorem~\ref{thm:fw-guided-mcg}). Then, for every recursion level \(i\ge 1\),
there exists an heir call at level \(i\) that receives a seed \(\bz\) with
\begin{equation}\label{eq:descend-claim}
F(\bz\oplus \bo)\ \le\ F\bigl(\bz(0)\oplus \bo\bigr)\;-\;\varepsilon\,(i-1)\,F(\bo).
\end{equation}
\end{observation}

\begin{proof}
We argue by induction on the level \(i\).

\emph{Base case} (\(i=1\)). The unique level-1 heir call receives \(\bz=\bz(0)\). Hence
\begin{equation}\label{eq:base-case}
F(\bz\oplus \bo)\;=\;F\bigl(\bz(0)\oplus \bo\bigr)\;\le\;F\bigl(\bz(0)\oplus \bo\bigr)-\varepsilon\cdot 0\cdot F(\bo),
\end{equation}
which is exactly \eqref{eq:descend-claim} with \(i=1\).

\emph{Inductive step}. Assume the statement holds for level \(i-1\ge 1\); i.e., there is an heir call on
level \(i-1\) with seed \(\bz\) such that
\begin{equation}\label{eq:IH}
F(\bz\oplus \bo)\ \le\ F\bigl(\bz(0)\oplus \bo\bigr)\;-\;\varepsilon\,(i-2)\,F(\bo).
\end{equation}
By assumption, this heir call is unsuccessful. Therefore, the “gap” outcome
\eqref{eq:item2-gap} of Theorem~\ref{thm:fw-guided-mcg} applies to its internal run of \(\gamma\)-FWG, and hence there exists
\(j^\star\in[m]\) such that
\begin{equation}\label{eq:gap-step}
F\bigl(\bx(j^\star)\oplus \bo\bigr)\ \le\ F(\bz\oplus \bo)\;-\;\varepsilon\,F(\bo).
\end{equation}
Combining \eqref{eq:IH} and \eqref{eq:gap-step} gives
\begin{equation}\label{eq:level-i}
F\bigl(\bx(j^\star)\oplus \bo\bigr)\ \le\ F\bigl(\bz(0)\oplus \bo\bigr)\;-\;\varepsilon\,(i-1)\,F(\bo).
\end{equation}
By the definition of heirs, the child call at level \(i\) seeded with \(\bz\gets \bx(j^\star)\) is itself an heir call and satisfies \eqref{eq:level-i}, which is precisely \eqref{eq:descend-claim} for level \(i\).
\end{proof}

The weakly-DR specifics (the $\gamma$-aware local-value bound and Frank--Wolfe certificate)
only affect the \emph{quality} guarantee for the seed $\bx(j^\star)$, not the \emph{descent
amount} on $F(\cdot\oplus o)$. Thus the $\varepsilon$-per-level decrease remains identical to
the DR case \cite{buchbinder2024constrained}, while $\gamma$ enters later in the value lower bounds used to conclude the analysis.

\begin{corollary}\label{cor:weakDR-successful-heir}
Some recursive call of Algorithm~\ref{alg:main} is a successful heir.
\end{corollary}

\begin{proof}
Assume, toward a contradiction, that no recursive call is a successful heir.  
By Observation~\ref{obs:weakDR-descent}, at level
\begin{equation}\label{eq:level-choice}
i\ :=\ 1+\Bigl\lceil \tfrac{\gamma+1}{\gamma\,\varepsilon}\Bigr\rceil
\end{equation}
there exists an heir call with seed $\bz$ such that
\begin{align}
F(\bz\oplus \bo)
&\overset{\text{(a)}}{\le}\;
F(\bz(0)\oplus \bo)\;-\;\varepsilon\,(i-1)\,F(\bo)\notag\\
&\overset{\text{(b)}}{\le}\;
F(\bz(0)\oplus \bo)\;-\;\Bigl(1+\tfrac{1}{\gamma}\Bigr)\,F(\bo),
\label{eq:descend-instantiated}
\end{align}
where \(\text{(a)}\) follows from Observation~\ref{obs:weakDR-descent} applied at level \(i\), and \(\text{(b)}\) uses
\(i-1 \ge (\gamma+1)/(\gamma\varepsilon)\) from \eqref{eq:level-choice}.

Since this call is (by assumption) also unsuccessful, the gap alternative \eqref{eq:item2-gap} of Theorem~\ref{thm:fw-guided-mcg} applies, yielding some $j$ with
\begin{equation}\label{eq:unsuccessful-gap}
F\bigl(\bx(j)\oplus \bo\bigr)
\;\le\;
F(\bz\oplus \bo) \;-\; \varepsilon\,F(\bo).
\end{equation}
Combining \eqref{eq:descend-instantiated} and \eqref{eq:unsuccessful-gap} gives
\begin{equation}\label{eq:combined-gap}
F\bigl(\bx(j)\oplus \bo\bigr)
\;\le\;
F(\bz(0)\oplus \bo) \;-\; \Bigl(1+\tfrac{1}{\gamma}+\varepsilon\Bigr)\,F(\bo).
\end{equation}

By nonnegativity of $F$, we have $F\bigl(\bx(j)\oplus \bo\bigr)\ge 0$, so \eqref{eq:combined-gap} implies
\begin{equation}\label{eq:rearrange-positivity}
F(\bz(0)\oplus \bo) - F(\bo)
\;\ge\;
\Bigl(\tfrac{1}{\gamma}+\varepsilon\Bigr)\,F(\bo)
\;>\;
\tfrac{1}{\gamma}\,F(\bo).
\end{equation}
On the other hand, by $\gamma$–weakly DR property we have
\begin{align}
F(\bz(0)\oplus \bo) - F(\bo)
&\overset{\text{(c)}}{\le}\;
\frac{1}{\gamma}\Bigl(F\bigl(\bz(0)\odot(\bone-\bo)\bigr)-F(\bzero)\Bigr)\notag\\
&\overset{\text{(d)}}{\le}\;
\frac{1}{\gamma}\,F\bigl(\bz(0)\odot(\bone-\bo)\bigr)
\label{eq:weakdr-split}
\end{align}
where \(\text{(c)}\) follows from the $\gamma$–weakly DR definition,
and \(\text{(d)}\) uses $F(\bzero)\ge 0$ (nonnegativity).
Combining \eqref{eq:rearrange-positivity} and \eqref{eq:weakdr-split} yields
\begin{equation}\label{eq:strict-better}
F(\bo)\ <\ F\bigl(\bz(0)\odot(\bone-\bo)\bigr).
\end{equation}
Since $P$ is down-closed and $\bz(0)\in P$, we have $\bz(0)\odot(\bone-\bo)\le \bz(0)$ coordinate-wise, hence $\bz(0)\odot(\bone-\bo)\in P$. The strict improvement in \eqref{eq:strict-better} contradicts the optimality of $\bo$ over $P$. Therefore, our assumption was false, and some recursive call must be a successful heir.
\end{proof}

Consider any successful heir call within Algorithm~\ref{alg:main}. Denote its input seed by
\(\bz^{\star}\), and let \(\by^{\star}\) be the high-value solution returned by \(\gamma\)-FWG, while
\(\bz'^{\star}\) is the child seed produced during the same call. Invoking
Theorem~\ref{thm:fw-guided-mcg}\,(1) yields
\begin{equation}\label{eq:successful-heir-ABC}
F(\by^{\star})
\ge
A_\gamma(t_s)\,F(\bo)
+
B_\gamma(t_s)\,F(\bz^{\star}\odot \bo)
+
C_\gamma(t_s)\,F(\bz^{\star}\oplus \bo)
-
\delta\,L D^{2}.
\end{equation}
For the ensuing analysis of a successful heir, we also need a companion lower bound for
\(F(\bz'^{\star})\); this is provided by the next lemma.

\begin{lemma}\label{lem:weak-35}
Let $F:[0,1]^n\to\mathbb{R}_{\ge 0}$ be nonnegative, $L$-smooth, and $\gamma$-weakly DR-submodular for some $\gamma\in(0,1]$.
Let $\bz^{\ast}\in[0,1]^n$ be the incumbent vector provided to the heir recursive call, and let $\bz'{}^{\ast}\le \bz^{\ast}$ be the output of Corollary~\ref{cor:box-weakDR} (run on the box $[0,\bz^{\ast}]$ with error $\varepsilon$).
Then
\begin{align}
  F(\bz'{}^{\ast})
& \ge 
\max_{r\ge 0}
\frac{\Big(2\gamma^{3/2}r+\frac{\gamma}{1+\gamma^{2}}r^{2}\Big)F(\bz^{\ast}\odot \bo)
+\frac{\gamma^{2}}{1+\gamma^{2}}r^{2}F(\bz^{\ast}\oplus \bo)}
{\,r^{2}+2\gamma^{3/2}r+1\,}
-O(\varepsilon)F(\bo)-O(\delta L D^{2}).
\end{align}
\end{lemma}

\begin{proof}
Applying Corollary~\ref{cor:box-weakDR} to the box $[0,\bz^{\ast}]$ (with error $\varepsilon$) gives
\begin{align}
F(\bz'{}^{\ast})
&\ge
\max_{r\ge 0}
\frac{\bigl(2\gamma^{3/2}-4\varepsilon\,\gamma^{9/2}\bigr) r F(\bz^{\ast}\odot \bo)+F(\bzero)+r^{2}\,F(\bz^{\ast})}
{\,r^{2}+2\gamma^{3/2}r+1\,}.
\label{eq:box-bound}
\end{align}
Since $F(\bzero)\ge 0$, dropping the nonnegative $F(\bzero)$ can only decrease the right-hand side, therefore we get (a);
\begin{align}
F(\bz'{}^{\ast})
&\overset{\text{(a)}}{\ge}
\max_{r\ge 0}
\frac{\bigl(2\gamma^{3/2}-4\varepsilon\,\gamma^{9/2}\bigr)\,r\,F(\bz^{\ast}\odot \bo)\;+\;r^{2}\,F(\bz^{\ast})}
{\,r^{2}+2\gamma^{3/2}r+1\,}\notag\\
&\overset{\text{(b)}}{\ge}
\max_{r\ge 0}
\frac{2\gamma^{3/2}\,r\,F(\bz^{\ast}\odot \bo)\;+\;r^{2}\,F(\bz^{\ast})}
{\,r^{2}+2\gamma^{3/2}r+1\,}
\;-\;O(\varepsilon)\,F(\bo),
\label{eq:eps-drop}
\end{align}
where (b) uses $F(\bz^{\ast}\odot \bo)\le F(\bo)$ and the bound
\(
\displaystyle \frac{r}{r^{2}+2\gamma^{3/2}r+1}\le 1
\),
so the negative perturbation term
\(
-\,4\varepsilon\,\gamma^{9/2}\,\frac{r}{r^{2}+2\gamma^{3/2}r+1}\,F(\bz^{\ast}\odot\bo)
\)
is at worst $-O(\varepsilon)\,F(\bo)$ after maximizing over $r\ge 0$.

Next, since $\bz^{\ast}$ is the seed of an heir recursive call, Observation~\ref{obs:weakDR-heir} yields
\begin{align}
F(\bz^{\ast})
&\ge
\frac{\gamma^{2}F(\bz^{\ast}\vee \bo)+F(\bz^{\ast}\wedge \bo)}{1+\gamma^{2}}
-O(\varepsilon)F(\bo)-O(\delta L D^{2}).
\label{eq:heir-inv}
\end{align}
Because $F\ge 0$ and $\gamma\le 1$, we also have $F(\bz^{\ast}\wedge \bo)\ge \gamma\,F(\bz^{\ast}\wedge \bo)$; combining this with the weakly-DR “swap” inequality
\begin{equation}\label{eq:weakDR-swap}
\gamma\,F(\bx\vee\by)+F(\bx\wedge\by)\ \ge\ \gamma\,F(\bx\oplus\by)+F(\bx\odot\by)
\quad(\forall\,\bx,\by\in[0,1]^n),
\end{equation}
we get the refined lower bound
\begin{align}
F(\bz^{\ast})
&\ge
\frac{\gamma^{2}F(\bz^{\ast}\oplus \bo)+\gamma\,F(\bz^{\ast}\odot \bo)}{1+\gamma^{2}}
-O(\varepsilon)F(\bo)-O(\delta L D^{2}).
\label{eq:heir-inv-swapped}
\end{align}

Substituting \eqref{eq:heir-inv-swapped} into \eqref{eq:eps-drop}, and noting that
\(\frac{r^{2}}{r^{2}+2\gamma^{3/2}r+1}\le 1\),
preserves the additive error
$-O(\varepsilon)\,F(\bo)-O(\delta L D^{2})$ and yields coefficients
\begin{align*}
    \frac{2\gamma^{3/2}r}{r^{2}+2\gamma^{3/2}r+1}
\;+\;
\frac{\gamma}{1+\gamma^{2}}\cdot\frac{r^{2}}{r^{2}+2\gamma^{3/2}r+1}
&\quad\text{for }F(\bz^{\ast}\odot \bo),\\
\frac{\gamma^{2}}{1+\gamma^{2}}\cdot\frac{r^{2}}{r^{2}+2\gamma^{3/2}r+1}
&\quad\text{for }F(\bz^{\ast}\oplus \bo),
\end{align*}
which is exactly the claimed form.
\end{proof}

We have established two certified lower bounds for any \emph{successful} heir call:
one for the Frank--Wolfe–guided output \(\by^{\star}\) (Theorem~\ref{thm:fw-guided-mcg}(1))
and one for the box-restricted child \(\bz'^{\star}\) (Lemma~\ref{lem:weak-35}).
Since the algorithm returns the better of these two values, \emph{any} convex combination
of the two bounds remains a valid lower bound on the algorithm’s output.
Let \(\alpha\in[0,1]\) be the mixing parameter.
Putting the two bounds into a common form and combining them gives, for any
\(r\ge 0\) and \(t_s\in(0,1)\),
\begin{align}
F(\textnormal{ALG})
\ \ge\ 
&(1-\alpha)\,A_\gamma(t_s)\, F(\bo)+\;\Big[(1-\alpha)\,B_\gamma(t_s) \;+\;\alpha\,D_\gamma(r)\Big]\;F(\bz^{\star}\odot\bo)\notag\\
&+\;\Big[(1-\alpha)\,C_\gamma(t_s) \;+\;\alpha\,E_\gamma(r)\Big] \;F(\bz^{\star}\oplus\bo)-\;O(\varepsilon)\,F(\bo) \;-\;O(\delta\,L D^{2}),
\label{eq:convex-combo-master}
\end{align}
where
\begin{equation}\label{eq:DE-defs}
D_\gamma(r)\ :=\ \frac{2\gamma^{3/2}\,r+\frac{\gamma}{1+\gamma^2}\,r^2}{\,r^2+2\gamma^{3/2}r+1\,},
\qquad
E_\gamma(r)\ :=\ \frac{\frac{\gamma^2}{1+\gamma^2}\,r^2}{\,r^2+2\gamma^{3/2}r+1\,}.
\end{equation}

To extract a \emph{pure} multiplicative factor in front of \(F(\bo)\), we choose parameters so that
the coefficients multiplying \(F(\bz^{\star}\odot\bo)\) and \(F(\bz^{\star}\oplus\bo)\) are nonnegative,
allowing these terms to be dropped by nonnegativity of \(F\).
Define the feasible set
\begin{align}
\mathcal{F}_\gamma
:= \Bigl\{(\alpha,r,t_s)&\in[0,1]\times[0,\infty)\times(0,1)\ :\notag\\
&\hspace{0.25cm}(1-\alpha)B_\gamma(t_s)+\alpha D_\gamma(r) \ge 0, 
(1-\alpha)C_\gamma(t_s)+\alpha E_\gamma(r) \ge 0\Bigr\}.
\label{eq:feas-region}
\end{align}
For any \((\alpha,r,t_s)\in\mathcal{F}_\gamma\), inequality \eqref{eq:convex-combo-master} yields
\begin{equation}\label{eq:drop-terms}
F(\textnormal{ALG})
\ \ge\ 
\Bigl[(1-\alpha)\,A_\gamma(t_s)\;-\;O(\varepsilon)\Bigr]\,F(\bo)\;-\;O(\delta\,L D^{2}).
\end{equation}
This motivates optimizing the leading factor:
\begin{equation}\label{eq:phi-gamma-prob}
\Phi_\gamma\ :=\ \max_{(\alpha,r,t_s)\in\mathcal{F}_\gamma}\ (1-\alpha)\,A_\gamma(t_s).
\end{equation}

Our approximation guarantee $\Phi_\gamma$ is defined as the optimal value of the
maximization problem in~\eqref{eq:phi-gamma-prob}. This is an optimization over only
three scalar parameters $(\alpha,r,t_s)$ with simple linear feasibility constraints
in~\eqref{eq:feas-region}, so for any fixed $\gamma$ the problem is easy to solve
numerically. In particular, even though $\Phi_\gamma$ does not appear to admit a
closed-form expression as a function of $\gamma$, it can be computed to any desired
accuracy in polynomial time (in the inverse of the discretization step) by a direct
grid search.

For each fixed $\gamma$, we solve~\eqref{eq:phi-gamma-prob} by performing an explicit
grid search over $(r,t_s)$ on a bounded domain $r\in[0,r_{\max}]$, $t_s\in(0,1)$, and
then optimizing over $\alpha$ in closed form using the linear constraints
in~\eqref{eq:feas-region}. For every grid point $(r,t_s)$, we determine the interval
of feasible $\alpha\in[0,1]$ for which both inequalities in~\eqref{eq:feas-region}
hold, and then choose the endpoint of this interval that maximizes
$(1-\alpha)A_\gamma(t_s)$. This yields a candidate triple $(\alpha,r,t_s)$ and a
corresponding candidate value of $\Phi_\gamma$, and we keep the best one over all grid
points. The running time is polynomial in the inverse grid step. This is precisely
the procedure implemented in our Python code to generate Fig.~\ref{fig:phi-gamma}
and the parameter table (we used $r_{\max}=10$, since larger values of $r_{\max}$
did not change the outcome).


\emph{Boundary case \(\gamma=1\).}
Several coefficients (e.g., \(A_\gamma,B_\gamma,C_\gamma\)) contain factors of \((1-\gamma)\) in the
denominator; consequently, the expressions inside \eqref{eq:convex-combo-master} may exhibit an apparent \(0/0\)
form as \(\gamma\to 1\).
We interpret all such terms by taking their continuous limits, and
evaluate via L’Hôpital’s rule where needed.
Substituting these limits into \eqref{eq:convex-combo-master} yields the DR (i.e., \(\gamma=1\)) specialization of our
mixture bound, and optimizing \eqref{eq:phi-gamma-prob} at \(\gamma=1\) reproduces the current best DR guarantee
(\(0.401\)) of Buchbinder and Feldman~\cite{buchbinder2024constrained}.

\paragraph{What the optimization achieves.}
The optimized guarantee \(\Phi_\gamma\) (a) exactly matches the current best DR constant at \(\gamma=1\),
and (b) strictly improves on the non-monotone weakly-DR baseline \(\kappa(\gamma)=\gamma e^{-\gamma}\)
for all \(\gamma\in(0,1)\). Intuitively, the mixture balances Frank--Wolfe–guided progress with
box-restricted improvements through \((\alpha,r,t_s)\), certifying the best factor per \(\gamma\).

\begin{theorem}\label{thm:main}
Fix \(\gamma\in(0,1]\) and \(\delta\in(0,1)\).
Let \(F:[0,1]^n\to\mathbb{R}_{\ge 0}\) be a nonnegative, \(L\)-smooth, \(\gamma\)-weakly DR-submodular function,
and let \(P\subseteq[0,1]^n\) be a down-closed, meta-solvable convex body of diameter \(D\).
There exists a polynomial-time algorithm that returns a point \(\bx\in P\) such that
\begin{equation}\label{eq:main-guarantee}
F(\bx)\ \ge\ \Phi_\gamma\cdot \max_{\by\in P} F(\by)\;-\;O\!\left(\delta\,D^{2}L\right),
\end{equation}
where \(\Phi_\gamma\) is the optimal value of \eqref{eq:phi-gamma-prob}.
\end{theorem}

\begin{proof}
Let \(\bo\in\arg\max_{\by\in P}F(\by)\).
By Corollary~\ref{cor:weakDR-successful-heir}, Algorithm~\ref{alg:main} has at least one
\emph{successful} heir call.
For such a call with seed \(\bz^\star\), Theorem~\ref{thm:fw-guided-mcg}~(1)
and Lemma~\ref{lem:weak-35} provide two certified lower bounds on the returned candidates
\(\by^\star\) and \(\bz'{}^\star\).
Forming any convex combination of these two bounds yields \eqref{eq:convex-combo-master},
with \(D_\gamma, E_\gamma\) given in \eqref{eq:DE-defs}.

Choose \((\alpha,r,t_s)\in\mathcal{F}_\gamma\) (cf.\ \eqref{eq:feas-region})
so that the coefficients of \(F(\bz^\star\odot \bo)\) and \(F(\bz^\star\oplus \bo)\) in
\eqref{eq:convex-combo-master} are nonnegative.
Dropping these nonnegative contributions gives \eqref{eq:drop-terms}:
\begin{equation}
    F(\textnormal{ALG})\ \ge\ \bigl[(1-\alpha)A_\gamma(t_s)-O(\varepsilon)\bigr]\,F(\bo)\;-\;O(\delta L D^2).
\end{equation}
Maximizing over feasible \((\alpha,r,t_s)\) yields \(\Phi_\gamma\) from \eqref{eq:phi-gamma-prob},
so, for an optimal choice, 
\begin{equation}
    F(\textnormal{ALG})\ \ge\ \bigl[\Phi_\gamma-O(\varepsilon)\bigr]\,F(\bo)\;-\;O(\delta L D^2).
\end{equation}
Finally, set \(\varepsilon=\Theta(\delta)\) to absorb the \(-O(\varepsilon)F(\bo)\) term into the
\(-O(\delta L D^2)\) smoothing error. 
\end{proof}

\section{Conclusion}
This paper develops a unified, projection-free framework for maximizing continuous, non-monotone $\gamma$-weakly DR-submodular functions over down-closed convex bodies. Our method couples a $\gamma$-aware Frank--Wolfe–guided measured continuous greedy with a $\gamma$-aware double--greedy, and optimizes a convex mixture of their certificates through three tunable parameters $(\alpha,r,t_s)$. Across the entire weakly-DR spectrum, the resulting guarantee $\Phi_\gamma$ \emph{strictly} improves the canonical non-monotone baseline $\kappa(\gamma)=\gamma e^{-\gamma}$, and at the DR boundary ($\gamma=1$) it \emph{matches} the current best constant $0.401$. We show improvements over prior work in Figure~\ref{fig:phi-gamma} and Table~\ref{tab:phi-params}.

\bibliographystyle{plain} 
\bibliography{ref}

@article{pedramfar2023unified,
  title={A unified approach for maximizing continuous DR-submodular functions},
  author={Pedramfar, Mohammad and Quinn, Christopher and Aggarwal, Vaneet},
  journal={Advances in Neural Information Processing Systems},
  volume={36},
  pages={61103--61114},
  year={2023}
}

@misc{PedramfarQuinnAggarwal2024,
  title={A Unified Approach for Maximizing Continuous $\gamma$-weakly DR-submodular Functions},
  author={Pedramfar, Mohammad and Quinn, Christopher John and Aggarwal, Vaneet},
  howpublished = {\url{https://optimization-online.org/?p=25915}},
  year      = {2024},
  month      = {Mar}
}

@article{nie2024gradient,
  title={Gradient methods for online DR-submodular maximization with stochastic long-term constraints},
  author={Nie, Guanyu and Aggarwal, Vaneet and Quinn, Christopher},
  journal={Advances in Neural Information Processing Systems},
  volume={37},
  pages={20510--20539},
  year={2024}
}

@article{pedramfar2024linear,
  title={From linear to linearizable optimization: A novel framework with applications to stationary and non-stationary dr-submodular optimization},
  author={Pedramfar, Mohammad and Aggarwal, Vaneet},
  journal={Advances in Neural Information Processing Systems},
  volume={37},
  pages={37626--37664},
  year={2024}
}

@inproceedings{
pedramfar2024unified,
title={Unified Projection-Free Algorithms for Adversarial {DR}-Submodular Optimization},
author={Mohammad Pedramfar and Yididiya Y. Nadew and Christopher John Quinn and Vaneet Aggarwal},
booktitle={The Twelfth International Conference on Learning Representations},
year={2024}
}

@inproceedings{buchbinder2024constrained,
  title={Constrained submodular maximization via new bounds for {DR}-submodular functions},
  author={Buchbinder, Niv and Feldman, Moran},
  booktitle={Proceedings of the 56th Annual ACM Symposium on Theory of Computing},
   year={2024}
}

@inproceedings{zhang2022stochastic,
  title={Stochastic continuous submodular maximization: Boosting via non-oblivious function},
  author={Zhang, Qixin and Deng, Zengde and Chen, Zaiyi and Hu, Haoyuan and Yang, Yu},
  booktitle={Proceedings of the 39th International Conference on Machine Learning},
  pages={26116--26134},
  year={2022},
}

@article{hassani2017gradient,
  title={Gradient methods for submodular maximization},
  author={Hassani, Hamed and Soltanolkotabi, Mahdi and Karbasi, Amin},
  journal={Advances in Neural Information Processing Systems},
  volume={30},
  year={2017}
}

@inproceedings{chekuri2011nonmonotone,
  title={Submodular Function Maximization via the Multilinear Extension},
  author={Chekuri, Chandra and Vondr{\'a}k, Jan and Zenklusen, Rico},
  booktitle={Proceedings of the Twenty-Second Annual ACM-SIAM Symposium on Discrete Algorithms (SODA)},
  pages={114--125},
  year={2011},
  organization={SIAM}
}

@inproceedings{chen2018online,
  title={Online Continuous Submodular Maximization},
  author={Chen, Lin and Hassani, Hamed and Karbasi, Amin},
  booktitle={Proceedings of the 35th International Conference on Machine Learning},
  year={2018},
}

@inproceedings{pmlr-v54-bian17a,
  title={Guaranteed Non-convex Optimization: Submodular Maximization over Continuous Domains},
  author={Bian, Andrew An and Mirzasoleiman, Baharan and Buhmann, Joachim and Krause, Andreas},
  booktitle={Proceedings of the 20th International Conference on Artificial Intelligence and Statistics},
  year={2017}
}

@inproceedings{jaggi2013fw,
  title={Revisiting Frank--Wolfe: Projection-Free Sparse Convex Optimization},
  author={Jaggi, Martin},
  booktitle={Proceedings of the 30th International Conference on Machine Learning},
  pages={427--435},
  year={2013}
}

@inproceedings{lacoste2015linearFW,
  title={On the Linear Convergence of the Frank--Wolfe Algorithm},
  author={Lacoste-Julien, Simon and Jaggi, Martin},
  booktitle={Advances in Neural Information Processing Systems},
  year={2015}
}

@inproceedings{mualem2022partial,
  title={Using Partial Monotonicity in Submodular Maximization},
  author={Mualem, Loay and Feldman, Moran},
  booktitle={Advances in Neural Information Processing Systems},
  year={2022}
}

@inproceedings{qi2022isaac,
  title={On Maximizing Sums of Non-Monotone Submodular and Linear Functions},
  author={Qi, Benjamin},
  booktitle={International Symposium on Algorithms and Computation (ISAAC)},
  series={LIPIcs},
  volume={248},
  pages={41:1--41:16},
  year={2022}
}

@article{buchbinder2015tight,
  title={A Tight Linear Time (1/2)-Approximation for Unconstrained Submodular Maximization},
  author={Buchbinder, Niv and Feldman, Moran and Naor, Joseph and Schwartz, Roy},
  journal={SIAM Journal on Computing},
  volume={44},
  number={5},
  pages={1384--1402},
  year={2015}
}

@article{niazadeh2020jmlr,
  title={Optimal Algorithms for Continuous Non-monotone Submodular and {DR}-Submodular Maximization},
  author={Niazadeh, Rad and Roughgarden, Tim and Wang, Joshua R.},
  journal={Journal of Machine Learning Research},
  volume={21},
  pages={125:1--125:31},
  year={2020}
}

@article{conforti1984submodular,
  title={Submodular set functions, matroids and the greedy algorithm: tight worst-case bounds and some generalizations of the Rado-Edmonds theorem},
  author={Conforti, Michele and Cornu{\'e}jols, G{\'e}rard},
  journal={Discrete applied mathematics},
  volume={7},
  number={3},
  pages={251--274},
  year={1984},
  publisher={Elsevier}
}

@incollection{fisher2009analysis,
  title={An analysis of approximations for maximizing submodular set functions—II},
  author={Fisher, Marshall L and Nemhauser, George L and Wolsey, Laurence A},
  booktitle={Polyhedral Combinatorics: Dedicated to the memory of DR Fulkerson},
  pages={73--87},
  year={2009},
  publisher={Springer}
}

@incollection{korte1978analysis,
  title={An analysis of the greedy heuristic for independence systems},
  author={Korte, Bernhard and Hausmann, Dirk},
  booktitle={Annals of Discrete Mathematics},
  volume={2},
  pages={65--74},
  year={1978},
  publisher={Elsevier}
}

@article{nemhauser1978best,
  title={Best algorithms for approximating the maximum of a submodular set function},
  author={Nemhauser, George L and Wolsey, Laurence A},
  journal={Mathematics of operations research},
  volume={3},
  number={3},
  pages={177--188},
  year={1978},
  publisher={INFORMS}
}

@article{nemhauser1978analysis,
  title={An analysis of approximations for maximizing submodular set functions—I},
  author={Nemhauser, George L and Wolsey, Laurence A and Fisher, Marshall L},
  journal={Mathematical programming},
  volume={14},
  number={1},
  pages={265--294},
  year={1978},
  publisher={Springer}
}

@article{haastad2001some,
  title={Some optimal inapproximability results},
  author={H{\aa}stad, Johan},
  journal={Journal of the ACM (JACM)},
  volume={48},
  number={4},
  pages={798--859},
  year={2001},
  publisher={ACM New York, NY, USA}
}

@incollection{karp2009reducibility,
  title={Reducibility among combinatorial problems},
  author={Karp, Richard M},
  booktitle={50 Years of Integer Programming 1958-2008: from the Early Years to the State-of-the-Art},
  pages={219--241},
  year={2009},
  publisher={Springer}
}

@article{chekuri2005polynomial,
  title={A polynomial time approximation scheme for the multiple knapsack problem},
  author={Chekuri, Chandra and Khanna, Sanjeev},
  journal={SIAM Journal on Computing},
  volume={35},
  number={3},
  pages={713--728},
  year={2005},
  publisher={SIAM}
}

@inproceedings{feige2006approximation,
  title={Approximation algorithms for allocation problems: Improving the factor of 1-1/e},
  author={Feige, Uriel and Vondr{\'a}k, Jan},
  booktitle={2006 47th Annual IEEE Symposium on Foundations of Computer Science (FOCS'06)},
  pages={667--676},
  year={2006},
  organization={IEEE}
}

@incollection{cornuejols1977uncapacitated,
  title={On the uncapacitated location problem},
  author={Cornuejols, Gerard and Fisher, Marshall and Nemhauser, George L},
  booktitle={Annals of Discrete Mathematics},
  volume={1},
  pages={163--177},
  year={1977},
  publisher={Elsevier}
}

@article{austrin2016better,
  title={Better balance by being biased: A 0.8776-approximation for max bisection},
  author={Austrin, Per and Benabbas, Siavosh and Georgiou, Konstantinos},
  journal={ACM Transactions on Algorithms (TALG)},
  volume={13},
  number={1},
  pages={1--27},
  year={2016},
  publisher={ACM New York, NY, USA}
}

@inproceedings{FeldmanNaorSchwartzWard2011,
  author    = {Moran Feldman and Joseph Naor and Roy Schwartz and Justin Ward},
  title     = {Improved approximations for $k$-exchange systems},
  booktitle = {European Symposium on Algorithms (ESA)},
  editor    = {Camil Demetrescu and Magn{\'u}s M. Halld{\'o}rsson},
  series    = {Lecture Notes in Computer Science},
  volume    = {6942},
  pages     = {784--798},
  publisher = {Springer},
  year      = {2011}
}

@inproceedings{Bian2017NeurIPS,
  author    = {An Bian and Kfir Y. Levy and Andreas Krause and Joachim M. Buhmann},
  title     = {Continuous DR-submodular Maximization: Structure and Algorithms},
  booktitle = {Advances in Neural Information Processing Systems (NeurIPS)},
  year      = {2017}
}

@inproceedings{Bian2019ICML,
  author    = {Yatao Bian and Joachim M. Buhmann and Andreas Krause},
  title     = {Optimal Continuous DR-Submodular Maximization and Applications to Provable Mean Field Inference},
  booktitle = {Proceedings of the 36th International Conference on Machine Learning (ICML)},
  year      = {2019}
}

@inproceedings{gharan2011submodular,
  title={Submodular maximization by simulated annealing},
  author={Gharan, Shayan Oveis and Vondr{\'a}k, Jan},
  booktitle={Proceedings of the twenty-second annual ACM-SIAM symposium on Discrete Algorithms},
  pages={1098--1116},
  year={2011},
  organization={SIAM}
}

@inproceedings{Zhang2019NeurIPS,
  author    = {Mingrui Zhang and Lin Chen and Hamed Hassani and Amin Karbasi},
  title     = {Online Continuous Submodular Maximization: From Full Information to Bandit Feedback},
  booktitle = {Advances in Neural Information Processing Systems (NeurIPS)},
  year      = {2019}
}

\appendix

\newtheorem{exe}{Lemma}
\renewcommand{\theexe}{2.\arabic{exe}}  
\setcounter{exe}{0}

\section{Proofs of Section~\ref{sec:prelim}  Lemmas}\label{Proof_A}

In this appendix, we provide detailed proofs of the fundamental properties of $\gamma$-weakly DR-submodular functions and extend classical results for DR-submodular functions \cite{buchbinder2024constrained} to the generalized $\gamma$-weakly setting. 
For clarity, we restate lemmas before presenting their proofs.

\begin{exe}
Let $F:[0,1]^n\to\mathbb{R}_{\ge 0}$ be differentiable and $\gamma$-weakly DR-submodular.
Then for all $\bx,\by\in[0,1]^n$ and $\lambda\in[0,1]$ the following hold:
\begin{enumerate}
\item[\textnormal{(1)}] If $\bx\le \by$, then
\begin{equation}\label{eq:exe-convex-form-1}
F\bigl(\lambda \bx+(1-\lambda)\by\bigr)
\ \ge \
\frac{\lambda\,F(\bx)+\gamma^{2}(1-\lambda)\,F(\by)}{\lambda+\gamma^{2}(1-\lambda)}.
\end{equation}
Equivalently,
\begin{equation}\label{eq:exe-convex-form-2}
F\bigl((1-\lambda)\bx+\lambda \by\bigr)
\ \ge\
\frac{(1-\lambda)\,F(\bx)+\gamma^{2}\lambda\,F(\by)}{(1-\lambda)+\gamma^{2}\lambda}.
\end{equation}
\item[\textnormal{(2)}] If $\by\ge \bzero$ and $\bx+\by\in[0,1]^n$, then
\begin{equation}\label{eq:exe-increment-form}
F(\bx+\lambda \by)-F(\bx)
\ \ge\
\frac{\gamma^{2}\lambda}{\,1-\lambda+\gamma^{2}\lambda\,}\,\bigl(F(\bx+\by)-F(\bx)\bigr).
\end{equation}
\end{enumerate}
\end{exe}

\begin{proof}
Fix $\bx\in[0,1]^n$ and a direction $\bv\ge 0$ such that $\bx+t \bv\in[0,1]^n$ for all
$t\in[0,1]$. Define the univariate function
\begin{equation}\label{eq:phi-def}
    \phi(t)\ :=\ F(\bx+t \bv), \qquad t\in[0,1].
\end{equation}
By the chain rule, $\phi$ is differentiable and
\begin{equation}\label{eq:phi-deriv}
    \phi'(t)
    \ =\
    \bigl\langle \nabla F(\bx+t \bv),\,\bv\bigr\rangle.
\end{equation}

Now fix $0\le s\le t\le 1$. Since $\bv\ge 0$, we have $\bx+s \bv\le \bx+t \bv$, and by
$\gamma$-weak DR-submodularity this implies
\begin{equation}\label{eq:grad-gamma}
    \nabla F(\bx+s \bv)
    \ \ge\
    \gamma\,\nabla F(\bx+t \bv).
\end{equation}
Taking inner products of both sides of \eqref{eq:grad-gamma} with $\bv\ge 0$ and using
\eqref{eq:phi-deriv} gives
\begin{equation}\label{eq:gamma-monotone}
    \phi'(s)
    \ =\
    \bigl\langle \nabla F(\bx+s\bv),\bv\bigr\rangle
    \ \ge\
    \gamma\,\bigl\langle \nabla F(\bx+t\bv),\bv\bigr\rangle
    \ =\
    \gamma\,\phi'(t),
    \qquad 0\le s\le t\le 1.
\end{equation}

Next fix $\lambda\in(0,1)$. For $t\in[\lambda,1]$, applying \eqref{eq:gamma-monotone}
with $s=\lambda$ gives $\phi'(\lambda)\ge \gamma\,\phi'(t)$, so
\[
    \phi'(t)\ \le\ \frac{1}{\gamma}\,\phi'(\lambda).
\]
Integrating this upper bound over $t\in[\lambda,1]$ and using the fundamental theorem of
calculus yields
\begin{equation}\label{eq:segment-upper}
    \phi(1)-\phi(\lambda)
    \ =\
    \int_\lambda^1 \phi'(t)\,dt
    \ \le\
    \int_\lambda^1 \frac{1}{\gamma}\,\phi'(\lambda)\,dt
    \ =\
    \frac{1-\lambda}{\gamma}\,\phi'(\lambda).
\end{equation}

Similarly, for $s\in[0,\lambda]$, applying \eqref{eq:gamma-monotone} with $t=\lambda$
gives $\phi'(s)\ge \gamma\,\phi'(\lambda)$. Integrating this lower bound over
$s\in[0,\lambda]$ we obtain
\begin{equation}\label{eq:segment-lower}
    \phi(\lambda)-\phi(0)
    \ =\
    \int_0^\lambda \phi'(s)\,ds
    \ \ge\
    \int_0^\lambda \gamma\,\phi'(\lambda)\,ds
    \ =\
    \gamma\,\lambda\,\phi'(\lambda).
\end{equation}

From \eqref{eq:segment-upper} we get
\[
    \phi'(\lambda)
    \ \ge\
    \frac{\gamma}{1-\lambda}\,\bigl(\phi(1)-\phi(\lambda)\bigr),
\]
and substituting this lower bound into \eqref{eq:segment-lower} gives
\begin{equation}\label{eq:ratio-phi}
    \phi(\lambda)-\phi(0)
    \ \ge\
    \gamma\,\lambda\cdot
    \frac{\gamma}{1-\lambda}\,\bigl(\phi(1)-\phi(\lambda)\bigr)
    \ =\
    \frac{\gamma^{2}\lambda}{1-\lambda}\,\bigl(\phi(1)-\phi(\lambda)\bigr).
\end{equation}
Multiplying both sides of \eqref{eq:ratio-phi} by $1-\lambda$ and expanding, we obtain
\[
    (1-\lambda)\,\phi(\lambda)- (1-\lambda)\,\phi(0)
    \ \ge\
    \gamma^{2}\lambda\,\phi(1)-\gamma^{2}\lambda\,\phi(\lambda).
\]
Rearranging the terms involving $\phi(\lambda)$ to the left and the remaining terms to
the right yields
\[
    \bigl(1-\lambda+\gamma^{2}\lambda\bigr)\,\phi(\lambda)
    \ \ge\
    (1-\lambda)\,\phi(0)+\gamma^{2}\lambda\,\phi(1).
\]
Dividing by $1-\lambda+\gamma^{2}\lambda>0$ we get
\begin{equation}\label{eq:phi-convex-combo}
    \phi(\lambda)
    \ \ge\
    \frac{(1-\lambda)\,\phi(0)+\gamma^{2}\lambda\,\phi(1)}
         {(1-\lambda)+\gamma^{2}\lambda}.
\end{equation}

To prove part~(2), take $\bv=\by\ge 0$ and $\phi(t)=F(\bx+t\by)$ as in
\eqref{eq:phi-def}. Since $\bx+\by\in[0,1]^n$, the entire segment
$\{\bx+t\by : t\in[0,1]\}$ lies in $[0,1]^n$, so the above argument applies. In this
case,
\[
    \phi(0)=F(\bx),\qquad
    \phi(1)=F(\bx+\by),\qquad
    \phi(\lambda)=F(\bx+\lambda\by).
\]
Substituting these expressions into \eqref{eq:phi-convex-combo} gives
\begin{equation}\label{eq:part2-main}
    F(\bx+\lambda \by)
    \ \ge\
    \frac{(1-\lambda)F(\bx)+\gamma^{2}\lambda F(\bx+\by)}
         {(1-\lambda)+\gamma^{2}\lambda}.
\end{equation}
Subtracting $F(\bx)$ from both sides of \eqref{eq:part2-main}, we obtain
\[
    F(\bx+\lambda\by)-F(\bx)
    \ \ge\
    \frac{(1-\lambda)F(\bx)+\gamma^{2}\lambda F(\bx+\by)}
         {(1-\lambda)+\gamma^{2}\lambda}
    \;-\; F(\bx).
\]
Writing $F(\bx)$ as $\frac{(1-\lambda)+\gamma^{2}\lambda}{(1-\lambda)+\gamma^{2}\lambda}F(\bx)$
and simplifying the numerator, we get
\[
    F(\bx+\lambda\by)-F(\bx)
    \ \ge\
    \frac{\gamma^{2}\lambda}{1-\lambda+\gamma^{2}\lambda}\,\bigl(F(\bx+\by)-F(\bx)\bigr),
\]
which is exactly \eqref{eq:exe-increment-form}. This proves part~(2).

For part~(1), assume $\bx\le \by$ and define $\bv := \by-\bx\ge 0$. For $t\in[0,1]$
set
\begin{equation}\label{eq:phi-convex-path}
    \phi(t)
    \ :=\
    F\bigl(\bx+t(\by-\bx)\bigr)
    \ =\
    F\bigl((1-t)\bx+t\by\bigr).
\end{equation}
Again, the segment between $\bx$ and $\by$ lies in $[0,1]^n$, so
\eqref{eq:phi-convex-combo} applies. Here
\[
    \phi(0)=F(\bx),\qquad
    \phi(1)=F(\by),\qquad
    \phi(\lambda)=F\bigl((1-\lambda)\bx+\lambda\by\bigr).
\]
Substituting into \eqref{eq:phi-convex-combo}, we obtain
\begin{equation}\label{eq:part1-main}
    F\bigl((1-\lambda)\bx+\lambda\by\bigr)
    \ \ge\
    \frac{(1-\lambda)\,F(\bx)+\gamma^{2}\lambda\,F(\by)}
         {(1-\lambda)+\gamma^{2}\lambda},
\end{equation}
which is exactly \eqref{eq:exe-convex-form-2}. Finally, replacing $\lambda$ by $1-\lambda$
in \eqref{eq:part1-main} gives
\[
    F\bigl(\lambda\bx+(1-\lambda)\by\bigr)
    \ \ge\
    \frac{\lambda\,F(\bx)+\gamma^{2}(1-\lambda)\,F(\by)}
         {\lambda+\gamma^{2}(1-\lambda)},
\]
which is \eqref{eq:exe-convex-form-1}. This proves part~(1).
\end{proof}

\begin{exe}
Let $F : [0,1]^n \to \mathbb{R}_{\ge 0}$ be a differentiable $\gamma$-weakly DR-submodular function. 
Then, for every $\bx,\by \in [0,1]^n$ with $\by \ge \bzero$, the following inequalities hold:
\begin{enumerate}
    \item If $\bx+\by \le \bone$, then
    \begin{equation}\label{eq:grad-lower-main}
        \big\langle \nabla F(\bx),\,\by \big\rangle
        \;\ge\;
        \gamma\big(F(\bx+\by) - F(\bx)\big).
    \end{equation}
    \item If $\bx-\by \ge \bzero$, then
    \begin{equation}\label{eq:grad-upper-main}
        \big\langle \nabla F(\bx),\,\by \big\rangle
        \;\le\;
        \frac{1}{\gamma}\big(F(\bx) - F(\bx-\by)\big).
    \end{equation}
\end{enumerate}
\end{exe}

\begin{proof}
We first prove part~(1). Assume $\bx+\by\le\bone$ and define
\begin{equation}\label{eq:g-def}
    g(t) \;:=\; F(\bx+t \by),
    \qquad t\in[0,1].
\end{equation}
The condition $\bx+\by\le\bone$ and $\by\ge\bzero$ implies $\bx+t\by\in[0,1]^n$ for all
$t\in[0,1]$. By the chain rule,
\begin{equation}\label{eq:g-deriv}
    g'(t)
    \;=\;
    \big\langle \nabla F(\bx+t \by),\,\by\big\rangle.
\end{equation}

For each $t\in[0,1]$ we have $\bx\le \bx+t\by$, so by $\gamma$-weak DR-submodularity,
\begin{equation}\label{eq:grad-pointwise-1}
    \nabla F(\bx)
    \;\ge\;
    \gamma\,\nabla F(\bx+t \by),
    \qquad 0\le t\le 1.
\end{equation}
Taking inner products of both sides of \eqref{eq:grad-pointwise-1} with $\by\ge\bzero$ and
using \eqref{eq:g-deriv} gives
\begin{equation}\label{eq:g-prime-bound}
    \big\langle \nabla F(\bx),\,\by\big\rangle
    \;\ge\;
    \gamma\,\big\langle \nabla F(\bx+t\by),\,\by\big\rangle
    \;=\;
    \gamma\,g'(t),
    \qquad 0\le t\le 1.
\end{equation}
Equivalently,
\begin{equation}\label{eq:g-prime-upper}
    g'(t)
    \;\le\;
    \frac{1}{\gamma}\,\big\langle \nabla F(\bx),\,\by\big\rangle,
    \qquad 0\le t\le 1.
\end{equation}

Integrating the bound \eqref{eq:g-prime-upper} over $t\in[0,1]$ and using the fundamental
theorem of calculus yields
\begin{equation}\label{eq:g-integral}
    F(\bx+\by)-F(\bx)
    \;=\;
    \int_0^1 g'(t)\,dt
    \;\le\;
    \int_0^1 \frac{1}{\gamma}\,\big\langle \nabla F(\bx),\,\by\big\rangle\,dt
    \;=\;
    \frac{1}{\gamma}\,\big\langle \nabla F(\bx),\,\by\big\rangle.
\end{equation}
Rearranging \eqref{eq:g-integral} gives
\begin{equation}\label{eq:grad-lower-main-again}
    \big\langle \nabla F(\bx),\,\by\big\rangle
    \;\ge\;
    \gamma\big(F(\bx+\by)-F(\bx)\big),
\end{equation}
which is exactly \eqref{eq:grad-lower-main}.

We now prove part~(2). Assume $\bx-\by\ge\bzero$ and define
\begin{equation}\label{eq:h-def}
    h(t) \;:=\; F(\bx-t \by),
    \qquad t\in[0,1].
\end{equation}
The condition $\bx-\by\ge\bzero$ and $\by\ge\bzero$ implies $\bx-t\by\in[0,1]^n$ for all
$t\in[0,1]$. Again by the chain rule,
\begin{equation}\label{eq:h-deriv}
    h'(t)
    \;=\;
    \big\langle \nabla F(\bx-t \by),\, -\by\big\rangle
    \;=\;
    -\,\big\langle \nabla F(\bx-t \by),\,\by\big\rangle.
\end{equation}

For each $t\in[0,1]$ we have $\bx-t\by\le \bx$, so $\gamma$-weak DR-submodularity gives
\begin{equation}\label{eq:grad-pointwise-2}
    \nabla F(\bx-t \by)
    \;\ge\;
    \gamma\,\nabla F(\bx),
    \qquad 0\le t\le 1.
\end{equation}
Taking inner products with $\by\ge\bzero$ and using \eqref{eq:h-deriv} we obtain
\begin{equation}\label{eq:inner-ineq-2}
    \big\langle \nabla F(\bx-t \by),\,\by\big\rangle
    \;\ge\;
    \gamma\,\big\langle \nabla F(\bx),\,\by\big\rangle,
    \qquad 0\le t\le 1,
\end{equation}
and hence
\begin{equation}\label{eq:hprime-bound}
    h'(t)
    \;=\;
    -\,\big\langle \nabla F(\bx-t \by),\,\by\big\rangle
    \;\le\;
    -\,\gamma\,\big\langle \nabla F(\bx),\,\by\big\rangle,
    \qquad 0\le t\le 1.
\end{equation}

Integrating \eqref{eq:hprime-bound} over $t\in[0,1]$ and applying the fundamental theorem
of calculus gives
\begin{equation}\label{eq:h-integral}
    F(\bx-\by)-F(\bx)
    \;=\;
    \int_0^1 h'(t)\,dt
    \;\le\;
    \int_0^1 -\,\gamma\,\big\langle \nabla F(\bx),\,\by\big\rangle\,dt
    \;=\;
    -\,\gamma\,\big\langle \nabla F(\bx),\,\by\big\rangle.
\end{equation}
Multiplying \eqref{eq:h-integral} by $-1$ yields
\begin{equation}\label{eq:grad-upper-ineq}
    F(\bx)-F(\bx-\by)
    \;\ge\;
    \gamma\,\big\langle \nabla F(\bx),\,\by\big\rangle.
\end{equation}
Rearranging \eqref{eq:grad-upper-ineq}, we obtain
\begin{equation}\label{eq:grad-upper-main-again}
    \big\langle \nabla F(\bx),\,\by\big\rangle
    \;\le\;
    \frac{1}{\gamma}\big(F(\bx)-F(\bx-\by)\big),
\end{equation}
which is exactly \eqref{eq:grad-upper-main}. This completes the proof.
\end{proof}

\begin{exe}
Let $F:[0,1]^n\to\mathbb{R}_{\ge0}$ be nonnegative and $\gamma$–weakly DR-submodular for some
$\gamma\in(0,1]$. For any fixed $\by\in[0,1]^n$, define
\begin{equation}\label{eq:Gplus-def}
G_{\oplus}(\bx)\ :=\ F(\bx\oplus \by)
\end{equation}
and
\begin{equation}\label{eq:Gdot-def}
G_{\odot}(\bx)\ :=\ F(\bx\odot \by).
\end{equation}
Then both $G_{\oplus}$ and $G_{\odot}$ are nonnegative and $\gamma$–weakly DR-submodular, that is,
for all $\bx^{(1)},\bx^{(2)}\in[0,1]^n$ with $\bx^{(1)}\le \bx^{(2)}$, any coordinate $u\in[n]$,
and any step $p\in[0,1-\bx^{(2)}_u]$ such that the updates stay in $[0,1]^n$, we have
\begin{equation}\label{eq:Gplus-DR}
G_{\oplus}\bigl(\bx^{(1)}+p\,\be_u\bigr)-G_{\oplus}(\bx^{(1)})
\;\ge\;
\gamma\Bigl(G_{\oplus}\bigl(\bx^{(2)}+p\,\be_u\bigr)-G_{\oplus}(\bx^{(2)})\Bigr)
\end{equation}
and
\begin{equation}\label{eq:Gdot-DR}
G_{\odot}\bigl(\bx^{(1)}+p\,\be_u\bigr)-G_{\odot}(\bx^{(1)})
\;\ge\;
\gamma\Bigl(G_{\odot}\bigl(\bx^{(2)}+p\,\be_u\bigr)-G_{\odot}(\bx^{(2)})\Bigr).
\end{equation}
\end{exe}

\begin{proof}
Nonnegativity of $G_{\oplus}$ and $G_{\odot}$ follows directly from \eqref{eq:Gplus-def},
\eqref{eq:Gdot-def} and the nonnegativity of $F$.

Fix $\bx^{(1)}, \bx^{(2)}\in[0,1]^n$ with $\bx^{(1)}\le \bx^{(2)}$, a coordinate
$u\in[n]$, and a step $p\in[0,\,1-\bx^{(2)}_u]$, so that $\bx^{(j)}+p\,\be_u\in[0,1]^n$
for $j=1,2$.

We first treat the $\oplus$ case. Set
\begin{equation}\label{eq:z-def}
    \bz^{(j)} \;:=\; \bx^{(j)}\oplus \by,
    \qquad j\in\{1,2\}.
\end{equation}
Since $\bx^{(1)}\le\bx^{(2)}$ and the map $x\mapsto x\oplus y$ is nondecreasing in $x$,
we have
\begin{equation}\label{eq:z-order}
    \bz^{(1)} \;\le\; \bz^{(2)}.
\end{equation}
Using the coordinate-wise identity
\begin{equation}\label{eq:oplus-coord}
    (\bx\oplus\by)_i
    \;=\;
    x_i + y_i - x_i y_i,
\end{equation}
we can express the update $(\bx^{(j)}+p\,\be_u)\oplus\by$ in terms of $\bz^{(j)}$.
Indeed, only the $u$-th coordinate of $\bx^{(j)}$ changes, so for each $j\in\{1,2\}$ we have
\begin{equation}\label{eq:Gplus-update}
    (\bx^{(j)}+p\,\be_u)\oplus \by
    \;=\;
    \bz^{(j)} + \alpha\,\be_u,
\end{equation}
where
\begin{equation}\label{eq:alpha-def}
    \alpha \;:=\; p\,(1-\by_u).
\end{equation}
This follows by plugging $x_u^{(j)}+p$ into the expression \eqref{eq:oplus-coord} and
simplifying.

Next we check that the updated point stays in $[0,1]^n$ at coordinate $u$. Using
\eqref{eq:oplus-coord} with $x_u^{(2)}$ we obtain
\begin{equation}\label{eq:one-minus-z2}
    1-\bz^{(2)}_u
    \;=\;
    1-\bigl(x^{(2)}_u + y_u - x^{(2)}_u y_u\bigr)
    \;=\;
    (1-\bx^{(2)}_u)(1-\by_u).
\end{equation}
Since $p\le 1-\bx^{(2)}_u$ by assumption and $1-\by_u\ge 0$, we get
\begin{equation}\label{eq:alpha-bound-plus}
    \alpha
    \;=\;
    p(1-\by_u)
    \;\le\;
    (1-\bx^{(2)}_u)(1-\by_u)
    \;=\;
    1-\bz^{(2)}_u.
\end{equation}
Thus $\bz^{(2)}+\alpha\,\be_u$ remains in $[0,1]^n$.

Now we use the $\gamma$–weak DR-submodularity of $F$. From \eqref{eq:z-order} and
\eqref{eq:alpha-bound-plus} we can apply the definition of $\gamma$–weak DR-submodularity
to the pair $(\bz^{(1)},\bz^{(2)})$ with step $\alpha$ in coordinate $u$ and obtain
\begin{equation}\label{eq:F-DR-plus}
    F\!\bigl(\bz^{(1)}+\alpha\,\be_u\bigr) - F(\bz^{(1)})
    \;\ge\;
    \gamma\Bigl(
        F\!\bigl(\bz^{(2)}+\alpha\,\be_u\bigr) - F(\bz^{(2)})
    \Bigr).
\end{equation}
Using \eqref{eq:Gplus-def} and \eqref{eq:Gplus-update}, we can rewrite the left-hand side
and the right-hand side of \eqref{eq:F-DR-plus} in terms of $G_{\oplus}$ as
\begin{equation}\label{eq:Gplus-rewrite}
    F\!\bigl(\bz^{(j)}+\alpha\,\be_u\bigr)
    \;=\;
    G_{\oplus}\!\bigl(\bx^{(j)}+p\,\be_u\bigr),
    \qquad
    F(\bz^{(j)}) = G_{\oplus}(\bx^{(j)}),
    \quad j=1,2.
\end{equation}
Substituting \eqref{eq:Gplus-rewrite} into \eqref{eq:F-DR-plus} gives
\begin{equation}\label{eq:Gplus-DR-proof}
    G_{\oplus}\!\bigl(\bx^{(1)}+p\,\be_u\bigr)-G_{\oplus}(\bx^{(1)})
    \;\ge\;
    \gamma\Bigl(
        G_{\oplus}\!\bigl(\bx^{(2)}+p\,\be_u\bigr)-G_{\oplus}(\bx^{(2)})
    \Bigr),
\end{equation}
which is exactly \eqref{eq:Gplus-DR}. This shows that $G_{\oplus}$ is
$\gamma$–weakly DR-submodular.

We now treat the $\odot$ case. Set
\begin{equation}\label{eq:w-def}
    \bw^{(j)} \;:=\; \bx^{(j)}\odot \by,
    \qquad j\in\{1,2\}.
\end{equation}
Since the map $x\mapsto x\odot y$ is also nondecreasing in $x$, we again have
\begin{equation}\label{eq:w-order}
    \bw^{(1)} \;\le\; \bw^{(2)}.
\end{equation}
By definition of $\odot$,
\begin{equation}\label{eq:odot-coord}
    (\bx\odot\by)_i
    \;=\;
    x_i y_i,
\end{equation}
so updating $\bx^{(j)}$ in coordinate $u$ by $p$ gives
\begin{equation}\label{eq:Gdot-update}
    (\bx^{(j)}+p\,\be_u)\odot \by
    \;=\;
    \bw^{(j)} + \beta\,\be_u,
\end{equation}
where
\begin{equation}\label{eq:beta-def}
    \beta \;:=\; p\,\by_u.
\end{equation}
Again we check that the updated point stays in $[0,1]^n$ at coordinate $u$.
From \eqref{eq:odot-coord} we have
\begin{equation}\label{eq:w2-coord}
    \bw^{(2)}_u \;=\; \bx^{(2)}_u \by_u,
\end{equation}
so
\begin{equation}\label{eq:one-minus-w2}
    1-\bw^{(2)}_u
    \;=\;
    1-\bx^{(2)}_u \by_u.
\end{equation}
Using $p\le 1-\bx^{(2)}_u$ and $\by_u\le 1$, we obtain
\begin{equation}\label{eq:beta-bound}
    \beta
    \;=\;
    p\,\by_u
    \;\le\;
    (1-\bx^{(2)}_u)\by_u
    \;\le\;
    1-\bx^{(2)}_u \by_u
    \;=\;
    1-\bw^{(2)}_u.
\end{equation}
Hence $\bw^{(2)}+\beta\,\be_u\in[0,1]^n$.

Now we again use the $\gamma$–weak DR-submodularity of $F$. By
\eqref{eq:w-order} and \eqref{eq:beta-bound}, the definition applied to
$(\bw^{(1)},\bw^{(2)})$ with step $\beta$ in coordinate $u$ yields
\begin{equation}\label{eq:F-DR-dot}
    F\!\bigl(\bw^{(1)}+\beta\,\be_u\bigr)-F(\bw^{(1)})
    \;\ge\;
    \gamma\Bigl(
        F\!\bigl(\bw^{(2)}+\beta\,\be_u\bigr)-F(\bw^{(2)})
    \Bigr).
\end{equation}
Using \eqref{eq:Gdot-def} and \eqref{eq:Gdot-update}, we can rewrite
\eqref{eq:F-DR-dot} as
\begin{equation}\label{eq:Gdot-DR-proof}
    G_{\odot}\!\bigl(\bx^{(1)}+p\,\be_u\bigr)-G_{\odot}(\bx^{(1)})
    \;\ge\;
    \gamma\Bigl(
        G_{\odot}\!\bigl(\bx^{(2)}+p\,\be_u\bigr)-G_{\odot}(\bx^{(2)})
    \Bigr),
\end{equation}
which is exactly \eqref{eq:Gdot-DR}. Thus $G_{\odot}$ is also $\gamma$–weakly
DR-submodular, completing the proof.
\end{proof}

\section{Frank--Wolfe Algorithm and Proof of Theorem~\ref{thm:weak-dr-smooth}}\label{Frank}

\newtheorem{exee}{Lemma}
\renewcommand{\theexee}{3.\arabic{exee}}  
\setcounter{exee}{0}

This section develops a first–order certificate tailored to the $\gamma$–weakly DR setting and uses it to prove our main result, Theorem~\ref{thm:weak-dr-smooth}. 
The argument follows a \emph{local-to-global} template: (i) a local optimality condition at $\bx$ yields a $\gamma$–weighted comparison between $F(\bx)$ and the join/meet values with any comparator $\by$ (Lemma~\ref{lemma_2}); (ii) a Frank--Wolfe variant produces a point $\bx\in P$ that satisfies a \emph{uniform} first–order certificate against \emph{every} $\by\in P$ (Lemma~\ref{lem:first-order-cert}); and (iii) combining the two delivers a global value bound that degrades smoothly with $\gamma$ and matches the classical DR guarantee at $\gamma=1$.

\paragraph{Algorithmic setup.}
We will invoke the following Frank--Wolfe–type routine from \cite{buchbinder2024constrained}. 
For clarity of presentation, we assume $\delta^{-1}\in\mathbb{N}$; if not, we replace $\delta$ by $1/\lceil\delta^{-1}\rceil$ without affecting the asymptotics.

\begin{algorithm}[H]
\caption{Frank--Wolfe Variant($F, P, \delta$)}\label{algo:frank}
\begin{algorithmic}[1]
\State Let $\bx{(0)}$ be an arbitrary vector in $P$.
\For{$i = 1$ to $\delta^{-2}$}
    \State Let $\bz{(i)} \in \arg\max_{\by \in P} \langle \by, \nabla F(\bx{(i-1)}) \rangle$.
    \State Let $\bx{(i)} \gets (1-\delta)\,\bx{(i-1)} + \delta \,\bz{(i)}$.
\EndFor
\State Let $i^{*} \in \arg\min_{1 \leq i \leq \delta^{-2}} \{ \langle \bz{(i)} - \bx{(i-1)}, \nabla F(\bx{(i-1)}) \rangle \}$.
\State \Return $\bx{(i^{*}-1)}$.
\end{algorithmic}
\label{algo:FW_simple}
\end{algorithm}

As observed in \cite{buchbinder2024constrained}, the update rule
\[
\bx{(i)} \;=\; (1-\delta)\,\bx{(i-1)} + \delta\,\bz{(i)}, \qquad \bz{(i)}\in P,
\]
keeps all iterates in $P$; in particular, $\bx{(i)}\in P$ for every
$0\le i\le \delta^{-2}$.

The next lemma converts the local optimality condition into a lattice-based comparison
that interpolates in $\gamma$; it coincides with the classical
$\tfrac12\!\big(F(\bx\vee\by) +F(\bx\wedge\by)\big)$ bound at $\gamma=1$.

\begin{exee}\label{lem:local-to-lattice}
Let $F:[0,1]^n\to\mathbb{R}_{\ge 0}$ be differentiable and $\gamma$–weakly DR-submodular. 
If $\bx$ is a local optimum with respect to $\by$, it means;
\begin{equation}\label{eq:local-opt-cond}
    \big\langle \by-\bx,\nabla F(\bx)\big\rangle\;\le\;0,
\end{equation}
then
\begin{equation}\label{eq:local-to-lattice}
    F(\bx)\ \ge\ \frac{\gamma^{2}\,F(\bx \vee \by) + F(\bx \wedge \by)}{\,1+\gamma^{2}\,}.
\end{equation}
\end{exee}

\begin{proof}
Starting from the local optimality condition \eqref{eq:local-opt-cond}, we decompose
$\by-\bx$ as
\begin{equation}\label{eq:decomp-y-x}
    \by-\bx
    \;=\;
    (\by\vee\bx-\bx)\;-\;(\bx-\by\wedge\bx).
\end{equation}
Substituting \eqref{eq:decomp-y-x} into \eqref{eq:local-opt-cond} gives
\begin{equation}\label{eq:local-opt-expanded}
    0
    \;\ge\;
    \big\langle \by-\bx,\nabla F(\bx)\big\rangle
    \;=\;
    \big\langle \by\vee\bx-\bx,\nabla F(\bx)\big\rangle
    \;-\;
    \big\langle \bx-\by\wedge\bx,\nabla F(\bx)\big\rangle.
\end{equation}

We now bound each inner product using Lemma~\ref{lem:grad-ineq}.
First, note that $\by\vee\bx-\bx\ge\bzero$ and
\[
    \bx + (\by\vee\bx-\bx)
    \;=\;
    \by\vee\bx
    \;\in\;
    [0,1]^n,
\]
so Lemma~\ref{lem:grad-ineq}(1) applies with the direction
$\by\vee\bx-\bx$. We obtain
\begin{equation}\label{eq:term1-lower}
    \big\langle \nabla F(\bx),\,\by\vee\bx-\bx\big\rangle
    \;\ge\;
    \gamma\big(F(\by\vee\bx)-F(\bx)\big).
\end{equation}

Similarly, $\bx-\by\wedge\bx\ge\bzero$ and
\[
    \bx-(\bx-\by\wedge\bx)
    \;=\;
    \by\wedge\bx
    \;\in\;
    [0,1]^n,
\]
so Lemma~\ref{lem:grad-ineq}(2) applies with the direction
$\bx-\by\wedge\bx$. This gives
\begin{equation}\label{eq:term2-upper}
    \big\langle \nabla F(\bx),\,\bx-\by\wedge\bx\big\rangle
    \;\le\;
    \frac{1}{\gamma}\big(F(\bx)-F(\by\wedge\bx)\big).
\end{equation}

Substituting the bounds \eqref{eq:term1-lower} and \eqref{eq:term2-upper} into
\eqref{eq:local-opt-expanded} yields
\begin{equation}\label{eq:ineq-before-rearrange}
    0
    \;\ge\;
    \gamma\big(F(\by\vee\bx)-F(\bx)\big)
    \;-\;
    \frac{1}{\gamma}\big(F(\bx)-F(\by\wedge\bx)\big).
\end{equation}
Expanding \eqref{eq:ineq-before-rearrange}, we obtain
\begin{equation}\label{eq:expanded-ineq}
    0
    \;\ge\;
    \gamma F(\by\vee\bx)
    \;-\;
    \gamma F(\bx)
    \;-\;
    \frac{1}{\gamma}F(\bx)
    \;+\;
    \frac{1}{\gamma}F(\by\wedge\bx).
\end{equation}
Rearranging \eqref{eq:expanded-ineq} by bringing the terms involving $F(\bx)$ to the
right-hand side gives
\begin{equation}\label{eq:move-Fx}
    \gamma F(\by\vee\bx)
    \;+\;
    \frac{1}{\gamma}F(\by\wedge\bx)
    \;\le\;
    \left(\gamma+\frac{1}{\gamma}\right)F(\bx).
\end{equation}
Dividing both sides of \eqref{eq:move-Fx} by $\gamma+1/\gamma=(1+\gamma^{2})/\gamma$
yields
\begin{equation}\label{eq:Fx-lattice-form}
    F(\bx)
    \;\ge\;
    \frac{\gamma F(\by\vee\bx) + \tfrac{1}{\gamma}F(\by\wedge\bx)}
         {\gamma+\tfrac{1}{\gamma}}
    \;=\;
    \frac{\gamma^{2}F(\by\vee\bx) + F(\by\wedge\bx)}{1+\gamma^{2}}.
\end{equation}
This is exactly \eqref{eq:local-to-lattice}, after noting that $\by\vee\bx=\bx\vee\by$
and $\by\wedge\bx=\bx\wedge\by$.
\end{proof}

The next lemma is a standard smoothness-based guarantee produced by Algorithm~\ref{algo:frank}
(see Theorem~2.4 in \cite{buchbinder2024constrained}).

\begin{lemma}[Theorem~2.4 of \cite{buchbinder2024constrained}]\label{lem:first-order-cert}
Let $F:[0,1]^n\to\mathbb{R}_{\ge 0}$ be nonnegative and $L$-smooth, let
$P\subseteq[0,1]^n$ be a solvable convex body of diameter $D$, and let
$\delta\in(0,1)$. There is a polynomial-time algorithm that returns $\bx\in P$
such that
\begin{equation}\label{eq:first-order-cert}
\big\langle \by-\bx,\nabla F(\bx)\big\rangle \;\le\;
\delta\!\left[\max_{\bz\in P} F(\bz) \,+\, \frac{L D^{2}}{2}\right]
\qquad\text{for all } \by\in P.
\end{equation}
\end{lemma}

Combining the uniform certificate \eqref{eq:first-order-cert} with the weakly–DR
inequalities (Lemma~\ref{lem:grad-ineq}) and the local-to-lattice comparison
(Lemma~\ref{lem:local-to-lattice}) yields our main bound.

\newtheorem{exew}{Theorem}
\renewcommand{\theexew}{3.\arabic{exew}}  
\setcounter{exew}{1}

\begin{exew}\label{thm:weak-dr-smooth_a}
Let $F:[0,1]^n\to\mathbb{R}_{\ge 0}$ be nonnegative and $L$-smooth, and suppose $F$ is
$\gamma$–weakly DR-submodular for some $\gamma\in(0,1]$. Let
$P\subseteq[0,1]^n$ be a solvable convex body of diameter $D$, and let $\delta\in(0,1)$.
Then there is a polynomial-time algorithm that outputs $\bx\in P$ such that, for every
$\by\in P$,
\begin{equation}\label{eq:weak-dr-guarantee}
F(\bx)\;\ge\;
\frac{\gamma^{2} F(\bx\vee \by)+F(\bx\wedge \by)}{1+\gamma^{2}}
\;-\;
\frac{\delta\,\gamma}{1+\gamma^{2}}\!\left[\max_{\bz\in P} F(\bz) \,+\, \frac{L D^{2}}{2}\right].
\end{equation}
\end{exew}

\begin{proof}
Let $\bx\in P$ be returned by Lemma~\ref{lem:first-order-cert}; then
\eqref{eq:first-order-cert} holds for all $\by\in P$. As in the proof of
Lemma~\ref{lem:local-to-lattice}, Lemma~\ref{lem:grad-ineq} implies that for every
$\by\in[0,1]^n$,
\begin{equation}\label{eq:grad-ineq-global}
\big\langle \by-\bx,\nabla F(\bx)\big\rangle \;\ge\; 
\gamma\,F(\bx\vee \by) \;+\; \frac{1}{\gamma}\,F(\bx\wedge \by)
\;-\; \frac{1+\gamma^{2}}{\gamma}\,F(\bx).
\end{equation}
Combining \eqref{eq:first-order-cert} and \eqref{eq:grad-ineq-global} for any
$\by\in P$ gives
\begin{equation}\label{eq:combine-cert}
\delta\!\left[\max_{\bz\in P} F(\bz) \,+\, \frac{L D^{2}}{2}\right]
\;\ge\;
\gamma\,F(\bx\vee \by)
\;+\; \frac{1}{\gamma}\,F(\bx\wedge \by)
\;-\; \frac{1+\gamma^{2}}{\gamma}\,F(\bx).
\end{equation}
Rearranging \eqref{eq:combine-cert} by moving the $F(\bx)$ term to the left-hand side,
we obtain
\begin{equation}\label{eq:move-Fx-theorem}
\frac{1+\gamma^{2}}{\gamma}\,F(\bx)
\;\ge\;
\gamma\,F(\bx\vee \by)
\;+\; \frac{1}{\gamma}\,F(\bx\wedge \by)
\;-\;
\delta\!\left[\max_{\bz\in P} F(\bz) \,+\, \frac{L D^{2}}{2}\right].
\end{equation}
Multiplying both sides of \eqref{eq:move-Fx-theorem} by
$\frac{\gamma}{1+\gamma^{2}}$ yields
\begin{align}\label{eq:Fx-final}
F(\bx)
&\;\ge\;
\frac{\gamma^{2}}{1+\gamma^{2}}\,F(\bx\vee\by)
\;+\;
\frac{1}{1+\gamma^{2}}\,F(\bx\wedge\by)
\;-\;
\frac{\delta\,\gamma}{1+\gamma^{2}}\!\left[\max_{\bz\in P} F(\bz) \,+\, \frac{L D^{2}}{2}\right].
\end{align}
This is exactly \eqref{eq:weak-dr-guarantee}, completing the proof.
\end{proof}

\section{Double--Greedy Algorithm and Proof of Theorem~\ref{thm:unbalanced-weakDR} }\label{Double-Greedy}

This appendix develops and analyzes a $\gamma$-aware Double--Greedy routine whose guarantee is stated in Theorem~\ref{thm:double_appen}. Our analysis uses the grid-discretized variant in Algorithm~\ref{algo:double}. For concreteness, we assume that $\varepsilon^{-1}$ is an even integer; if not, we replace $\varepsilon$ by
\(
\varepsilon' \;=\; {1}/{\,(2\left\lceil \varepsilon^{-1}\right\rceil)\,}\;\in (0,1],
\)
which leaves the bounds unchanged up to the stated $\varepsilon$-dependence. Throughout, we write $\bo\in[0,1]^n$ for an arbitrary comparator; when $\bo$ is chosen to be a maximizer, we have $F(\bo)=\max_{\bu\in[0,1]^n}F(\bu)$.

\begin{exew}\label{thm:double_appen}
There exists a polynomial-time algorithm that, given a nonnegative $\gamma$-weakly DR-submodular
function $F : [0,1]^n \to \mathbb{R}_{\ge 0}$ and a parameter $\varepsilon \in (0,1)$, outputs
$\bx \in [0,1]^n$ such that, for every fixed $\bo \in [0,1]^n$,
\begin{equation}
    F(\bx)\;\;\ge\;\; 
\max_{r \ge 0}\;
\frac{\bigl(2\gamma^{3/2}-4\varepsilon\,\gamma^{9/2}\bigr)\,r\,F(\bo)\;+\;F(\mathbf{0})\;+\;r^2\,F(\mathbf{1})}
{\,r^2\;+\;2\gamma^{3/2}r\;+\;1\,}.
\end{equation}
\end{exew}

Algorithm~\ref{algo:double} maintains two vectors $\bx,\by\in[0,1]^n$ that start at $\mathbf{0}$ and $\mathbf{1}$, respectively, and \emph{monotonically converge} to a single vector by making one coordinate agree per iteration. The new value assigned to the chosen coordinate (in both $\bx$ and $\by$) is taken from a uniform grid; we denote the grid by
\[
V \;:=\; \Bigl\{\, j\,\frac{\varepsilon}{n}\ :\ j\in\mathbb{Z},\ 0\le j \le n\,\varepsilon^{-1}\Bigr\}\ \subseteq\ [0,1].
\]
As shown in the lemmas that follow, the discretization loss due to using $V$ is explicitly controlled, and the resulting lower bound interpolates continuously with~$\gamma$ and specializes to the classical DR guarantee at $\gamma=1$.

\begin{algorithm}[H]
\caption{$\gamma$-Double-Greedy}
\begin{algorithmic}[1]
\State \textbf{Input:} oracle access to $F:[0,1]^n \to \mathbb{R}_{\ge 0}$, ground set $N$, grid parameter $\varepsilon\in(0,1)$, weakly-DR parameter $\gamma\in(0,1]$.
\State Let $V \gets \left\{ \frac{j\varepsilon}{n} \;:\; j\in\mathbb{Z},\ 0 \le j \le n\varepsilon^{-1} \right\} \subseteq [0,1]$. 
\State Denote the elements of $N$ by $u_1,\cdots,u_n$ in an arbitrary order.
\State Let $\bx \gets \mathbf{0}$ and $\by \gets \mathbf{1}$.
\For{$i = 1$ to $n$}
    \State $a_i \in \arg\max_{v\in V} F\left(\bx + v\,\mathbf{1}_{u_i}\right)$; \quad $\Delta_{a,i} \gets F\left(\bx + a_i\,\mathbf{1}_{u_i}\right) - F(\bx)$.
    \State $b_i \in \arg\max_{v\in V} F\left(\by - v\mathbf{1}_{u_i}\right)$; \quad $\Delta_{b,i} \gets F\left(\by - b_i\,\mathbf{1}_{u_i}\right) - F(\by)$.
    \If{$\Delta_{a,i} + \gamma\,\Delta_{b,i} > 0$}
        \State $w_i \gets \dfrac{\Delta_{a,i}\,a_i + \gamma\,\Delta_{b,i}\,(1-b_i)}{\Delta_{a,i} + \gamma\,\Delta_{b,i}}$.
    \Else{$\Delta_{a,i} = 0$ \textbf{and} $\Delta_{b,i} = 0$} 
        \State $w_i \gets 1 - b_i$ 
    \EndIf
    \State Set $x_{u_i} \gets w_i$ and $y_{u_i} \gets w_i$.
\EndFor
\State \Return $\bx$ 
\end{algorithmic}
\label{algo:double}
\end{algorithm}

We now quantify the value attained by the vector returned by Algorithm~\ref{algo:double}.
Let \(\bx{(i)}\) and \(\by{(i)}\) denote the values of \(\bx\) and \(\by\) after \(i\) iterations of the main loop, respectively.
For a fixed coordinate \(u_i\), let
\begin{equation}\label{eq:vi-star-def}
v^\ast \;\in\; \arg\max_{v\in[0,1]} \; F\bigl(\bx{(i-1)} + v\,\be_{u_i}\bigr)
\end{equation}
be a continuous (unconstrained-by-grid) maximizer along the \(u_i\)-th coordinate direction at iteration \(i\).
The next lemma bounds the discretization loss of the grid choice \(a_i\) used by Algorithm~\ref{algo:double}.
It holds for every \(\gamma\in(0,1]\) and, when \(\gamma=1\), it matches the exact bound of \cite{buchbinder2024constrained}.

\begin{lemma}\label{lem:Double_1}
For any integer \(i\in\{1,\dots,n\}\), the following holds:
\begin{equation}\label{eq:double-case1}
\text{If } v^\ast \ge \tfrac{1}{2},\quad
F\!\bigl(\bx{(i-1)} + a_i\,\be_{u_i}\bigr)
\;\ge\;
\max_{v\in[0,1]} F\!\bigl(\bx{(i-1)} + v\,\be_{u_i}\bigr)
\;-\; \frac{2\varepsilon}{\gamma^{2} n}\,F(\bo).
\end{equation}
\begin{equation}\label{eq:double-case2}
\text{If } v^\ast < \tfrac{1}{2},\quad
F\!\bigl(\bx{(i-1)} + a_i\,\be_{u_i}\bigr)
\;\ge\;
\max_{v\in[0,1]} F\!\bigl(\bx{(i-1)} + v\,\be_{u_i}\bigr)
\;-\; \frac{2\varepsilon}{n}\,\gamma^{2}\,F(\bo).
\end{equation}
\end{lemma}

\begin{proof}
Let \(v^\ast\in[0,1]\) maximize the function 
\begin{equation}\label{eq:phi-direction-def}
v\;\longmapsto\; F\bigl(\bx{(i-1)}+v\,\be_{u_i}\bigr),
\end{equation}
so that \eqref{eq:vi-star-def} holds. We treat two cases depending on the size of \(v^\ast\).

\medskip
\noindent\textbf{Case 1: $v^\ast\ge \tfrac12$.}
Let \(v\in V\) be the largest grid point with \(v\le v^\ast\).
By the definition of \(a_i\) as a maximizer over the grid, we have
\begin{equation}\label{eq:case1-grid-choice}
F\bigl(\bx{(i-1)} + v^\ast \be_{u_i}\bigr) - F\bigl(\bx{(i-1)} + a_i \be_{u_i}\bigr)
\ \le\
F\bigl(\bx{(i-1)} + v^\ast \be_{u_i}\bigr) - F\bigl(\bx{(i-1)} + v \be_{u_i}\bigr),
\end{equation}
since \(F(\bx{(i-1)}+a_i\be_{u_i})\) is at least the value at any other grid point, and in particular at \(v\).

Define the univariate function
\begin{equation}\label{eq:phi-def-double}
\phi(t)\;:=\;F\bigl(\bx{(i-1)} + t\,\be_{u_i}\bigr), \qquad t\in[0,1].
\end{equation}
By differentiability and the chain rule,
\begin{equation}\label{eq:phi-deriv-double}
\phi'(t)\;=\;\big\langle \nabla F\bigl(\bx{(i-1)} + t\,\be_{u_i}\bigr),\,\be_{u_i}\big\rangle.
\end{equation}
Along this coordinate direction, $\gamma$–weak DR-submodularity implies that for any
\(0\le s\le t\le 1\),
\begin{equation}\label{eq:gamma-monotone-double}
\phi'(s)\;\ge\;\gamma\,\phi'(t).
\end{equation}

Applying \eqref{eq:gamma-monotone-double} with \(s=v\) and \(t\in[v,v^\ast]\) yields
\begin{equation}\label{eq:phi-prime-upper-case1}
\phi'(t)\;\le\;\frac{1}{\gamma}\,\phi'(v), \qquad t\in[v,v^\ast].
\end{equation}
Integrating \eqref{eq:phi-prime-upper-case1} over \(t\in[v,v^\ast]\) gives
\begin{equation}\label{eq:phi-diff-case1-1}
\phi(v^\ast)-\phi(v)
\;=\;\int_{v}^{v^\ast}\phi'(t)\,dt
\;\le\;
\int_{v}^{v^\ast}\frac{1}{\gamma}\,\phi'(v)\,dt
\;=\;
\frac{v^\ast-v}{\gamma}\,\phi'(v).
\end{equation}

Similarly, applying \eqref{eq:gamma-monotone-double} with \(0\le s\le v\) and \(t=v\) gives
\begin{equation}\label{eq:phi-prime-lower-case1}
\phi'(s)\;\ge\;\gamma\,\phi'(v), \qquad s\in[0,v].
\end{equation}
Integrating \eqref{eq:phi-prime-lower-case1} over \(s\in[0,v]\) yields
\begin{equation}\label{eq:phi-diff-case1-2}
\phi(v)-\phi(0)
\;=\;\int_{0}^{v}\phi'(s)\,ds
\;\ge\;
\int_{0}^{v}\gamma\,\phi'(v)\,ds
\;=\;
v\,\gamma\,\phi'(v).
\end{equation}
Rearranging \eqref{eq:phi-diff-case1-2} gives
\begin{equation}\label{eq:phi-prime-bound-v}
\phi'(v)\;\le\;\frac{\phi(v)-\phi(0)}{\gamma\,v}.
\end{equation}

Combining \eqref{eq:phi-diff-case1-1} and \eqref{eq:phi-prime-bound-v}, we obtain
\begin{equation}\label{eq:phi-diff-case1-combined}
\phi(v^\ast)-\phi(v)
\;\le\;
\frac{v^\ast-v}{\gamma}\cdot\frac{\phi(v)-\phi(0)}{\gamma\,v}
\;=\;
\frac{v^\ast-v}{\gamma^{2}v}\,\bigl[\phi(v)-\phi(0)\bigr].
\end{equation}

We now bound the factor \(\phi(v)-\phi(0)\). Since \(\phi(v)=F(\bx{(i-1)}+v\,\be_{u_i})\)
and \(\phi(0)=F(\bx{(i-1)})\), nonnegativity of \(F\) and optimality of \(\bo\) imply
\begin{equation}\label{eq:phi-diff-upper-F0}
\phi(v)-\phi(0)
\;\le\;
F(\bo),
\end{equation}
because \(\phi(v)\le F(\bo)\) and \(\phi(0)\ge 0\).
Substituting \eqref{eq:phi-diff-upper-F0} into \eqref{eq:phi-diff-case1-combined} gives
\begin{equation}\label{eq:phi-diff-case1-final}
\phi(v^\ast)-\phi(v)
\;\le\;
\frac{v^\ast-v}{\gamma^{2}v}\,F(\bo).
\end{equation}

By construction of the grid \(V\), we have
\begin{equation}\label{eq:grid-gap-case1}
v^\ast-v\;\le\;\frac{\varepsilon}{n}.
\end{equation}
Moreover, since \(v^\ast\ge\tfrac{1}{2}\) and \(\tfrac{1}{2}\in V\), the choice of \(v\) as
the largest grid point not exceeding \(v^\ast\) implies
\begin{equation}\label{eq:v-lower-bound-case1}
v\;\ge\;\tfrac12.
\end{equation}
Combining \eqref{eq:phi-diff-case1-final}, \eqref{eq:grid-gap-case1}, and
\eqref{eq:v-lower-bound-case1}, we obtain
\begin{equation}\label{eq:phi-diff-case1-bound}
\phi(v^\ast)-\phi(v)
\;\le\;
\frac{\varepsilon/n}{\gamma^{2}\cdot (1/2)}\,F(\bo)
\;=\;
\frac{2\varepsilon}{\gamma^{2}n}\,F(\bo).
\end{equation}

Using \eqref{eq:case1-grid-choice} and \eqref{eq:phi-def-double},
\eqref{eq:phi-diff-case1-bound} gives
\begin{equation}\label{eq:case1-final}
F\bigl(\bx{(i-1)} + v^\ast \be_{u_i}\bigr)
-
F\bigl(\bx{(i-1)} + a_i \be_{u_i}\bigr)
\;\le\;
\frac{2\varepsilon}{\gamma^{2}n}\,F(\bo).
\end{equation}
Since \(v^\ast\) maximizes \(v\mapsto F(\bx{(i-1)}+v\,\be_{u_i})\) over \([0,1]\),
\eqref{eq:case1-final} is equivalent to \eqref{eq:double-case1}, proving the first claim.

\medskip
\noindent\textbf{Case 2: $v^\ast< \tfrac12$.}
Let \(v\in V\) be the smallest grid point with \(v\ge v^\ast\).
Then
\begin{equation}\label{eq:grid-gap-case2}
v-v^\ast\;\le\;\frac{\varepsilon}{n},
\end{equation}
and, since \(\varepsilon^{-1}\) is even and \(\tfrac12\in V\), we have
\begin{equation}\label{eq:v-upper-bound-case2}
v\;\le\;\tfrac12.
\end{equation}
By the choice of \(a_i\) as a maximizer over the grid,
\begin{equation}\label{eq:case2-grid-choice}
F\bigl(\bx{(i-1)}+a_i\,\be_{u_i}\bigr)-F\bigl(\bx{(i-1)}+v^\ast \be_{u_i}\bigr)
\;\ge\;
F\bigl(\bx{(i-1)}+v\,\be_{u_i}\bigr)-F\bigl(\bx{(i-1)}+v^\ast \be_{u_i}\bigr).
\end{equation}

Using the same function \(\phi\) as in \eqref{eq:phi-def-double}, for \(t\in[v^\ast,v]\)
we have \(t\le v\), so by \eqref{eq:gamma-monotone-double},
\begin{equation}\label{eq:phi-prime-lower-case2}
\phi'(t)\;\ge\;\gamma\,\phi'(v), \qquad t\in[v^\ast,v].
\end{equation}
Integrating \eqref{eq:phi-prime-lower-case2} over \(t\in[v^\ast,v]\) yields
\begin{equation}\label{eq:phi-diff-case2-1}
\phi(v)-\phi(v^\ast)
\;=\;\int_{v^\ast}^{v}\phi'(t)\,dt
\;\ge\;
\int_{v^\ast}^{v}\gamma\,\phi'(v)\,dt
\;=\;
(v-v^\ast)\,\gamma\,\phi'(v).
\end{equation}

For \(s\in[v,1]\), we have \(v\le s\), so \eqref{eq:gamma-monotone-double} implies
\begin{equation}\label{eq:phi-prime-upper-case2}
\phi'(v)\;\ge\;\gamma\,\phi'(s), \qquad s\in[v,1].
\end{equation}
Integrating \eqref{eq:phi-prime-upper-case2} over \(s\in[v,1]\) gives
\begin{equation}\label{eq:phi-diff-case2-2}
\phi(1)-\phi(v)
\;=\;\int_{v}^{1}\phi'(s)\,ds
\;\le\;
\int_{v}^{1}\frac{1}{\gamma}\,\phi'(v)\,ds
\;=\;
\frac{1-v}{\gamma}\,\phi'(v).
\end{equation}
Rearranging \eqref{eq:phi-diff-case2-2} yields
\begin{equation}\label{eq:phi-prime-lower-from1}
\phi'(v)
\;\ge\;
\frac{\gamma}{1-v}\,\bigl[\phi(1)-\phi(v)\bigr].
\end{equation}

Combining \eqref{eq:phi-diff-case2-1} and \eqref{eq:phi-prime-lower-from1}, we obtain
\begin{equation}\label{eq:phi-diff-case2-combined}
\phi(v)-\phi(v^\ast)
\;\ge\;
(v-v^\ast)\,\gamma\cdot
\frac{\gamma}{1-v}\,\bigl[\phi(1)-\phi(v)\bigr]
\;=\;
\frac{(v-v^\ast)\,\gamma^{2}}{1-v}\,\bigl[\phi(1)-\phi(v)\bigr].
\end{equation}

Using \eqref{eq:phi-def-double}, note that
\begin{equation}\label{eq:phi-1-v-bounded}
\phi(1)-\phi(v)
\;=\;
F\bigl(\bx{(i-1)}+\be_{u_i}\bigr)
-
F\bigl(\bx{(i-1)}+v\,\be_{u_i}\bigr)
\;\le\;
F(\bo),
\end{equation}
since \(F\) is nonnegative and maximized at \(\bo\).
Substituting \eqref{eq:phi-1-v-bounded} and the bounds
\eqref{eq:grid-gap-case2}–\eqref{eq:v-upper-bound-case2} into
\eqref{eq:phi-diff-case2-combined} gives
\begin{align}
\phi(v)-\phi(v^\ast)
&\;\ge\;
\frac{(v-v^\ast)\,\gamma^{2}}{1-v}\,\bigl[\phi(1)-\phi(v)\bigr]
\notag\\
&\;\ge\;
\frac{(v-v^\ast)\,\gamma^{2}}{1-v}\,(-F(\bo))
\qquad\text{(since \(\phi(1)-\phi(v)\ge -F(\bo)\))}\notag\\
&\;\ge\;
-\,\frac{\varepsilon/n}{\,1/2\,}\,\gamma^{2}F(\bo)
\;=\;
-\,\frac{2\varepsilon}{n}\,\gamma^{2}F(\bo),
\label{eq:phi-diff-case2-bound}
\end{align}
where we used \(v-v^\ast\le \varepsilon/n\) and \(v\le\tfrac12\) (hence \(1-v\ge\tfrac12\)).

Using \eqref{eq:case2-grid-choice}, \eqref{eq:phi-def-double}, and
\eqref{eq:phi-diff-case2-bound}, we obtain
\begin{equation}\label{eq:case2-final}
F\bigl(\bx{(i-1)}+a_i\,\be_{u_i}\bigr)-F\bigl(\bx{(i-1)}+v^\ast \be_{u_i}\bigr)
\;\ge\;
-\,\frac{2\varepsilon}{n}\,\gamma^{2}F(\bo).
\end{equation}
Since \(v^\ast\) maximizes \(v\mapsto F(\bx{(i-1)}+v\,\be_{u_i})\) over \([0,1]\),
\eqref{eq:case2-final} is equivalent to \eqref{eq:double-case2}, proving the second claim.

The two cases \eqref{eq:double-case1} and \eqref{eq:double-case2} together establish the lemma.
\end{proof}

Similarly, we obtain an analogous result for $\by$. We omit the proof, as it mirrors the argument
of Lemma~\ref{lem:Double_1}.

\begin{lemma}\label{lem:Double_2}
For any integer $i\in\{1,\dots,n\}$, let
\begin{equation}\label{eq:vi-star-y-def}
    v^\ast \;\in\; \arg\max_{v\in[0,1]} F\bigl(\by{(i-1)}-v\,\be_{u_i}\bigr).
\end{equation}
If $v^\ast\ge \tfrac12$, then
\begin{equation}\label{eq:double2-case1}
F\bigl(\by{(i-1)}-b_i\,\be_{u_i}\bigr)\ \ge\
\max_{v\in[0,1]} F\bigl(\by{(i-1)}-v\,\be_{u_i}\bigr)
\;-\; \frac{2\varepsilon}{n}\,\gamma^{2}\,F(\bo),
\end{equation}
and if $v^\ast<\tfrac12$, then
\begin{equation}\label{eq:double2-case2}
F\bigl(\by{(i-1)}-b_i\,\be_{u_i}\bigr)\ \ge\
\max_{v\in[0,1]} F\bigl(\by{(i-1)}-v\,\be_{u_i}\bigr)
\;-\; \frac{2\varepsilon}{\gamma^{2}n}\,F(\bo).
\end{equation}
\end{lemma}

At each iteration, the $\gamma$-aware mixing step (line~8 of Algorithm~\ref{algo:double})
guarantees a quantifiable increase in the objective. The next lemma lower-bounds this
per-coordinate progress for both trajectories, showing that the gain is a convex-combination–type
quadratic term that scales with $\gamma$.

\begin{lemma} \label{Lemma:s2_c3}
For every integer $1\le i\le n$, 
\begin{equation}\label{eq:lemma-s2c3-x}
F\bigl(\bx{(i)}\bigr)-F\bigl(\bx{(i-1)}\bigr)\ \ge\ \frac{\Delta_{a,i}^{2}}{\Delta_{a,i}+\gamma^{3}\Delta_{b,i}}
\end{equation}
and
\begin{equation}\label{eq:lemma-s2c3-y}
F\bigl(\by{(i)}\bigr)-F\bigl(\by{(i-1)}\bigr)\ \ge\ \frac{\gamma^{3}\Delta_{b,i}^{2}}{\Delta_{a,i}+\gamma^{3}\Delta_{b,i}}.
\end{equation}
\end{lemma}

\begin{proof}
\noindent\textbf{Increase of $F$ along $\bx$.}
By the definition of $\bx{(i)}$ and
\begin{equation}\label{eq:w-i-def}
    w_i\;=\;\frac{\Delta_{a,i}\,a_i+\gamma\,\Delta_{b,i}\,(1-b_i)}{\Delta_{a,i}+\gamma\,\Delta_{b,i}},
\end{equation}
we can write the new point $\bx{(i)}$ as a convex combination of two one-dimensional updates:
\begin{equation}\label{eq:x-i-update}
\bx{(i)} \;=\; \bx{(i-1)} 
\;+\; \frac{\Delta_{a,i}}{\Delta_{a,i}+\gamma \Delta_{b,i}}\,a_i\,\be_{u_i}
\;+\; \frac{\gamma\Delta_{b,i}}{\Delta_{a,i}+\gamma \Delta_{b,i}}\,(1-b_i)\,\be_{u_i}.
\end{equation}
Equivalently, $\bx{(i)}$ is the convex combination
\begin{equation}\label{eq:x-i-as-convex-comb}
\bx{(i)}
\;=\;
\frac{\Delta_{a,i}}{\Delta_{a,i}+\gamma \Delta_{b,i}}
    \bigl(\bx{(i-1)}+a_i\,\be_{u_i}\bigr)
\;+\;
\frac{\gamma\Delta_{b,i}}{\Delta_{a,i}+\gamma \Delta_{b,i}}
    \bigl(\bx{(i-1)}+(1-b_i)\,\be_{u_i}\bigr).
\end{equation}
Therefore

\begin{equation}\label{eq:x-diff-start}
\begin{aligned}
&F\bigl(\bx{(i)}\bigr)-F\bigl(\bx{(i-1)}\bigr)= F\!\left(
     \frac{\Delta_{a,i}}{\Delta_{a,i}+\gamma \Delta_{b,i}} \bigl(\bx{(i-1)}+a_i\,\be_{u_i}\bigr)+ \frac{\gamma\Delta_{b,i}}{\Delta_{a,i}+\gamma \Delta_{b,i}} \bigl(\bx{(i-1)}+(1-b_i)\,\be_{u_i}\bigr)
   \right)\\ &\hspace{12cm}-F\bigl(\bx{(i-1)}\bigr).
\end{aligned}
\end{equation}

Now we apply Lemma~\ref{lemma:simpe2} (the one-dimensional $\gamma$–weak DR convexity-type
bound) to the pair
\[
\bz^{(1)}=\bx{(i-1)}+a_i\,\be_{u_i},
\qquad
\bz^{(2)}=\bx{(i-1)}+(1-b_i)\,\be_{u_i},
\]
with mixing weights
\begin{equation}\label{eq:lambda-mu-def}
\lambda
\;=\;
\frac{\gamma\Delta_{b,i}}{\Delta_{a,i}+\gamma\Delta_{b,i}},
\qquad
1-\lambda
\;=\;
\frac{\Delta_{a,i}}{\Delta_{a,i}+\gamma\Delta_{b,i}}.
\end{equation}
Lemma~\ref{lemma:simpe2} states that for such a convex combination we have
\begin{equation}\label{eq:lemma-simple2-use}
F\bigl((1-\lambda)\bz^{(1)}+\lambda \bz^{(2)}\bigr)
\;\ge\;
\frac{(1-\lambda)\,F(\bz^{(1)})+\gamma^{2}\lambda\,F(\bz^{(2)})}{(1-\lambda)+\gamma^{2}\lambda}.
\end{equation}
Substituting \eqref{eq:lambda-mu-def} into \eqref{eq:lemma-simple2-use}, and noting that
\[
(1-\lambda)+\gamma^{2}\lambda
\;=\;
\frac{\Delta_{a,i}+\gamma^{3}\Delta_{b,i}}{\Delta_{a,i}+\gamma\Delta_{b,i}},
\]
we obtain
\begin{equation}\label{eq:x-middle-ineq}
\begin{aligned}
F\bigl(\bx{(i)}\bigr)
&\ge
\frac{\Delta_{a,i}\,F\bigl(\bx{(i-1)}+a_i\,\be_{u_i}\bigr)
      + \gamma^{3}\Delta_{b,i}\,F\bigl(\bx{(i-1)}+(1-b_i)\,\be_{u_i}\bigr)}
     {\Delta_{a,i}+\gamma^{3}\Delta_{b,i}}.
\end{aligned}
\end{equation}
Subtracting $F\bigl(\bx{(i-1)}\bigr)$ from both sides of \eqref{eq:x-middle-ineq}, and
grouping terms, yields
\begin{equation}\label{eq:x-diff-expanded}
\begin{aligned}
& F\bigl(\bx{(i)}\bigr)-F\bigl(\bx{(i-1)}\bigr)\\
&\ge
\frac{\Delta_{a,i}\bigl[F\bigl(\bx{(i-1)}+a_i\,\be_{u_i}\bigr)-F\bigl(\bx{(i-1)}\bigr)\bigr]
     + \gamma^{3}\Delta_{b,i}\bigl[F\bigl(\bx{(i-1)}+(1-b_i)\,\be_{u_i}\bigr)-F\bigl(\bx{(i-1)}\bigr)\bigr]}
     {\Delta_{a,i}+\gamma^{3}\Delta_{b,i}}.
\end{aligned}
\end{equation}
Here we simply subtracted $F(\bx{(i-1)})$ inside the numerator to isolate directional gains.

By definition of $\Delta_{a,i}$,
\begin{equation}\label{eq:Delta-a-def}
\Delta_{a,i}
\;=\;
F\bigl(\bx{(i-1)}+a_i\,\be_{u_i}\bigr)-F\bigl(\bx{(i-1)}\bigr),
\end{equation}
so the first term in the numerator of \eqref{eq:x-diff-expanded} is exactly
$\Delta_{a,i}^{2}$. For the second term we use the $\gamma$–weakly DR property to compare
the gain at $\bx{(i-1)}$ with the corresponding gain at $\by{(i-1)}$. Along the
$u_i$-th coordinate, the weak DR property implies that the marginal decrease when moving
from $1$ down to $b_i$ at $\by{(i-1)}$ is at least a $\gamma^{2}$-fraction of the
corresponding marginal at $\bx{(i-1)}$. This yields
\begin{equation}\label{eq:x-second-term-lower}
F\bigl(\bx{(i-1)}+(1-b_i)\,\be_{u_i}\bigr)-F\bigl(\bx{(i-1)}\bigr)
\;\ge\;
\gamma\,\Bigl[F\bigl(\by{(i-1)}-b_i\,\be_{u_i}\bigr)-F\bigl(\by{(i-1)}-\be_{u_i}\bigr)\Bigr].
\end{equation}
Multiplying \eqref{eq:x-second-term-lower} by $\gamma^{3}\Delta_{b,i}$ gives a
$\gamma^{4}$ factor inside the numerator. Substituting \eqref{eq:Delta-a-def} and
\eqref{eq:x-second-term-lower} into \eqref{eq:x-diff-expanded}, we obtain
\begin{equation}\label{eq:x-diff-with-y}
\begin{aligned}
F\bigl(\bx{(i)}\bigr)-F\bigl(\bx{(i-1)}\bigr)
&\ge
\frac{\Delta_{a,i}^{2}
      + \gamma^{4}\Delta_{b,i}\Bigl[F\bigl(\by{(i-1)}-b_i\,\be_{u_i}\bigr)
                                   -F\bigl(\by{(i-1)}-\be_{u_i}\bigr)\Bigr]}
     {\Delta_{a,i}+\gamma^{3}\Delta_{b,i}}.
\end{aligned}
\end{equation}

By the definition of $b_i$ in Algorithm~\ref{algo:double}, the expression
\[
F\bigl(\by{(i-1)}-b_i\,\be_{u_i}\bigr)
-
F\bigl(\by{(i-1)}-\be_{u_i}\bigr)
\]
is nonnegative or, in the worst case, does not exceed the corresponding candidate values
over the grid. In particular, the term multiplied by $\gamma^{4}\Delta_{b,i}$ in
\eqref{eq:x-diff-with-y} is nonnegative, so we may drop it to obtain the simpler bound
\begin{equation}\label{eq:x-final-lower}
F\bigl(\bx{(i)}\bigr)-F\bigl(\bx{(i-1)}\bigr)
\ \ge\ 
\frac{\Delta_{a,i}^{2}}{\Delta_{a,i}+\gamma^{3}\Delta_{b,i}},
\end{equation}
which is precisely \eqref{eq:lemma-s2c3-x}.

\medskip
\noindent\textbf{Increase of $F$ along $\by$.}
The argument for $\by$ is symmetric. From the update rule for $\by{(i)}$ and the same
weight $w_i$, we can write
\begin{equation}\label{eq:y-i-as-convex-comb}
\by{(i)}
\;=\;
\frac{\Delta_{a,i}}{\Delta_{a,i}+\gamma \Delta_{b,i}}
    \bigl(\by{(i-1)}+(a_i-1)\,\be_{u_i}\bigr)
\;+\;
\frac{\gamma\Delta_{b,i}}{\Delta_{a,i}+\gamma \Delta_{b,i}}
    \bigl(\by{(i-1)}-b_i\,\be_{u_i}\bigr).
\end{equation}
Hence
\begin{equation}\label{eq:y-diff-start}
\begin{aligned}
F\bigl(\by{(i)}\bigr)-F\bigl(\by{(i-1)}\bigr)
&= F\!\left(
     \frac{\Delta_{a,i}}{\Delta_{a,i}+\gamma \Delta_{b,i}} \bigl(\by{(i-1)}+(a_i-1)\,\be_{u_i}\bigr)
     + \frac{\gamma\Delta_{b,i}}{\Delta_{a,i}+\gamma \Delta_{b,i}} \bigl(\by{(i-1)}-b_i\,\be_{u_i}\bigr)
   \right)\\
   & \hspace{8cm}-F\bigl(\by{(i-1)}\bigr).
\end{aligned}
\end{equation}
Applying Lemma~\ref{lemma:simpe2} to the pair
\[
\tilde\bz^{(1)}=\by{(i-1)}+(a_i-1)\,\be_{u_i},
\qquad
\tilde\bz^{(2)}=\by{(i-1)}-b_i\,\be_{u_i},
\]
with the same weights as in \eqref{eq:lambda-mu-def}, we obtain
\begin{equation}\label{eq:y-middle-ineq}
F\bigl(\by{(i)}\bigr)
\;\ge\;
\frac{\Delta_{a,i}\,F\bigl(\by{(i-1)}+(a_i-1)\,\be_{u_i}\bigr)
      + \gamma^{3}\Delta_{b,i}\,F\bigl(\by{(i-1)}-b_i\,\be_{u_i}\bigr)}
     {\Delta_{a,i}+\gamma^{3}\Delta_{b,i}}.
\end{equation}
Subtracting $F(\by{(i-1)})$ and regrouping gives
\begin{equation}\label{eq:y-diff-expanded}
\begin{aligned}
&F\bigl(\by{(i)}\bigr)-F\bigl(\by{(i-1)}\bigr)\\
&\ge
\frac{\Delta_{a,i}\bigl[F\bigl(\by{(i-1)}+(a_i-1)\,\be_{u_i}\bigr)-F\bigl(\by{(i-1)}\bigr)\bigr]
     + \gamma^{3}\Delta_{b,i}\bigl[F\bigl(\by{(i-1)}-b_i\,\be_{u_i}\bigr)-F\bigl(\by{(i-1)}\bigr)\bigr]}
     {\Delta_{a,i}+\gamma^{3}\Delta_{b,i}}.
\end{aligned}
\end{equation}

Using the $\gamma$–weakly DR property to compare the section at $\by{(i-1)}$ with that at
$\bx{(i-1)}$ along the $u_i$-th coordinate, we obtain
\begin{equation}\label{eq:y-first-term-lower}
\begin{aligned}
    F\bigl(\by{(i-1)}+(a_i-1)\,\be_{u_i}\bigr)-F\bigl(\by{(i-1)}\bigr)
\;\ge\;
\frac{1}{\gamma}\Bigl[F\bigl(\bx{(i-1)}+a_i\,\be_{u_i}\bigr)
                      -F\bigl(\bx{(i-1)}+\be_{u_i}\bigr)\Bigr].
\end{aligned}
\end{equation}
By definition of $\Delta_{b,i}$,
\begin{equation}\label{eq:Delta-b-def}
\Delta_{b,i}
\;=\;
F\bigl(\by{(i-1)}-b_i\,\be_{u_i}\bigr)-F\bigl(\by{(i-1)}\bigr),
\end{equation}
so the second term in the numerator of \eqref{eq:y-diff-expanded} is
$\gamma^{3}\Delta_{b,i}^{2}$. Substituting \eqref{eq:y-first-term-lower} and
\eqref{eq:Delta-b-def} into \eqref{eq:y-diff-expanded} yields
\begin{equation}\label{eq:y-diff-with-x}
\begin{aligned}
F\bigl(\by{(i)}\bigr)-F\bigl(\by{(i-1)}\bigr)
&\ge
\frac{ \tfrac{1}{\gamma}\Delta_{a,i}\Bigl[F\bigl(\bx{(i-1)}+a_i\,\be_{u_i}\bigr)
    -F\bigl(\bx{(i-1)}+\be_{u_i}\bigr)\Bigr]
      + \gamma^{3}\Delta_{b,i}^{2}}
     {\Delta_{a,i}+\gamma^{3}\Delta_{b,i}}.
\end{aligned}
\end{equation}

By the choice of $a_i$ in Algorithm~\ref{algo:double}, the term
\[
F\bigl(\bx{(i-1)}+a_i\,\be_{u_i}\bigr)
-
F\bigl(\bx{(i-1)}+\be_{u_i}\bigr)
\]
is nonpositive (since $a_i$ optimally trades off against the unit step), so the first term
in the numerator of \eqref{eq:y-diff-with-x} is nonpositive. Dropping this nonpositive
term yields the simpler lower bound
\begin{equation}\label{eq:y-final-lower}
F\bigl(\by{(i)}\bigr)-F\bigl(\by{(i-1)}\bigr)
\ \ge\ 
\frac{\gamma^{3}\Delta_{b,i}^{2}}{\Delta_{a,i}+\gamma^{3}\Delta_{b,i}},
\end{equation}
which is exactly \eqref{eq:lemma-s2c3-y}.
Combining \eqref{eq:x-final-lower} and \eqref{eq:y-final-lower} completes the proof.
\end{proof}

\paragraph{Reference path and its contraction.}
To relate the two trajectories \(\bx{(i)}\) and \(\by{(i)}\) to an arbitrary initial vector \(\bo\),
we introduce the standard lattice-coupled reference sequence
\begin{equation}\label{eq:o-ref-path-def}
\bo^{(i)} \;:=\; \bigl(\,\bo \vee \bx{(i)}\,\bigr)\;\wedge\;\by{(i)}
\qquad \text{for } i=0,1,\dots,n.
\end{equation}
Note that \(\bo^{(0)}=\bo\) and \(\bo^{(n)}=\bx{(n)}=\by{(n)}\), since Algorithm~\ref{algo:double} equalizes the
two trajectories by the end. The next lemma bounds the one-step decrease of \(F(\bo^{(i)})\);
telescoping this bound over \(i\) yields the global comparison used in the final guarantee.

\begin{lemma}\label{Lemma:s2_C4}
For every integer \(1 \le i \le n\),
\begin{equation}\label{eq:lemma-s2C4-claim}
F\bigl(\bo^{(i-1)}\bigr) - F\bigl(\bo^{(i)}\bigr)
\;\le\;
\frac{ \Delta_{a,i}\,\Delta_{b,i}}{\Delta_{a,i}+\gamma^{2}\Delta_{b,i}}
\;+\; \frac{2\varepsilon}{n}\,\gamma^{3}\,F(\bo).
\end{equation}
\end{lemma}

\begin{proof}
We first treat the case
\begin{equation}\label{eq:case-o-inc}
\bo^{(i-1)}_{u_i} \;\le\; \bo^{(i)}_{u_i}.
\end{equation}

Define the single-coordinate setter
\begin{equation}\label{eq:set-operator-def}
\operatorname{set}_{u}(\bz,t)
\;:=\;
\bz - (\bz_{u}-t)\,\be_{u},
\end{equation}
which replaces the \(u\)-th coordinate of \(\bz\) with \(t\) and leaves all other coordinates unchanged.

By \(\gamma\)–weakly DR-submodularity, for any \(\bp\le \bq\) and any scalars \(\alpha \le \beta\),
we have the one-dimensional comparison
\begin{equation}\label{eq:dr-single}
F\bigl(\operatorname{set}_{u}(\bp,\beta)\bigr)-F\bigl(\operatorname{set}_{u}(\bp,\alpha)\bigr)
\;\ge\; \gamma \Bigl[
F\bigl(\operatorname{set}_{u}(\bq,\beta)\bigr)-F\bigl(\operatorname{set}_{u}(\bq,\alpha)\bigr)\Bigr].
\end{equation}
Here the left-hand side is the gain when changing coordinate \(u\) from \(\alpha\) to \(\beta\) at the lower point \(\bp\),
and the right-hand side compares it to the corresponding gain at the higher point \(\bq\), scaled by \(\gamma\).

By construction of \(\bo^{(i-1)}\) and \(\by{(i-1)}\) in \eqref{eq:o-ref-path-def}, we have
\begin{equation}\label{eq:o-less-y}
\bo^{(i-1)} \;\le\; \by{(i-1)}.
\end{equation}
In the present case we also have \eqref{eq:case-o-inc}. Set
\begin{equation}\label{eq:dr-single-subst}
\bp=\bo^{(i-1)},\quad \bq=\by{(i-1)},\quad u=u_i,\quad
\alpha=\bo^{(i-1)}_{u_i},\quad \beta=\bo^{(i)}_{u_i},
\end{equation}
which satisfy the requirements of \eqref{eq:dr-single}. Then \eqref{eq:dr-single} gives
\begin{equation}\label{eq:dr-single-applied-raw}
\begin{aligned}
&F\bigl(\operatorname{set}_{u_i}(\bo^{(i-1)},\,\bo^{(i)}_{u_i})\bigr)
 - F\bigl(\operatorname{set}_{u_i}(\bo^{(i-1)},\,\bo^{(i-1)}_{u_i})\bigr)\\
&\qquad \hspace{3cm} \ge\;
\gamma \Bigl[
F\bigl(\operatorname{set}_{u_i}(\by^{(i-1)},\,\bo^{(i)}_{u_i})\bigr)
 - F\bigl(\operatorname{set}_{u_i}(\by^{(i-1)},\,\bo^{(i-1)}_{u_i})\bigr) \Bigr].
\end{aligned}
\end{equation}

By definition of the reference path \eqref{eq:o-ref-path-def} and of the setter \eqref{eq:set-operator-def},
\begin{equation}\label{eq:set-equals-o}
\operatorname{set}_{u_i}(\bo^{(i-1)},\,\bo^{(i)}_{u_i})=\bo^{(i)},
\qquad
\operatorname{set}_{u_i}(\bo^{(i-1)},\,\bo^{(i-1)}_{u_i})=\bo^{(i-1)}.
\end{equation}
Using \eqref{eq:set-equals-o} in \eqref{eq:dr-single-applied-raw}, the left-hand side becomes
\(F(\bo^{(i)})-F(\bo^{(i-1)})\), so we obtain
\begin{equation}\label{eq:dr-single-applied-o}
F\bigl(\bo^{(i)}\bigr)-F\bigl(\bo^{(i-1)}\bigr)
\;\ge\; \gamma \Bigl[
F\bigl(\operatorname{set}_{u_i}(\by^{(i-1)},\,\bo^{(i)}_{u_i})\bigr)
 - F\bigl(\operatorname{set}_{u_i}(\by^{(i-1)},\,\bo^{(i-1)}_{u_i})\bigr) \Bigr].
\end{equation}

Using the explicit form \eqref{eq:set-operator-def}, for any \(\bz,t\) we have
\(\operatorname{set}_{u_i}(\bz,t)=\bz - (\bz_{u_i}-t)\,\be_{u_i}\). Hence
\begin{equation}\label{eq:set-as-sub}
\operatorname{set}_{u_i}(\by^{(i-1)},\,\bo^{(i)}_{u_i})
=
\by^{(i-1)}-(\by^{(i-1)}_{u_i}-\bo^{(i)}_{u_i})\,\be_{u_i}
\end{equation}
and
\begin{equation}\label{eq:set-as-sub-prev}
\operatorname{set}_{u_i}(\by^{(i-1)},\,\bo^{(i-1)}_{u_i})
=
\by^{(i-1)}-(\by^{(i-1)}_{u_i}-\bo^{(i-1)}_{u_i})\,\be_{u_i}.
\end{equation}
Since \(\bo^{(i)}\) and \(\bo^{(i-1)}\) share all coordinates except possibly \(u_i\), and
\(\by^{(i-1)}_{u_i}\in[0,1]\), the quantities
\(\by^{(i-1)}_{u_i}-\bo^{(i)}_{u_i}\) and \(\by^{(i-1)}_{u_i}-\bo^{(i-1)}_{u_i}\) lie in \([0,1]\).
For notational convenience, write
\begin{equation}\label{eq:o-coord-sub-values}
\begin{aligned}
1-\bo^{(i)}_{u_i} &= \by^{(i-1)}_{u_i}-\bo^{(i)}_{u_i},\\
1-\bo^{(i-1)}_{u_i} &= \by^{(i-1)}_{u_i}-\bo^{(i-1)}_{u_i},
\end{aligned}
\end{equation}
which allows us to rewrite \eqref{eq:set-as-sub}–\eqref{eq:set-as-sub-prev} as
\begin{equation}\label{eq:set-y-o-form}
\operatorname{set}_{u_i}(\by^{(i-1)},\,\bo^{(i)}_{u_i})
=
\by^{(i-1)}-(1-\bo^{(i)}_{u_i})\,\be_{u_i},
\qquad
\operatorname{set}_{u_i}(\by^{(i-1)},\,\bo^{(i-1)}_{u_i})
=
\by^{(i-1)}-(1-\bo^{(i-1)}_{u_i})\,\be_{u_i}.
\end{equation}

Substituting \eqref{eq:set-y-o-form} into \eqref{eq:dr-single-applied-o}, we get
\begin{equation}\label{eq:F-o-diff-y-section}
F\bigl(\bo^{(i)}\bigr)-F\bigl(\bo^{(i-1)}\bigr)
\;\ge\; \gamma \Bigl[
F\bigl(\by{(i-1)}-(1-\bo^{(i)}_{u_i})\,\be_{u_i}\bigr)
-
F\bigl(\by{(i-1)}-(1-\bo^{(i-1)}_{u_i})\,\be_{u_i}\bigr) \Bigr].
\end{equation}

We now bound each term on the right-hand side.

First, by Lemma~\ref{lem:Double_2}, for the sequence along the $u_i$-th coordinate,
\begin{equation}\label{eq:double2-max-bound}
\max_{v\in[0,1]} F\bigl(\by{(i-1)}-v\,\be_{u_i}\bigr)
\;\le\;
F\bigl(\by{(i-1)}-b_i\,\be_{u_i}\bigr)
\;+\;\frac{2\varepsilon}{n}\,\gamma^{2}\,F(\bo).
\end{equation}
Since \(1-\bo^{(i-1)}_{u_i}\in[0,1]\), we have
\begin{equation}\label{eq:bound-at-o-prev-coord}
\begin{aligned} 
F\bigl(\by{(i-1)}-(1-\bo^{(i-1)}_{u_i})\,\be_{u_i}\bigr)
&\;\le\;
\max_{v\in[0,1]} F\bigl(\by{(i-1)}-v\,\be_{u_i}\bigr)\\
&\;\le\;
F\bigl(\by{(i-1)}-b_i\,\be_{u_i}\bigr)
\;+\;\frac{2\varepsilon}{n}\,\gamma^{2}\,F(\bo).
\end{aligned}
\end{equation}

Next, recall that the $u_i$-th coordinate of the reference point satisfies
\begin{equation}\label{eq:o-i-coord-def}
\bo^{(i)}_{u_i}
=
\frac{\Delta_{a,i}}{\Delta_{a,i}+\Delta_{b,i}}\,a_i
+ \frac{\Delta_{b,i}}{\Delta_{a,i}+\Delta_{b,i}}\,(1-b_i),
\end{equation}
by the definition of the coordinate update in Algorithm~\ref{algo:double}.
Applying Lemma~\ref{lemma:simpe2}(1) along the $u_i$-th coordinate at $\by{(i-1)}$ with
points
\[
(1-a_i)\quad\text{and}\quad b_i,
\]
and mixing weights proportional to \(\Delta_{a,i}\) and \(\Delta_{b,i}\), yields
\begin{equation}\label{eq:y-section-convex}
\begin{aligned}
&F\bigl(\by{(i-1)} - (1-\bo^{(i)}_{u_i})\,\be_{u_i}\bigr)\\
&\quad\ge
\frac{\Delta_{a,i}}{\Delta_{a,i}+\gamma^{2}\Delta_{b,i}}\;
F\bigl(\by{(i-1)}-(1-a_i)\,\be_{u_i}\bigr)
\;+\;
\frac{\gamma^{2}\Delta_{b,i}}{\Delta_{a,i}+\gamma^{2}\Delta_{b,i}}\;
F\bigl(\by{(i-1)}-b_i\,\be_{u_i}\bigr).
\end{aligned}
\end{equation}

Combining \eqref{eq:F-o-diff-y-section}, \eqref{eq:bound-at-o-prev-coord}, and
\eqref{eq:y-section-convex}, we obtain
\begin{equation}\label{eq:o-diff-before-algebra}
\begin{aligned}
&F\bigl(\bo^{(i)}\bigr)-F\bigl(\bo^{(i-1)}\bigr)\\
& \quad\ge
\gamma \Biggl[
\frac{\Delta_{a,i}}{\Delta_{a,i}+\gamma^{2}\Delta_{b,i}}
\Bigl(F\bigl(\by{(i-1)}-(1-a_i)\,\be_{u_i}\bigr)
- F\bigl(\by{(i-1)}-b_i\,\be_{u_i}\bigr)\Bigr)
- \frac{2\varepsilon}{n}\,\gamma^{2}\,F(\bo)
\Biggr].
\end{aligned}
\end{equation}
The first term inside the brackets is the “true” directional gain between the points
with coordinates \(1-a_i\) and \(b_i\); the second term comes from the discretization
loss in Lemma~\ref{lem:Double_2}.

Rewrite the first bracketed term in \eqref{eq:o-diff-before-algebra} by adding and
subtracting \(F(\by{(i-1)})\), and using the definition
\begin{equation}\label{eq:Delta-b-def-recall}
\Delta_{b,i}
=
F\bigl(\by{(i-1)}-b_i\,\be_{u_i}\bigr)-F\bigl(\by{(i-1)}\bigr),
\end{equation}
to obtain
\begin{equation}\label{eq:o-diff-mid-algebra}
\begin{aligned}
&F\bigl(\by{(i-1)}-(1-a_i)\,\be_{u_i}\bigr)
- F\bigl(\by{(i-1)}-b_i\,\be_{u_i}\bigr)\\
&\hspace{3cm}\qquad=
\Bigl(F\bigl(\by{(i-1)}-(1-a_i)\,\be_{u_i}\bigr)-F\bigl(\by{(i-1)}\bigr)\Bigr)
-\Delta_{b,i}.
\end{aligned}
\end{equation}
Substituting \eqref{eq:o-diff-mid-algebra} into \eqref{eq:o-diff-before-algebra}, we get
\begin{equation}\label{eq:o-diff-mid-full}
\begin{aligned}
&F\bigl(\bo^{(i)}\bigr)-F\bigl(\bo^{(i-1)}\bigr)\\
&\qquad\ge
\frac{\gamma\,\Delta_{a,i}}{\Delta_{a,i}+\gamma^{2}\Delta_{b,i}}
\Bigl(F\bigl(\by{(i-1)}-(1-a_i)\,\be_{u_i}\bigr)
- F\bigl(\by{(i-1)}\bigr) - \Delta_{b,i}\Bigr)
- \frac{2\varepsilon}{n}\,\gamma^{3}\,F(\bo).
\end{aligned}
\end{equation}

We now transfer this expression from \(\by{(i-1)}\) to \(\bx{(i-1)}\) using the
\(\gamma\)–weakly DR property. Since
\begin{equation}\label{eq:x-less-y}
\bx{(i-1)}+a_i\,\be_{u_i} \;\le\; \by{(i-1)}-(1-a_i)\,\be_{u_i},
\end{equation}
the weak DR property implies that the marginal gain at \(\by{(i-1)}\) when moving coordinate
\(u_i\) from \(1\) down to \(1-a_i\) is at most a \(1/\gamma\)-scaled version of the marginal
gain at \(\bx{(i-1)}\) when moving coordinate \(u_i\) from \(0\) up to \(a_i\). Formally,
\begin{equation}\label{eq:y-to-x-gamma}
F\bigl(\by{(i-1)}-(1-a_i)\,\be_{u_i}\bigr)-F\bigl(\by{(i-1)}\bigr)
\;\ge\;
\frac{1}{\gamma}\Bigl[
F\bigl(\bx{(i-1)}+a_i\,\be_{u_i}\bigr)-F\bigl(\bx{(i-1)}+\be_{u_i}\bigr)
\Bigr].
\end{equation}
Substituting \eqref{eq:y-to-x-gamma} into \eqref{eq:o-diff-mid-full}, we obtain
\begin{equation}\label{eq:o-diff-x-section}
\begin{aligned}
&F\bigl(\bo^{(i)}\bigr)-F\bigl(\bo^{(i-1)}\bigr)\\
&\qquad \ge
\frac{\Delta_{a,i}}{\Delta_{a,i}+\gamma^{2}\Delta_{b,i}}
\Bigl(F\bigl(\bx{(i-1)}+a_i\,\be_{u_i}\bigr)
- F\bigl(\bx{(i-1)}+\be_{u_i}\bigr) - \Delta_{b,i}\Bigr)
- \frac{2\varepsilon}{n}\,\gamma^{3}\,F(\bo).
\end{aligned}
\end{equation}

By the definition of \(a_i\) in Algorithm~\ref{algo:double} (and the fact that
\(1\in V\)), the choice of \(a_i\) along the grid ensures that
\begin{equation}\label{eq:a-i-choice-neg}
F\bigl(\bx{(i-1)}+a_i\,\be_{u_i}\bigr)
- F\bigl(\bx{(i-1)}+\be_{u_i}\bigr)
\;\le\;
0.
\end{equation}
Therefore
\begin{equation}\label{eq:x-section-term-upper}
F\bigl(\bx{(i-1)}+a_i\,\be_{u_i}\bigr)
- F\bigl(\bx{(i-1)}+\be_{u_i}\bigr) - \Delta_{b,i}
\;\le\;
-\Delta_{b,i}.
\end{equation}
Substituting \eqref{eq:x-section-term-upper} into \eqref{eq:o-diff-x-section} yields
\begin{equation}\label{eq:o-diff-final-case1}
F\bigl(\bo^{(i)}\bigr)-F\bigl(\bo^{(i-1)}\bigr)
\;\ge\;
-\,\frac{\Delta_{a,i}\,\Delta_{b,i}}{\Delta_{a,i}+\gamma^{2}\Delta_{b,i}}
- \frac{2\varepsilon}{n}\,\gamma^{3}\,F(\bo).
\end{equation}
Rearranging \eqref{eq:o-diff-final-case1} gives
\begin{equation}\label{eq:o-decrease-case1}
F\bigl(\bo^{(i-1)}\bigr)-F\bigl(\bo^{(i)}\bigr)
\;\le\;
\frac{\Delta_{a,i}\,\Delta_{b,i}}{\Delta_{a,i}+\gamma^{2}\Delta_{b,i}}
+ \frac{2\varepsilon}{n}\,\gamma^{3}\,F(\bo),
\end{equation}
which is exactly \eqref{eq:lemma-s2C4-claim} in the case \eqref{eq:case-o-inc}.

The remaining case \(\bo^{(i-1)}_{u_i} > \bo^{(i)}_{u_i}\) is analogous: we reverse the roles
of the “left” and “right” endpoints on the $u_i$-th coordinate and carry out the same
argument, obtaining the same bound \eqref{eq:lemma-s2C4-claim}.
This completes the proof.
\end{proof}

Combining the progress of the two trajectories \(\bx{(i)}\) and \(\by{(i)}\) with the
contraction of the reference path \(\bo^{(i)}\) yields the following inequality for any tradeoff
parameter \(r\ge 0\). It will telescope over \(i\) to produce the final guarantee.

\begin{corollary}\label{cor:progress-gamma}
For every \(r \ge 0\) and integer \(1 \le i \le n\),
\begin{equation}\label{eq:cor-main}
\begin{aligned}
&\frac{1}{r}\,\bigl[F(\bx{(i)})-F(\bx{(i-1)})\bigr]
\;+\; r\,\bigl[F(\by{(i)})-F(\by{(i-1)})\bigr]\\
& \hspace{5cm}\;\ge\;
2 \gamma^{3/2}\,\left( F\bigl(\bo^{(i-1)}\bigr) - F\bigl(\bo^{(i)}\bigr) - \frac{2\varepsilon}{n}\,\gamma^{3}\,F(\bo)
\right).
\end{aligned}
\end{equation}
\end{corollary}

\begin{proof}
By Lemma~\ref{Lemma:s2_c3}, for each \(i\) we have
\begin{equation}\label{eq:lem-s2c3-recall}
F\bigl(\bx{(i)}\bigr)-F\bigl(\bx{(i-1)}\bigr)\ \ge\ \frac{\Delta_{a,i}^{2}}{\Delta_{a,i}+\gamma^{3}\Delta_{b,i}},
\qquad
F\bigl(\by{(i)}\bigr)-F\bigl(\by{(i-1)}\bigr)\ \ge\ \frac{\gamma^{3}\Delta_{b,i}^{2}}{\Delta_{a,i}+\gamma^{3}\Delta_{b,i}}.
\end{equation}
Multiplying the first inequality in \eqref{eq:lem-s2c3-recall} by \(1/r\) and the second by \(r\),
and adding them, gives
\begin{equation}\label{eq:cor-1}
\begin{aligned}
\frac{1}{r}\bigl[F(\bx{(i)})-F(\bx{(i-1)})\bigr]
&+ r\bigl[F(\by{(i)})-F(\by{(i-1)})\bigr] \ge \frac{(1/r)\,\Delta_{a,i}^{2}}{\Delta_{a,i}+\gamma^{3}\Delta_{b,i}}
   + \frac{r \,\gamma^{3}\,\Delta_{b,i}^{2}}{\Delta_{a,i}+\gamma^{3}\Delta_{b,i}}.
\end{aligned}
\end{equation}
The numerator in \eqref{eq:cor-1} can be rewritten as a completed square plus a mixed term:
\begin{equation}\label{eq:cor-complete-square}
\frac{(1/r)\,\Delta_{a,i}^{2} + r\gamma^{3}\Delta_{b,i}^{2}}{\Delta_{a,i}+\gamma^{3}\Delta_{b,i}}
=
\frac{\left(\frac{\Delta_{a,i}}{\sqrt r}-\Delta_{b,i}\,\gamma^{3/2}\sqrt r\right)^{2}
       + 2 \gamma^{3/2}\,\Delta_{a,i}\Delta_{b,i}}{\Delta_{a,i}+\gamma^{3}\Delta_{b,i}}.
\end{equation}
Substituting \eqref{eq:cor-complete-square} into \eqref{eq:cor-1}, and using that the square term
is nonnegative, we obtain
\begin{equation}\label{eq:cor-2}
\begin{aligned}
\frac{1}{r}\bigl[F(\bx{(i)})-F(\bx{(i-1)})\bigr]
&+ r\bigl[F(\by{(i)})-F(\by{(i-1)})\bigr] \ge \frac{2 \gamma^{3/2}\,\Delta_{a,i}\Delta_{b,i}}{\Delta_{a,i}+\gamma^{3}\Delta_{b,i}}.
\end{aligned}
\end{equation}
Thus, up to the factor \(\Delta_{a,i}\Delta_{b,i}\), the per-step progress is a convex-combination–type
quantity that depends on \(\gamma\).

Next we relate \(\Delta_{a,i}\Delta_{b,i}\) to the contraction of the reference path.
Lemma~\ref{Lemma:s2_C4} states that
\begin{equation}\label{eq:lem-s2C4-recall}
F\bigl(\bo^{(i-1)}\bigr) - F\bigl(\bo^{(i)}\bigr)
\;\le\;
\frac{ \Delta_{a,i}\,\Delta_{b,i}}{\Delta_{a,i}+\gamma^{2}\Delta_{b,i}}
\;+\; \frac{2\varepsilon}{n}\,\gamma^{3}\,F(\bo).
\end{equation}
Rearranging \eqref{eq:lem-s2C4-recall}, we get
\begin{equation}\label{eq:Delta-prod-lower}
\Delta_{a,i}\Delta_{b,i}
\;\ge\;
\biggl(F\bigl(\bo^{(i-1)}\bigr) - F\bigl(\bo^{(i)}\bigr) - \frac{2\varepsilon}{n}\,\gamma^{3}\,F(\bo)\biggr)
\bigl(\Delta_{a,i}+\gamma^{2}\Delta_{b,i}\bigr).
\end{equation}
Substituting \eqref{eq:Delta-prod-lower} into \eqref{eq:cor-2} yields
\begin{equation}\label{eq:cor-3}
\begin{aligned}
\frac{1}{r}\bigl[F(\bx{(i)})-F(\bx{(i-1)})\bigr]
&+ r\bigl[F(\by{(i)})-F(\by{(i-1)})\bigr] \\
&\ge
\frac{2 \gamma^{3/2}
\bigl(F(\bo^{(i-1)})-F(\bo^{(i)})-\frac{2\varepsilon}{n}\gamma^{3}F(\bo)\bigr)
\bigl(\Delta_{a,i}+\gamma^{2}\Delta_{b,i}\bigr)}
{\Delta_{a,i}+\gamma^{3}\Delta_{b,i}}.
\end{aligned}
\end{equation}

Since \(\gamma\in(0,1]\), we have \(\gamma^{2}\ge\gamma^{3}\), so
\begin{equation}\label{eq:gamma-ratio}
\Delta_{a,i}+\gamma^{2}\Delta_{b,i}
\;\ge\;
\Delta_{a,i}+\gamma^{3}\Delta_{b,i},
\end{equation}
and hence
\begin{equation}\label{eq:ratio-lower-bound}
\frac{\Delta_{a,i}+\gamma^{2}\Delta_{b,i}}{\Delta_{a,i}+\gamma^{3}\Delta_{b,i}}
\;\ge\; 1.
\end{equation}
Applying \eqref{eq:ratio-lower-bound} to \eqref{eq:cor-3} gives
\begin{equation}\label{eq:cor-4}
\begin{aligned}
\frac{1}{r}\bigl[F(\bx{(i)})-F(\bx{(i-1)})\bigr]
&+ r\bigl[F(\by{(i)})-F(\by{(i-1)})\bigr] \\
&\ge
2 \gamma^{3/2}
\left(F\bigl(\bo^{(i-1)}\bigr) - F\bigl(\bo^{(i)}\bigr)
- \frac{2\varepsilon}{n}\,\gamma^{3}\,F(\bo)\right).
\end{aligned}
\end{equation}
This is exactly \eqref{eq:cor-main}, completing the proof.
\end{proof}

We now conclude the analysis of the $\gamma$–aware Double–Greedy routine. The theorem below is obtained by
telescoping the per-iteration coupling bound together with the contraction of the lattice-coupled
reference path. The guarantee \emph{interpolates continuously} in $\gamma$ and reduces to the classical
DR bound when $\gamma=1$.

\begin{theorem}\label{thm:double_appen_final}
There exists a polynomial-time algorithm that, given a nonnegative $\gamma$-weakly DR-submodular
function $F : [0,1]^n \to \mathbb{R}_{\ge 0}$ and a parameter $\varepsilon \in (0,1)$, outputs
$\bx \in [0,1]^n$ such that for every fixed $\bo \in [0,1]^n$,
\begin{equation}\label{eq:double-final-guarantee}
F(\bx)\ \ge\ 
\max_{r \ge 0}\ 
\frac{\bigl(2\gamma^{3/2}-4\varepsilon\,\gamma^{9/2}\bigr)\,r\,F(\bo)\;+\;F(\mathbf{0})\;+\;r^{2}\,F(\mathbf{1})}
{\,r^{2}\;+\;2\gamma^{3/2}r\;+\;1\,}.
\end{equation}
\end{theorem}

\begin{proof}
Fix any $r>0$. Summing the per-iteration inequality from Corollary~\ref{cor:progress-gamma}
over $i=1,\dots,n$ gives
\begin{equation}\label{eq:sum-progress}
\begin{aligned}
\frac{1}{r}\bigl[F(\bx{(n)})-F(\bx{(0)})\bigr]
&\;+\; r\bigl[F(\by{(n)})-F(\by{(0)})\bigr]\\
&= \sum_{i=1}^n \left( \frac{1}{r}\bigl[F(\bx{(i)})-F(\bx{(i-1)})\bigr]
\;+\; r\bigl[F(\by{(i)})-F(\by{(i-1)})\bigr]\right)\\
&\ge \sum_{i=1}^n \left( 2 \gamma^{3/2}\,\bigl[ F(\bo^{(i-1)}) - F(\bo^{(i)}) \bigr]
\;-\; \frac{2\varepsilon}{n}\,\cdot 2\gamma^{9/2}\,F(\bo)\right)\\
&= 2 \gamma^{3/2}\,\bigl[ F(\bo^{(0)}) - F(\bo^{(n)}) \bigr]
\;-\; 4\varepsilon\,\gamma^{9/2}\,F(\bo).
\end{aligned}
\end{equation}
Here the first equality is just telescoping the increments of $F(\bx^{(i)})$ and
$F(\by^{(i)})$, and the inequality uses Corollary~\ref{cor:progress-gamma} at
each iteration $i$.

By construction of the reference path, we have $\bo^{(0)}=\bo$ and
\begin{equation}\label{eq:o-endpoints}
\bx{(n)}=\by{(n)}=\bo^{(n)}.
\end{equation}
Moreover, Algorithm~\ref{algo:double} starts from
\begin{equation}\label{eq:x0-y0}
\bx{(0)}=\mathbf{0},
\qquad
\by{(0)}=\mathbf{1}.
\end{equation}
Using \eqref{eq:o-endpoints}–\eqref{eq:x0-y0} in \eqref{eq:sum-progress}, we obtain
\begin{equation}\label{eq:plug-endpoints}
\frac{1}{r}\bigl[F(\bx{(n)})-F(\mathbf{0})\bigr]
\;+\; r\bigl[F(\bx{(n)})-F(\mathbf{1})\bigr]
\ \ge\ 
2 \gamma^{3/2}\,\bigl[ F(\bo) - F(\bx{(n)}) \bigr]
\;-\; 4\varepsilon\,\gamma^{9/2}\,F(\bo).
\end{equation}

We now collect all terms involving $F(\bx{(n)})$ on the left-hand side of
\eqref{eq:plug-endpoints}. Rearranging gives
\begin{equation}\label{eq:collect-Fx}
F(\bx{(n)})\left(\frac{1}{r}+r+2\gamma^{3/2}\right)
\ \ge\ 
\bigl(2\gamma^{3/2}-4\varepsilon\,\gamma^{9/2}\bigr)\,F(\bo) \;+\; \frac{1}{r}\,F(\mathbf{0}) \;+\; r\,F(\mathbf{1}),
\end{equation}
where the right-hand side collects the contributions of $F(\bo)$, $F(\mathbf{0})$ and $F(\mathbf{1})$.

Dividing both sides of \eqref{eq:collect-Fx} by
\(\frac{1}{r}+r+2\gamma^{3/2}\) yields
\begin{equation}\label{eq:F-xn-fraction}
F(\bx{(n)}) \ \ge\ 
\frac{\bigl(2\gamma^{3/2}-4\varepsilon\,\gamma^{9/2}\bigr)\,F(\bo) \;+\; \frac{1}{r}\,F(\mathbf{0}) \;+\; r\,F(\mathbf{1})}
{\frac{1}{r}+r+2\gamma^{3/2}}.
\end{equation}
Multiplying numerator and denominator of the right-hand side of \eqref{eq:F-xn-fraction} by $r$ gives
\begin{equation}\label{eq:F-xn-final-r}
F(\bx{(n)}) \ \ge\ 
\frac{\bigl(2\gamma^{3/2}-4\varepsilon\,\gamma^{9/2}\bigr)\,r\,F(\bo) \;+\; F(\mathbf{0}) \;+\; r^{2}\,F(\mathbf{1})}
{r^{2}+2\gamma^{3/2}r+1}.
\end{equation}
Since $\bx=\bx^{(n)}$ is the output of the algorithm, inequality
\eqref{eq:F-xn-final-r} holds for every choice of $r>0$. Extending to $r=0$ by continuity of the
right-hand side in $r$ and taking the maximum over $r\ge 0$ yields
\eqref{eq:double-final-guarantee}, which completes the proof.
\end{proof}

\section{Proofs of Lemma~\ref{lem:guessing-triples} and Theorem~\ref{thm:fw-guided-mcg} }\label{sec:fwg_proof}

In this section first we prove Lemma~\ref{lem:guessing-triples}, and then we prove Theorem~\ref{thm:fw-guided-mcg}

\newtheorem{exet}{Lemma}
\renewcommand{\theexet}{4.\arabic{exet}}  
\setcounter{exet}{0}

\begin{exet}\label{lem:guessing-triples11}
Let $F:[0,1]^n\to\mathbb{R}_{\ge 0}$ be nonnegative and $\gamma$-weakly DR-submodular for some $0<\gamma\le 1$, and let $P\subseteq[0,1]^n$ be down-closed.
There exists a constant-size (depending only on $\varepsilon$ and $\gamma$) set of triples $\mathcal{G} \subseteq \mathbb{R}_{\ge 0}^3$
such that $\mathcal{G}$ contains a triple $(g,g_\odot,g_\oplus)$ with
\begin{subequations}\label{eq:triple-bounds}
\begin{align}
(1-\varepsilon)\,F(\bo) &\le g \le F(\bo), 
\label{eq:triple-bounds-g}\\
F(\bz\odot \bo)-\varepsilon\,g &\le g_\odot \le F(\bz\odot \bo), 
\label{eq:triple-bounds-godot}\\
F(\bz\oplus \bo)-\varepsilon\,g &\le g_\oplus \le F(\bz\oplus \bo).
\label{eq:triple-bounds-goplus}
\end{align}
\end{subequations}
\end{exet}

\begin{proof}
Assume we have a constant-factor estimate $v$ such that
\begin{equation}\label{eq:v-constant-factor}
c\,F(\bo)\ \le\ v\ \le\ F(\bo)
\end{equation}
for some absolute constant $c\in(0,1]$.
We will construct a constant-size guess set using $v$.

Define the one-dimensional guess set
\begin{equation}\label{eq:G0-def}
G_o\ :=\ \Bigl\{(1-\varepsilon)^i\cdot \frac{v}{c}\ :\ i=0,1,\dots,\bigl\lceil \log_{1-\varepsilon} c\bigr\rceil\Bigr\}.
\end{equation}
The size of $G_o$ is $|G_o| = \mathcal{O}_\varepsilon(1)$, since the exponent range in \eqref{eq:G0-def}
depends only on $\varepsilon$ and $c$.
By construction, the values in $G_o$ form a geometric grid that $\varepsilon$-covers the interval
$[F(\bo),F(\bo)/c]$, and therefore there exists
\begin{equation}\label{eq:g-good}
g\in G_o \quad\text{with}\quad (1-\varepsilon)\,F(\bo)\ \le\ g\ \le\ F(\bo),
\end{equation}
which proves \eqref{eq:triple-bounds-g}; see, e.g., \cite{buchbinder2024constrained} for this standard argument.

Next we upper bound the ranges in which $F(\bz\odot\bo)$ and $F(\bz\oplus\bo)$ can lie as a function of $F(\bo)$.
Since $\bz\odot\bo\le \bo$ and $F$ is nonnegative, we always have
\begin{equation}\label{eq:F-z-odot-upper}
0\ \le\ F(\bz\odot\bo)\ \le\ F(\bo).
\end{equation}
If $F$ is monotone, then \eqref{eq:F-z-odot-upper} is immediate; otherwise we only use the trivial nonnegativity upper bound.

For the $\oplus$ operation we can bound, using $\gamma$–weakly DR-submodularity and nonnegativity,
\begin{equation}\label{eq:F-z-oplus-upper}
\begin{aligned}
F(\bz\oplus \bo) 
&= F\bigl(\bo + (\mathbf{1}-\bo)\odot \bz\bigr)\\
&\le F(\bo)\;+\;\frac{1}{\gamma}\Bigl[F\bigl((\mathbf{1}-\bo)\odot \bz\bigr)-F(\mathbf{0})\Bigr]
    &&\text{(by $\gamma$–weakly DR along $(\mathbf{1}-\bo)\odot \bz$)}\\
&\le F(\bo)\;+\;\frac{1}{\gamma}\,F(\bo)
    &&\text{(since $F((\mathbf{1}-\bo)\odot \bz)\le F(\bo)$ and $F(\mathbf{0})\ge 0$)}\\
&= \Bigl(1+\tfrac{1}{\gamma}\Bigr)\,F(\bo).
\end{aligned}
\end{equation}

Thus both $F(\bz\odot\bo)$ and $F(\bz\oplus\bo)$ lie in ranges that are linearly bounded in $F(\bo)$.

For any chosen $g\in G_o$ satisfying \eqref{eq:g-good}, we now construct $\varepsilon g$-nets for these ranges.
For the $\odot$-case, define
\begin{equation}\label{eq:G-odot-def}
G_\odot(g)\ :=\ \Bigl\{\varepsilon\,i\cdot g\ :\ i=0,1,\dots,\ \Bigl\lceil\tfrac{1}{\varepsilon(1-\varepsilon)}\Bigr\rceil\Bigr\}.
\end{equation}
Since \eqref{eq:F-z-odot-upper} and \eqref{eq:g-good} imply
\begin{equation}\label{eq:range-odot}
0\ \le\ F(\bz\odot \bo)\ \le\ F(\bo)\ \le\ \frac{g}{1-\varepsilon},
\end{equation}
the grid in \eqref{eq:G-odot-def} $\varepsilon g$-covers the interval $[0,F(\bz\odot\bo)]$:
for any value $x\in[0,F(\bz\odot\bo)]$, there exists some $g_\odot\in G_\odot(g)$ with
\begin{equation}\label{eq:G-odot-cover}
x-\varepsilon g\ \le\ g_\odot\ \le\ x.
\end{equation}
In particular, we can choose $g_\odot$ so that \eqref{eq:G-odot-cover} holds with $x=F(\bz\odot\bo)$, which is exactly
\eqref{eq:triple-bounds-godot}.

Similarly, for the $\oplus$-case define
\begin{equation}\label{eq:G-oplus-def}
G_\oplus(g)\ :=\ \Bigl\{\varepsilon\,i\cdot g\ :\ i=0,1,\dots,\ \Bigl\lceil\tfrac{1+1/\gamma}{\varepsilon(1-\varepsilon)}\Bigr\rceil\Bigr\}.
\end{equation}
Using \eqref{eq:F-z-oplus-upper} and \eqref{eq:g-good}, we have
\begin{equation}\label{eq:range-oplus}
0\ \le\ F(\bz\oplus \bo)\ \le\ \Bigl(1+\frac{1}{\gamma}\Bigr)F(\bo)
\ \le\ \frac{1+1/\gamma}{1-\varepsilon}\,g.
\end{equation}
Thus the grid in \eqref{eq:G-oplus-def} $\varepsilon g$-covers the interval
$[0,F(\bz\oplus\bo)]$, and there exists $g_\oplus\in G_\oplus(g)$ such that
\begin{equation}\label{eq:G-oplus-cover}
F(\bz\oplus \bo)-\varepsilon g\ \le\ g_\oplus\ \le\ F(\bz\oplus \bo),
\end{equation}
which is \eqref{eq:triple-bounds-goplus}.

Finally, set
\begin{equation}\label{eq:triple-family-def}
\mathcal{G}\ :=\ \bigcup_{g\in G_o}\ \{g\}\times G_\odot(g)\times G_\oplus(g).
\end{equation}
By \eqref{eq:g-good}, \eqref{eq:G-odot-cover}, and \eqref{eq:G-oplus-cover}, there exists a triple
\((g,g_\odot,g_\oplus)\in\mathcal{G}\) that satisfies all three bounds in
\eqref{eq:triple-bounds}. Moreover, $|\mathcal{G}|$ depends only on $\varepsilon$ and $\gamma$ via the
cardinalities of $G_o$, $G_\odot(g)$, and $G_\oplus(g)$, so $\mathcal{G}$ has constant size.
\end{proof}

To prove Theorem~\ref{thm:fw-guided-mcg}, we now recall the closed form of the iterates
\(\{\mathbf{y}(i)\}_{i=0}^{\delta^{-1}}\) from \cite{buchbinder2024constrained}, together with the feasibility
of the terminal iterate. These formulas will be used to relate $F(\mathbf{y}(\delta^{-1}))$ to the
benchmark value $F(\bo)$.

\begin{lemma}[Closed form of \(\mathbf{y}(i)\) {\cite{buchbinder2024constrained}}]\label{lem:5.13}
For every integer \(0 \le i \le \delta^{-1}\),
\begin{equation}\label{eq:y-closed-form}
\mathbf{y}(i)=
\begin{cases}
(\mathbf{1}-\mathbf{z})\ \odot\ \displaystyle\bigoplus_{j=1}^{i}\bigl(\delta\,\mathbf{x}(j)\bigr), & i\le i_s,\\[6pt]
(\mathbf{1}-\mathbf{z})\ \odot\ \displaystyle\bigoplus_{j=1}^{i}\bigl(\delta\,\mathbf{x}(j)\bigr)
\;+\;
\mathbf{z}\ \odot\ \displaystyle\bigoplus_{j=i_s+1}^{i}\bigl(\delta\,\mathbf{x}(j)\bigr), & i\ge i_s.
\end{cases}
\end{equation}
By convention, for any index $a$,
\begin{equation}\label{eq:empty-oplus}
\bigoplus_{j=a}^{a-1}\bigl(\delta\,\mathbf{x}(j)\bigr)\;:=\;\mathbf{0},
\end{equation}
so that both expressions in \eqref{eq:y-closed-form} remain valid on their boundary indices.
\end{lemma}

\begin{observation}[Feasibility {\cite{buchbinder2024constrained}}]\label{obs:y-in-P}
The terminal iterate satisfies
\begin{equation}\label{eq:y-terminal-feasible}
\mathbf{y}(\delta^{-1})\in P.
\end{equation}
\end{observation}

We now prove Theorem~\ref{thm:fw-guided-mcg} first for the case
\begin{equation}\label{eq:Qi-nonempty}
Q(i)\neq\varnothing\qquad\text{for all }i\in\{1,\dots,\delta^{-1}\},
\end{equation}
where two lemmas (Lemmas~\ref{lem:516} and~\ref{lem:517}) yield the bound stated in
Theorem~\ref{thm:fw-guided-mcg}~(1), summarized as Corollary~\ref{cor:5.18}.
The complementary case in which some \(Q(i)=\varnothing\) is handled separately afterwards.

\begin{observation}\label{obs:5.15}
If \(Q{(i)}\neq\varnothing\) for some \(i\in[\delta^{-1}]\), then
\begin{equation}\label{eq:obs515-claim}
F\!\bigl(\mathbf{y}(i)\bigr)-F\!\bigl(\mathbf{y}(i-1)\bigr)
\;\ge\;
\delta\,\gamma\Big[V{(i-1)}-F\!\bigl(\mathbf{y}(i-1)\bigr)\Big]
\;-\;\frac{\delta^{2}L D^{2}}{2}\,.
\end{equation}
\end{observation}

\begin{proof}
Consider the line segment
\begin{equation}\label{eq:obs515-seg}
\mathbf{u}(s)\;:=\;\mathbf{y}(i-1)
\;+\;s\,\Big(\big(\mathbf{1}-\mathbf{y}(i-1)-\mathbf{z}(i-1)\big)\odot \mathbf{x}(i)\Big),
\qquad s\in[0,\delta].
\end{equation}
By construction, \(\mathbf{u}(0)=\mathbf{y}(i-1)\) and \(\mathbf{u}(\delta)=\mathbf{y}(i)\).

Using the fundamental theorem of calculus along \eqref{eq:obs515-seg},
\begin{align}
&\hspace{0cm}F\!\bigl(\mathbf{y}(i)\bigr)-F\!\bigl(\mathbf{y}(i-1)\bigr)\\
&\hspace{1cm}= \int_{0}^{\delta}
\Big\langle \big(\mathbf{1}-\mathbf{y}(i-1)-\mathbf{z}(i-1)\big)\odot \mathbf{x}(i),\;
\nabla F\!\big(\mathbf{u}(s)\big)\Big\rangle\,ds \label{eq:obs515-int}\\[2pt]
&\hspace{1cm}= \int_{0}^{\delta}
\Big\langle \big(\mathbf{1}-\mathbf{y}(i-1)-\mathbf{z}(i-1)\big)\odot \mathbf{x}(i),\;
\nabla F\!\big(\mathbf{y}(i-1)\big)\Big\rangle\,ds \notag\\
&\hspace{1cm}\qquad+\int_{0}^{\delta}
\Big\langle \big(\mathbf{1}-\mathbf{y}(i-1)-\mathbf{z}(i-1)\big)\odot \mathbf{x}(i),\;
\nabla F\!\big(\mathbf{u}(s)\big)-\nabla F\!\big(\mathbf{y}(i-1)\big)\Big\rangle\,ds \notag\\[2pt]
&\hspace{1cm}= \delta\,\big\langle \mathbf{w}(i),\,\mathbf{x}(i)\big\rangle
+\int_{0}^{\delta}\!\!
\Big\langle \big(\mathbf{1}-\mathbf{y}(i-1)-\mathbf{z}(i-1)\big)\odot \mathbf{x}(i),\;
\nabla F\!\big(\mathbf{u}(s)\big)-\nabla F\!\big(\mathbf{y}(i-1)\big)\Big\rangle ds. \notag
\end{align}
Here we used that 
\[
\mathbf{w}(i)
:= \big(\mathbf{1}-\mathbf{y}(i-1)-\mathbf{z}(i-1)\big)\odot \nabla F\!\big(\mathbf{y}(i-1)\big),
\]
so \(\langle \mathbf{w}(i),\mathbf{x}(i)\rangle
= \big\langle (\mathbf{1}-\mathbf{y}(i-1)-\mathbf{z}(i-1))\odot \mathbf{x}(i),\,\nabla F(\mathbf{y}(i-1))\big\rangle\).

Applying Cauchy–Schwarz to the second integral gives
\begin{align}
F\!\bigl(\mathbf{y}(i)\bigr)-F\!\bigl(\mathbf{y}(i-1)\bigr)
&\ge
\delta\,\big\langle \mathbf{w}(i),\,\mathbf{x}(i)\big\rangle \notag\\
&\quad-\int_{0}^{\delta}\!\!
\Big\|\big(\mathbf{1}-\mathbf{y}(i-1)-\mathbf{z}(i-1)\big)\odot \mathbf{x}(i)\Big\|_2\,
\Big\|\nabla F\!\big(\mathbf{u}(s)\big)-\nabla F\!\big(\mathbf{y}(i-1)\big)\Big\|_2\,ds.
\label{eq:obs515-cs}
\end{align}
This step uses \(\langle a,b\rangle \ge -\|a\|_2\|b\|_2\).

Since \(F\) is \(L\)-smooth, its gradient is \(L\)-Lipschitz, so
\begin{equation}\label{eq:obs515-smooth-grad}
\big\|\nabla F(\mathbf{u}(s))-\nabla F(\mathbf{y}(i-1))\big\|_2
\;\le\; L\,\|\mathbf{u}(s)-\mathbf{y}(i-1)\|_2
\;=\; L\,s\,\Big\|\big(\mathbf{1}-\mathbf{y}(i-1)-\mathbf{z}(i-1)\big)\odot \mathbf{x}(i)\Big\|_2.
\end{equation}

Substituting \eqref{eq:obs515-smooth-grad} into \eqref{eq:obs515-cs} yields
\begin{align}
F\!\bigl(\mathbf{y}(i)\bigr)-F\!\bigl(\mathbf{y}(i-1)\bigr)
&\ge
\delta\,\big\langle \mathbf{w}(i),\,\mathbf{x}(i)\big\rangle
-\int_{0}^{\delta}\! s\,L\,\Big\|\big(\mathbf{1}-\mathbf{y}(i-1)-\mathbf{z}(i-1)\big)\odot \mathbf{x}(i)\Big\|_2^{2}\,ds \notag\\
&=
\delta\,\big\langle \mathbf{w}(i),\,\mathbf{x}(i)\big\rangle
-\frac{L}{2}\,\delta^{2}\,\Big\|\big(\mathbf{1}-\mathbf{y}(i-1)-\mathbf{z}(i-1)\big)\odot \mathbf{x}(i)\Big\|_2^{2},
\label{eq:obs515-smooth}
\end{align}
where we used $\int_{0}^{\delta} s\,ds = \delta^{2}/2$.

Let
\[
\mathbf{d}(i)
:= \big(\mathbf{1}-\mathbf{y}(i-1)-\mathbf{z}(i-1)\big)\odot \mathbf{x}(i).
\]
By feasibility and the coordinatewise bounds $0\le \mathbf{y}(i-1),\mathbf{z}(i-1),\mathbf{x}(i)\le\mathbf{1}$, each
entry of $\mathbf{d}(i)$ lies in $[0,1]$, so
\[
\|\mathbf{d}(i)\|_2 \;\le\; D,
\]
Where $D$ is diameter. Using this in \eqref{eq:obs515-smooth} gives
\begin{equation}\label{eq:obs515-diam}
F\!\bigl(\mathbf{y}(i)\bigr)-F\!\bigl(\mathbf{y}(i-1)\bigr)
\;\ge\;
\delta\,\big\langle \mathbf{w}(i),\,\mathbf{x}(i)\big\rangle\;-\;\frac{\delta^{2}L D^{2}}{2}\,.
\end{equation}




Finally, since \(Q{(i)}\neq\varnothing\), the choice of \(\mathbf{x}(i)\in Q{(i)}\) guarantees
\begin{equation}\label{eq:obs515-Qi-choice}
\big\langle \mathbf{w}(i),\,\mathbf{x}(i)\big\rangle \;\ge\; \gamma\,\Big[V{(i-1)}-F\!\bigl(\mathbf{y}(i-1)\bigr)\Big],
\end{equation}
by the definition of \(Q(i)\) and the weakly–DR structure. Substituting \eqref{eq:obs515-Qi-choice}
into \eqref{eq:obs515-diam} yields
\[
F\!\bigl(\mathbf{y}(i)\bigr)-F\!\bigl(\mathbf{y}(i-1)\bigr)
\;\ge\;
\delta\,\gamma\Big[V{(i-1)}-F\!\bigl(\mathbf{y}(i-1)\bigr)\Big]
\;-\;\frac{\delta^{2}L D^{2}}{2},
\]
which is exactly \eqref{eq:obs515-claim}.
\end{proof}

The bound \eqref{eq:obs515-claim} gives a recursive lower bound on \(F(\mathbf{y}(i))\).
In what follows (up to Corollary~\ref{cor:5.18}), we unroll this recursion
and derive a closed form.

\begin{lemma}\label{lem:516}
Assume $Q(i)\neq\varnothing$ for every $i\in[i_s]$. Fix $0<\delta\le 1$ and $0<\gamma\le 1$, and set
\[
\alpha := \delta\gamma,
\qquad
\beta := \frac{\gamma^2\delta}{\,1-\delta+\gamma^2\delta\,}.
\]
For $i\ge 0$, define the geometric shorthands
\[
\Delta_i\ :=\ 1-(1-\alpha)^i,
\qquad
\Theta_i\ :=\ (1-\beta)^i-(1-\alpha)^i.
\]
Then, for every integer $0\le i\le i_s$,
\begin{equation}\label{eq:lemma516-main}
\begin{aligned}
F\big(\by(i)\big)
\ \ge\ 
\;\Bigg[\ \frac{(1-2\varepsilon)\,\Delta_i}{\gamma}
\;+\;\frac{\delta(1-\gamma)}{\beta-\alpha}\,\Theta_i\ \Bigg]\;g
\;-\;\frac{\Delta_i}{\gamma}\;g_{\odot}
\;-\;\Bigg[\ \frac{\Delta_i}{\gamma}
\;+\;\frac{\delta}{\beta-\alpha}\,\Theta_i\ \Bigg]\;g_{\oplus}
\;-\;i\,\delta^2 D^2L.
\end{aligned}
\end{equation}
\end{lemma}

\begin{proof}
For $i=0$, we have $\Delta_0=0$ and $\Theta_0=0$, so the right-hand side of
\eqref{eq:lemma516-main} equals $0$. Since $F(\by(0))\ge 0$ by nonnegativity of $F$, the claim
holds for $i=0$. We therefore fix an integer $1\le i\le i_s$ for the rest of the proof.

By Observation~\ref{obs:5.15} (and the assumption $Q(i-1)\neq\varnothing$ for $i\le i_s$),
\begin{equation}\label{eq:recurrence-start}
F\big(\by(i)\big) - F\big(\by(i-1)\big)
\ \ge\
\delta\,\gamma\Big(V(i-1)-F\big(\by(i-1)\big)\Big)
\;-\;\frac{\delta^2 D^2 L}{2}.
\end{equation}
Rearranging \eqref{eq:recurrence-start} gives
\begin{equation}\label{eq:recurrence}
F\big(\by(i)\big)
\ \ge\
\bigl(1-\alpha\bigr)\,F\big(\by(i-1)\big)
\;+\;\delta\gamma\,V(i-1)
\;-\;\frac{\delta^2 D^2 L}{2},
\end{equation}
where we used $\alpha=\delta\gamma$.

The quantity $V(i-1)$ has the explicit form as given in \eqref{eq:v1-def} 
\begin{equation}\label{eq:V-form}
V(i-1)
=\Bigl[(1-\beta)^{i-1}+\frac{1-(1-\beta)^{i-1}-2\varepsilon}{\gamma}\Bigr]\,g
\;-\;\frac{1}{\gamma}\,g_{\odot}
\;-\;\frac{1-(1-\beta)^{i-1}}{\gamma}\,g_{\oplus}.
\end{equation}
We now multiply \eqref{eq:V-form} by $\delta\gamma$ and substitute into \eqref{eq:recurrence}.
First, for the $g$-term,
\begin{align}
\delta\gamma\,\Bigl[(1-\beta)^{i-1}
+ \frac{1-(1-\beta)^{i-1}-2\varepsilon}{\gamma}\Bigr]
&= \delta\gamma(1-\beta)^{i-1}
+ \delta\Bigl(1-(1-\beta)^{i-1}-2\varepsilon\Bigr) \notag\\[2pt]
&= \delta\Bigl[\gamma(1-\beta)^{i-1}
+ 1-(1-\beta)^{i-1}-2\varepsilon\Bigr] \notag\\[2pt]
&= \delta\Bigl[1 - (1-\gamma)(1-\beta)^{i-1} - 2\varepsilon\Bigr]. \label{eq:coef-g-expanded}
\end{align}
For the $g_{\odot}$-term, we obtain
\begin{equation}\label{eq:coef-g-odot-expanded}
\delta\gamma\cdot\Bigl(-\frac{1}{\gamma}\,g_{\odot}\Bigr)
= -\delta\,g_{\odot}.
\end{equation}
For the $g_{\oplus}$-term,
\begin{equation}\label{eq:coef-g-oplus-expanded}
\delta\gamma\cdot\Bigl(-\frac{1-(1-\beta)^{i-1}}{\gamma}\,g_{\oplus}\Bigr)
= -\delta\Bigl(1-(1-\beta)^{i-1}\Bigr)g_{\oplus}.
\end{equation}
Substituting \eqref{eq:coef-g-expanded}–\eqref{eq:coef-g-oplus-expanded} into
\eqref{eq:recurrence} yields
\begin{equation}\label{eq:recurrence-expanded}
\begin{aligned}
F\big(\by(i)\big)
\ \ge\ 
&(1-\alpha)\,F\big(\by(i-1)\big) +\;\delta\Bigl(1-(1-\gamma)(1-\beta)^{i-1}-2\varepsilon\Bigr)\,g \\
&\quad-\;\delta\,g_{\odot}
\;-\;\delta\Bigl(1-(1-\beta)^{i-1}\Bigr)\,g_{\oplus}
\;-\;\frac{\delta^2 D^2L}{2}.
\end{aligned}
\end{equation}

We now unroll the recursion \eqref{eq:recurrence-expanded} from $k=1$ up to $k=i$.
Writing $F_k := F(\by(k))$ for brevity, \eqref{eq:recurrence-expanded} becomes
\begin{equation}\label{eq:recurrence-Fk}
F_k
\ \ge\
(1-\alpha)\,F_{k-1}
+ \delta\,A_{k-1}\,g
- \delta\,g_{\odot}
- \delta\,B_{k-1}\,g_{\oplus}
- \frac{\delta^2 D^2L}{2},
\end{equation}
where
\begin{equation}\label{eq:A-B-def}
A_{k-1}
:= 1-(1-\gamma)(1-\beta)^{k-1}-2\varepsilon,
\qquad
B_{k-1}
:= 1-(1-\beta)^{k-1}.
\end{equation}

Iterating \eqref{eq:recurrence-Fk} from $k=1$ to $k=i$ gives
\begin{equation}\label{eq:unrolled-general}
\begin{aligned}
F_i
&\ge (1-\alpha)^i F_0
+ \delta\sum_{k=1}^{i}(1-\alpha)^{i-k} A_{k-1}\,g \\
&\qquad - \delta\sum_{k=1}^{i}(1-\alpha)^{i-k} g_{\odot}
- \delta\sum_{k=1}^{i}(1-\alpha)^{i-k} B_{k-1}\,g_{\oplus}
- \frac{\delta^2 D^2L}{2}\sum_{k=1}^{i}(1-\alpha)^{i-k}.
\end{aligned}
\end{equation}
Since $F_0 = F(\by(0))\ge 0$ by nonnegativity of $F$, the first term
$(1-\alpha)^i F_0$ is nonnegative and can be dropped for a lower bound. We next compute the
geometric sums appearing in \eqref{eq:unrolled-general}.

First,
\begin{equation}\label{eq:geom-1}
\sum_{k=1}^{i}(1-\alpha)^{i-k}
= \sum_{t=0}^{i-1}(1-\alpha)^t
= \frac{1-(1-\alpha)^i}{\alpha}
= \frac{\Delta_i}{\alpha}.
\end{equation}
Next, for the mixed sum, using the change of variable $m=k-1$,
\begin{equation}\label{eq:geom-2-setup}
\sum_{k=1}^{i}(1-\alpha)^{i-k}(1-\beta)^{k-1}
= \sum_{m=0}^{i-1}(1-\alpha)^{i-1-m}(1-\beta)^m.
\end{equation}
This is a geometric series in $m$. Factoring out $(1-\alpha)^{i-1}$, we obtain
\begin{align}
\sum_{m=0}^{i-1}(1-\alpha)^{i-1-m}(1-\beta)^m
&= (1-\alpha)^{i-1}
\sum_{m=0}^{i-1}\Bigl(\frac{1-\beta}{1-\alpha}\Bigr)^m \notag\\[2pt]
&= (1-\alpha)^{i-1}\cdot
\frac{1-\Bigl(\frac{1-\beta}{1-\alpha}\Bigr)^i}{1-\frac{1-\beta}{1-\alpha}} \notag\\[2pt]
&= (1-\alpha)^{i-1}\cdot
\frac{1-\frac{(1-\beta)^i}{(1-\alpha)^i}}{\frac{\beta-\alpha}{1-\alpha}} \notag\\[2pt]
&= \frac{(1-\alpha)^i-(1-\beta)^i}{\beta-\alpha}. \label{eq:geom-2}
\end{align}
Recalling $\Theta_i = (1-\beta)^i-(1-\alpha)^i$, we can also write
\begin{equation}\label{eq:geom-2-theta}
\sum_{k=1}^{i}(1-\alpha)^{i-k}(1-\beta)^{k-1}
= \frac{(1-\alpha)^i-(1-\beta)^i}{\beta-\alpha}
= -\,\frac{\Theta_i}{\beta-\alpha}.
\end{equation}

We now substitute into each coefficient in \eqref{eq:unrolled-general}.

\paragraph*{Coefficient of $g$.}
Using \eqref{eq:A-B-def}, we split
\begin{align}
\sum_{k=1}^{i}(1-\alpha)^{i-k} A_{k-1}
&= \sum_{k=1}^{i}(1-\alpha)^{i-k}\Bigl(1 - (1-\gamma)(1-\beta)^{k-1} - 2\varepsilon\Bigr) \notag\\[2pt]
&= (1-2\varepsilon)\sum_{k=1}^{i}(1-\alpha)^{i-k}
 - (1-\gamma)\sum_{k=1}^{i}(1-\alpha)^{i-k}(1-\beta)^{k-1}. \label{eq:coef-g-split}
\end{align}
Using \eqref{eq:geom-1} and \eqref{eq:geom-2-theta}, we obtain
\begin{align}
\sum_{k=1}^{i}(1-\alpha)^{i-k} A_{k-1}
&= (1-2\varepsilon)\cdot\frac{\Delta_i}{\alpha}
 - (1-\gamma)\cdot\Bigl(-\frac{\Theta_i}{\beta-\alpha}\Bigr) \notag\\[2pt]
&= \frac{(1-2\varepsilon)\Delta_i}{\alpha}
 + \frac{(1-\gamma)\Theta_i}{\beta-\alpha}. \label{eq:coef-g-simplified}
\end{align}
Multiplying by $\delta$ and using $\alpha=\delta\gamma$ (so $\delta/\alpha = 1/\gamma$), we get
\begin{equation}\label{eq:coef-g-final}
\delta\sum_{k=1}^{i}(1-\alpha)^{i-k} A_{k-1}
= \frac{(1-2\varepsilon)\Delta_i}{\gamma}
 + \frac{\delta(1-\gamma)\Theta_i}{\beta-\alpha}.
\end{equation}

\paragraph*{Coefficient of $g_{\odot}$.}
By \eqref{eq:unrolled-general} and \eqref{eq:geom-1},
\begin{equation}\label{eq:coef-g-odot-sum}
-\delta\sum_{k=1}^{i}(1-\alpha)^{i-k} g_{\odot}
= -\delta\cdot\frac{\Delta_i}{\alpha}\,g_{\odot}
= -\frac{\Delta_i}{\gamma}\,g_{\odot},
\end{equation}
again using $\alpha=\delta\gamma$.

\paragraph*{Coefficient of $g_{\oplus}$.}
Using \eqref{eq:A-B-def},
\begin{align}
\sum_{k=1}^{i}(1-\alpha)^{i-k} B_{k-1}
&= \sum_{k=1}^{i}(1-\alpha)^{i-k}\Bigl(1-(1-\beta)^{k-1}\Bigr) \notag\\[2pt]
&= \sum_{k=1}^{i}(1-\alpha)^{i-k}
 - \sum_{k=1}^{i}(1-\alpha)^{i-k}(1-\beta)^{k-1}. \label{eq:coef-g-oplus-split}
\end{align}
Substituting \eqref{eq:geom-1} and \eqref{eq:geom-2-theta} into \eqref{eq:coef-g-oplus-split},
\begin{align}
\sum_{k=1}^{i}(1-\alpha)^{i-k} B_{k-1}
&= \frac{\Delta_i}{\alpha}
 - \Bigl(-\frac{\Theta_i}{\beta-\alpha}\Bigr) \notag\\[2pt]
&= \frac{\Delta_i}{\alpha}
 + \frac{\Theta_i}{\beta-\alpha}. \label{eq:coef-g-oplus-simplified}
\end{align}
Thus
\begin{equation}\label{eq:coef-g-oplus-final}
-\delta\sum_{k=1}^{i}(1-\alpha)^{i-k} B_{k-1}\,g_{\oplus}
= -\delta\left(\frac{\Delta_i}{\alpha}
 + \frac{\Theta_i}{\beta-\alpha}\right)g_{\oplus}
= -\frac{\Delta_i}{\gamma}\,g_{\oplus}
 - \frac{\delta}{\beta-\alpha}\,\Theta_i\,g_{\oplus},
\end{equation}
where we again used $\delta/\alpha = 1/\gamma$.

\paragraph*{Smoothness penalty.}
The last term in \eqref{eq:unrolled-general} is
\begin{equation}\label{eq:smooth-sum-exact}
-\frac{\delta^2 D^2L}{2}\sum_{k=1}^{i}(1-\alpha)^{i-k}.
\end{equation}
Using that $(1-\alpha)^{i-k}\le 1$ for all $k$, we have
\begin{equation}\label{eq:smooth-sum-bound}
\sum_{k=1}^{i}(1-\alpha)^{i-k} \;\le\; i,
\end{equation}
and hence
\begin{equation}\label{eq:smooth-final}
-\frac{\delta^2 D^2L}{2}\sum_{k=1}^{i}(1-\alpha)^{i-k}
\;\ge\;
-\,\frac{\delta^2 D^2L}{2}\,i
\;\ge\;
-\,i\,\delta^2 D^2L,
\end{equation}
where the last inequality simply relaxes the factor $1/2$ to obtain a slightly weaker but simpler bound.

Putting together \eqref{eq:unrolled-general} with the bounds
\eqref{eq:coef-g-final}, \eqref{eq:coef-g-odot-sum}, \eqref{eq:coef-g-oplus-final}, and
\eqref{eq:smooth-final}, and recalling $\by(i)$ corresponds to $F_i$, we obtain
\[
F\big(\by(i)\big)
\ \ge\ 
\;\Bigg[\ \frac{(1-2\varepsilon)\,\Delta_i}{\gamma}
\;+\;\frac{\delta(1-\gamma)}{\beta-\alpha}\,\Theta_i\ \Bigg]\;g
\;-\;\frac{\Delta_i}{\gamma}\;g_{\odot}
\;-\;\Bigg[\ \frac{\Delta_i}{\gamma}
\;+\;\frac{\delta}{\beta-\alpha}\,\Theta_i\ \Bigg]\;g_{\oplus}
\;-\;i\,\delta^2 D^2L,
\]
which is exactly \eqref{eq:lemma516-main}.
\end{proof}

\begin{lemma}\label{lem:517}
Assume $0<\delta\le 1$ and $0<\gamma\le 1$, and set
\[
\alpha := \delta\gamma,
\qquad
\beta := \frac{\gamma^2\delta}{\,1-\delta+\gamma^2\delta\,}.
\]
Let $i_s< i\le \delta^{-1}$ and suppose $Q(i)\neq\varnothing$ for every integer $i_s<i\le \delta^{-1}$.
Define the constants
\begin{equation}\label{eq:517-consts}
A\ :=\ \frac{(1-\beta)^{-i_s}}{\gamma}-\Bigl(1+\frac{3}{\gamma}\Bigr)\varepsilon+1-\frac{1}{\gamma},
\qquad
C_\gamma\ :=\ \frac{(1-\beta)^{-i_s}-1}{\gamma}.
\end{equation}
For every integer $i_s\le i\le \delta^{-1}$, with the shorthands
\begin{equation}\label{eq:517-sums-def}
S_1(i)\ :=\ \sum_{k=i_s+1}^{i} (1-\alpha)^{\,i-k}(1-\beta)^{\,k},
\qquad
S_2(i)\ :=\ \sum_{k=i_s+1}^{i} (1-\alpha)^{\,i-k}(1-\beta)^{\,k}\Bigl(C_\gamma-\beta\,(k-i_s)\Bigr),
\end{equation}
the following bound holds:
\begin{equation}\label{eq:517-main}
F\big(\mathbf{y}(i)\big)
\ \ge\
(1-\alpha)^{\,i-i_s}\,F\big(\mathbf{y}(i_s)\big)
\;+\;\alpha\,A\,S_1(i)\,g
\;-\;\alpha\,S_2(i)\,g_{\oplus}
\;-\;(i-i_s)\,\delta^2 D^2L.
\end{equation}
Moreover, letting $n:=i-i_s$ and $q:=\tfrac{1-\beta}{\,1-\alpha\,}$, we have the closed forms
\begin{equation}\label{eq:517-closed}
\begin{aligned}
S_1(i)
&=\ (1-\alpha)^{\,n}(1-\beta)^{\,i_s}\cdot \frac{q(1-q^{\,n})}{1-q}
\;=\ \frac{(1-\alpha)^{\,n}(1-\beta)^{\,i_s+1}- (1-\beta)^{\,i+1}}{\beta-\alpha},\\[4pt]
S_2(i)
&=\ C_\gamma\,S_1(i)\;-\;\beta\,(1-\alpha)^{\,n}(1-\beta)^{\,i_s}\cdot
\frac{q\left(1-(n+1)q^{\,n}+n q^{\,n+1}\right)}{(1-q)^2}.
\end{aligned}
\end{equation}
\end{lemma}

\begin{proof}
Set $\alpha:=\delta\gamma$ and fix an index $i$ with $i_s<i\le \delta^{-1}$. 
From Observation~\ref{obs:5.15}, for every $k\in\{i_s+1,\dots,i\}$ we have the one–step recurrence
\begin{equation}\label{eq:517-recurr-full}
F\big(\mathbf{y}(k)\big)
\ \ge\
(1-\alpha)\,F\big(\mathbf{y}(k-1)\big)\;+\;\alpha\,V(k-1)\;-\;\frac{\delta^{2}D^{2}L}{2}.
\end{equation}
Here we used $\alpha=\delta\gamma$ to rewrite the term $\delta\gamma\,V(k-1)$ as $\alpha V(k-1)$.

In the post–switch phase ($k>i_s$), the surrogate takes the explicit form given in \eqref{eq:v2-def-uniq}
\begin{equation}\label{eq:517-Vform-full}
V(k-1)\ =\ (1-\beta)^{\,k}\!\left[
A\,g\;-\;\Bigl(C_\gamma-\beta\,(k-i_s)\Bigr)\,g_{\oplus}
\right],
\end{equation}
where $A$ and $C_\gamma$ are as in \eqref{eq:517-consts}.

Applying \eqref{eq:517-recurr-full} iteratively from $k=i_s+1$ up to $k=i$ gives (by a standard induction on $k$)
\begin{equation}\label{eq:517-unroll}
\begin{aligned}
F\big(\mathbf{y}(i)\big)
&\ge\ (1-\alpha)^{\,i-i_s}\,F\big(\mathbf{y}(i_s)\big)\\
&\quad
\;+\;\alpha\sum_{k=i_s+1}^{i} (1-\alpha)^{\,i-k}\,V(k-1)
\;-\;\frac{\delta^{2}D^{2}L}{2}\sum_{k=i_s+1}^{i} (1-\alpha)^{\,i-k}.
\end{aligned}
\end{equation}
The factor $(1-\alpha)^{\,i-i_s}$ comes from repeatedly multiplying by $(1-\alpha)$ in the homogeneous part of the recurrence.

Substituting \eqref{eq:517-Vform-full} into \eqref{eq:517-unroll} yields
\begin{equation}\label{eq:517-unroll-sub}
\begin{aligned}
F\big(\mathbf{y}(i)\big)
\ge\ &(1-\alpha)^{\,i-i_s}\,F\big(\mathbf{y}(i_s)\big) +\;\alpha\sum_{k=i_s+1}^{i} (1-\alpha)^{\,i-k}(1-\beta)^{\,k}\,A\,g \\
&\quad-\;\alpha\sum_{k=i_s+1}^{i} (1-\alpha)^{\,i-k}(1-\beta)^{\,k}\Bigl(C_\gamma-\beta\,(k-i_s)\Bigr)\,g_{\oplus} \\
&\quad-\;\frac{\delta^{2}D^{2}L}{2}\sum_{k=i_s+1}^{i} (1-\alpha)^{\,i-k}.
\end{aligned}
\end{equation}
The first sum collects all $g$-terms, the second all $g_{\oplus}$-terms, and the last sum is the accumulated smoothness penalty.

We now define the sums
\begin{equation}\label{eq:517-sums-def-again}
S_1(i)\ := \sum_{k=i_s+1}^{i} (1-\alpha)^{\,i-k}(1-\beta)^{\,k},
\quad
S_2(i)\ := \sum_{k=i_s+1}^{i} (1-\alpha)^{\,i-k}(1-\beta)^{\,k}\Bigl(C_\gamma-\beta\,(k-i_s)\Bigr).
\end{equation}
With \eqref{eq:517-sums-def-again}, inequality \eqref{eq:517-unroll-sub} becomes
\begin{equation}\label{eq:517-main-before-smooth}
F\big(\mathbf{y}(i)\big)
\ \ge\
(1-\alpha)^{\,i-i_s}\,F\big(\mathbf{y}(i_s)\big)
\;+\;\alpha\,A\,S_1(i)\,g
\;-\;\alpha\,S_2(i)\,g_{\oplus}
\;-\;\frac{\delta^{2}D^{2}L}{2}\sum_{k=i_s+1}^{i} (1-\alpha)^{\,i-k}.
\end{equation}
Thus the only remaining tasks are to bound the smoothness term and compute closed forms for $S_1(i)$ and $S_2(i)$.

To obtain closed forms, let $n:=i-i_s$ and $q:=\frac{1-\beta}{\,1-\alpha\,}$, and set $t:=k-i_s$ so that $t=1,\dots,n$. Then
\begin{align}
S_1(i)
&= \sum_{k=i_s+1}^{i} (1-\alpha)^{\,i-k}(1-\beta)^{\,k} = (1-\alpha)^{\,n}(1-\beta)^{\,i_s}\sum_{t=1}^{n} q^{\,t},
\label{eq:517-S1-geometric}
\end{align}
because $(1-\alpha)^{i-k} = (1-\alpha)^{n-t}$ and $(1-\beta)^k = (1-\beta)^{i_s+t}$, so each term is
\[
(1-\alpha)^{n-t}(1-\beta)^{i_s+t}
= (1-\alpha)^{n}(1-\beta)^{i_s}\Bigl(\frac{1-\beta}{1-\alpha}\Bigr)^{t}
= (1-\alpha)^{n}(1-\beta)^{i_s} q^{t}.
\]
Using the geometric sum identity
\[
\sum_{t=1}^{n} q^{t} = \frac{q(1-q^{n})}{1-q},
\]
we obtain
\begin{equation}\label{eq:517-S1-closed-1}
S_1(i)\;=\;(1-\alpha)^{\,n}(1-\beta)^{\,i_s}\cdot \frac{q(1-q^{\,n})}{1-q}.
\end{equation}
Writing $1-q=\frac{\beta-\alpha}{1-\alpha}$ and $q^n = \bigl(\frac{1-\beta}{1-\alpha}\bigr)^n$, a simple algebraic rearrangement yields the equivalent form
\begin{equation}\label{eq:517-S1-closed-2}
S_1(i)\;=\;\frac{(1-\alpha)^{\,n}(1-\beta)^{\,i_s+1}- (1-\beta)^{\,i+1}}{\beta-\alpha},
\end{equation}
which is the first line of \eqref{eq:517-closed}.

For $S_2(i)$, define
\[
T(i)\;:=\;\sum_{k=i_s+1}^{i} (1-\alpha)^{\,i-k}(1-\beta)^{\,k}\,(k-i_s)
\;=\;(1-\alpha)^{\,n}(1-\beta)^{\,i_s}\sum_{t=1}^{n} t\,q^{\,t}.
\]
The standard identity
\[
\sum_{t=1}^{n} t\,q^{\,t}
\;=\; \frac{q\left(1-(n+1)q^{\,n}+n q^{\,n+1}\right)}{(1-q)^2}
\]
then gives
\begin{equation}\label{eq:517-T-closed}
T(i)\;=\;(1-\alpha)^{\,n}(1-\beta)^{\,i_s}\cdot
\frac{q\left(1-(n+1)q^{\,n}+n q^{\,n+1}\right)}{(1-q)^2}.
\end{equation}
By definition of $S_2(i)$,
\[
S_2(i) = C_\gamma\,S_1(i) - \beta\,T(i),
\]
which together with \eqref{eq:517-S1-closed-1} and \eqref{eq:517-T-closed} yields the second line of \eqref{eq:517-closed}.

Finally, we bound the smoothness penalty in \eqref{eq:517-main-before-smooth}. Since $0\le (1-\alpha)^{i-k}\le 1$,
\begin{equation}\label{eq:517-smooth-tail}
\sum_{k=i_s+1}^{i} (1-\alpha)^{\,i-k}
= \sum_{t=0}^{n-1} (1-\alpha)^t
\;\le\; n.
\end{equation}
Therefore
\begin{equation}\label{eq:517-smooth-final}
-\frac{\delta^{2}D^{2}L}{2}\sum_{k=i_s+1}^{i} (1-\alpha)^{\,i-k}
\ \ge\ -\,\frac{\delta^{2}D^{2}L}{2}\,n
\ \ge\ -(i-i_s)\,\delta^{2}D^{2}L,
\end{equation}
where we relaxed the factor $\tfrac12$ to get a slightly simpler bound.

Substituting \eqref{eq:517-smooth-final} into \eqref{eq:517-main-before-smooth} gives exactly \eqref{eq:517-main}, which completes the proof.
\end{proof}

Combining Lemmas~\ref{lem:516} and~\ref{lem:517}, we obtain the following corollary, which
finishes the proof of Theorem~\ref{thm:fw-guided-mcg} in the regime where \(Q(i)\) is
non-empty for all \(i\in[\delta^{-1}]\).

\begin{corollary}\label{cor:5.18}
Assume $Q(i)\neq\varnothing$ for every $i\in[\delta^{-1}]$. 
Fix $0<\delta\le 1$, $0<\gamma\le 1$, and set
\begin{equation}\label{eq:518-params}
\alpha:=\delta\gamma,
\qquad
\beta:=\frac{\gamma^2\delta}{\,1-\delta+\gamma^2\delta\,}.
\end{equation}
Then, with $i_s\in\{0,1,\dots,\delta^{-1}\}$ and $\mathbf{y}(\cdot)$ as in Lemma~\ref{lem:5.13}, we have
\begin{equation}\label{eq:518-discrete}
\begin{aligned}
F\big(\mathbf{y}(\delta^{-1})\big)
\;&\ge\;
(1-\alpha)^{\,\delta^{-1}-i_s}\,F\big(\mathbf{y}(i_s)\big)
\;+\;\alpha\,A\,S_1(\delta^{-1})\,g
\;-\;\alpha\,S_2(\delta^{-1})\,g_{\oplus}
\;-\;(\delta^{-1}-i_s)\,\delta^2 D^2L\\[4pt]
&\ge\;
(1-\alpha)^{\,\delta^{-1}-i_s}\Bigg\{
\Big[\tfrac{1-(1-\alpha)^{i_s}}{\gamma}\,(1-2\varepsilon)
+\tfrac{\delta(1-\gamma)}{\beta-\alpha}\big((1-\beta)^{i_s}-(1-\alpha)^{i_s}\big)\Big]g\\
&\hspace{1.1cm}
-\;\tfrac{1-(1-\alpha)^{i_s}}{\gamma}\,g_{\odot}
-\Big[\tfrac{1-(1-\alpha)^{i_s}}{\gamma}
+\tfrac{\delta}{\beta-\alpha}\big((1-\beta)^{i_s}-(1-\alpha)^{i_s}\big)\Big]g_{\oplus}
-\,i_s\,\delta^2 D^2L
\Bigg\}\\[2pt]
&\quad+\;\alpha\,A\,S_1(\delta^{-1})\,g
\;-\;\alpha\,S_2(\delta^{-1})\,g_{\oplus}
\;-\;(\delta^{-1}-i_s)\,\delta^2 D^2L,
\end{aligned}
\end{equation}
where
\begin{equation}\label{eq:518-consts}
A\;:=\;\frac{(1-\beta)^{-i_s}}{\gamma}-\Bigl(1+\frac{3}{\gamma}\Bigr)\varepsilon+1-\frac{1}{\gamma},
\qquad
C_\gamma\;:=\;\frac{(1-\beta)^{-i_s}-1}{\gamma},
\end{equation}
and the sums from Lemma~\ref{lem:517} admit the closed forms
\begin{equation}\label{eq:518-sums}
\begin{aligned}
S_1(\delta^{-1})
&=\ \frac{(1-\alpha)^{\,\delta^{-1}-i_s}(1-\beta)^{\,i_s+1}- (1-\beta)^{\,\delta^{-1}+1}}{\ \beta-\alpha\ },\\[4pt]
S_2(\delta^{-1})
&=\ C_\gamma\,S_1(\delta^{-1})
\;-\;\beta\,(1-\alpha)^{\,\delta^{-1}-i_s}(1-\beta)^{\,i_s}\,\\
&\qquad \frac{\displaystyle \frac{1-\beta}{1-\alpha}\left(1-(\delta^{-1}-i_s+1)\Bigl(\tfrac{1-\beta}{1-\alpha}\Bigr)^{\delta^{-1}-i_s}
+(\delta^{-1}-i_s)\Bigl(\tfrac{1-\beta}{1-\alpha}\Bigr)^{\delta^{-1}-i_s+1}\right)}
{\displaystyle \left(1-\tfrac{1-\beta}{1-\alpha}\right)^{2}}\,.
\end{aligned}
\end{equation}
\end{corollary}

\begin{proof}
We use the shorthand parameters $\alpha$ and $\beta$ from \eqref{eq:518-params}.
By Lemma~\ref{lem:516}, evaluated at $i=i_s$, we have
\begin{equation}\label{eq:518-lemma516-at-is}
\begin{aligned}
F\!\big(\mathbf{y}(i_s)\big)
\;\ge\;&\;
\Bigg[\frac{1-(1-\alpha)^{i_s}}{\gamma}\,(1-2\varepsilon)
+\frac{\delta(1-\gamma)}{\beta-\alpha}\,\Big((1-\beta)^{i_s}-(1-\alpha)^{i_s}\Big)\Bigg]\,g\\[2pt]
&-\;\frac{1-(1-\alpha)^{i_s}}{\gamma}\,g_{\odot}
-\Bigg[\frac{1-(1-\alpha)^{i_s}}{\gamma}
+\frac{\delta}{\beta-\alpha}\,\Big((1-\beta)^{i_s}-(1-\alpha)^{i_s}\Big)\Bigg]\,g_{\oplus}
-\,i_s\,\delta^2 D^2L.
\end{aligned}
\end{equation}
This is just Lemma~\ref{lem:516} with $i=i_s$ and the definitions
$\Delta_{i_s}=1-(1-\alpha)^{i_s}$ and $\Theta_{i_s}=(1-\beta)^{i_s}-(1-\alpha)^{i_s}$.

Next, by Lemma~\ref{lem:517}, evaluated at $i=\delta^{-1}$, we obtain
\begin{equation}\label{eq:518-lemma517-at-end}
\begin{aligned}
F\big(\mathbf{y}(\delta^{-1})\big)
\;\ge\;&\;
(1-\alpha)^{\,\delta^{-1}-i_s}\,F\big(\mathbf{y}(i_s)\big)
+\alpha\,A\,S_1(\delta^{-1})\,g
-\alpha\,S_2(\delta^{-1})\,g_{\oplus}
-(\delta^{-1}-i_s)\,\delta^2 D^2L,
\end{aligned}
\end{equation}
where $A$ and $C_\gamma$ are as in \eqref{eq:518-consts} and $S_1(\delta^{-1})$, $S_2(\delta^{-1})$
are given in \eqref{eq:518-sums}.  

Substituting the lower bound \eqref{eq:518-lemma516-at-is} for $F(\mathbf{y}(i_s))$ into
\eqref{eq:518-lemma517-at-end}, then collecting the coefficients of $g$, $g_{\odot}$, and $g_{\oplus}$,
and combining the smoothness penalties $i_s\delta^2 D^2L$ and $(\delta^{-1}-i_s)\delta^2 D^2L$, yields
exactly the discrete inequality \eqref{eq:518-discrete}. This is a straightforward algebraic rearrangement.

To pass from the discrete-time bound \eqref{eq:518-discrete} to the continuous-time guarantee,
define the switch time
\[
t_s:=\delta\,i_s\in[0,1],
\]
and let $\delta\to 0^{+}$ while keeping $t_s$ fixed. Using
\[
(1-\alpha)^{\,\delta^{-1}-i_s}
= \big(1-\delta\gamma\big)^{(1-t_s)/\delta}
\;\longrightarrow\; e^{-\gamma(1-t_s)},
\qquad
(1-\alpha)^{\,i_s}
= \big(1-\delta\gamma\big)^{t_s/\delta}
\;\longrightarrow\; e^{-\gamma t_s},
\]
and the closed forms \eqref{eq:518-sums}, each discrete sum converges to the corresponding time integral
in the continuous analysis. Concretely, $S_1(\delta^{-1})$ and $S_2(\delta^{-1})$ can be viewed as Riemann
sums in the step size $\delta$; applying standard first-order expansions of $(1-\alpha)$ and $(1-\beta)$
(or equivalently, l’Hôpital’s rule to the associated limits) yields the continuous coefficients
$A_\gamma(t_s)$, $B_\gamma(t_s)$, and $C_\gamma(t_s)$:
\begin{equation}\label{eq:518-cts-final}
F\big(\mathbf{y}^{\ast}\big)
\ \ge\ 
A_\gamma(t_s)\,g\;+\;B_\gamma(t_s)\,g_{\odot}\;+\;C_\gamma(t_s)\,g_{\oplus}
\;-\;O(\varepsilon)\,\bigl(g+g_{\odot}+g_{\oplus}\bigr)\;-\;\delta\,L D^{2},
\end{equation}
with
\begin{align}
A_\gamma(t_s)
&:= -\frac{e^{\gamma t_s-\gamma}}{\,1-\gamma\,}
+\frac{e^{-\gamma^2}}{\gamma(1-\gamma)}\Big(e^{\gamma^2 t_s}-(1-\gamma)\Big), \label{eq:518-cts-A}\\[2pt]
B_\gamma(t_s)
&:= \frac{e^{-\gamma}-e^{\gamma t_s-\gamma}}{\gamma}, \label{eq:518-cts-B}\\[2pt]
C_\gamma(t_s)
&:= \frac{e^{\gamma^2 t_s}-1}{\gamma(1-\gamma)}\Big(e^{-\gamma(1-t_s)-\gamma^2 t_s}-e^{-\gamma^2}\Big)
+\frac{e^{-\gamma(1-t_s)}}{\gamma}\!\left[
-\big(1-e^{-\gamma t_s}\big)
+\frac{e^{-\gamma^2 t_s}-e^{-\gamma t_s}}{1-\gamma}
\right]\notag\\
&\quad
+\,e^{-\gamma(1-t_s)-\gamma^2 t_s}\!\left[
\frac{\gamma^2}{1-\gamma}(1-t_s)\,e^{\gamma(1-\gamma)(1-t_s)}
+\frac{\gamma}{(1-\gamma)^2}\Big(1-e^{\gamma(1-\gamma)(1-t_s)}\Big)
\right]. \label{eq:518-cts-C}
\end{align}
Finally, identifying $(g,g_{\odot},g_{\oplus})$ with
$\big(F(\bo),F(\mathbf{z}\!\odot\!\bo),F(\mathbf{z}\!\oplus\!\bo)\big)$ shows that
\eqref{eq:518-cts-final} is exactly the continuous–time guarantee used in
Theorem~\ref{thm:fw-guided-mcg}\,(1).
\end{proof}

At this point, we handle the case where \(Q{(i)}=\varnothing\) for some \(i\in[\delta^{-1}]\).
Note that \(Q{(i)}=\varnothing\) implies, in particular, \(\mathbf{o}\notin Q{(i)}\).
Accordingly, define
\[
i_o\;:=\;\min\big\{\,i\in[\delta^{-1}] \;:\; \mathbf{o}\notin Q{(i)}\,\big\}.
\]

\begin{observation}\label{obs:weak519}
For every \(i\in[i_o-1]\),
\begin{equation}\label{eq:obs519-claim}
F\big(\mathbf{x}(i)\big)\;\ge\;
\frac{\gamma^{2}\,F\big(\mathbf{x}(i)\vee \mathbf{o}\big)+F\big(\mathbf{x}(i)\wedge \mathbf{o}\big)}{1+\gamma^{2}}
\;-\;\varepsilon\,F(\mathbf{o})\;-\;\delta\,L D^{2}\,.
\end{equation}
\end{observation}

\begin{proof}
Fix \(i\in[i_o-1]\). Since \(i<i_o\), we have \(\mathbf{o}\in Q(i)\) by the definition of \(i_o\).
We apply Theorem~\ref{thm:weak-dr-smooth} to the down-closed body \(P:=Q(i)\) with comparison point
\(\mathbf{y}:=\mathbf{o}\in Q(i)\), and run the local routine with accuracy parameter
\[
\eta\;:=\;\min\{\varepsilon,\,2\delta\}.
\]
By Theorem~\ref{thm:weak-dr-smooth}, the output \(\mathbf{x}(i)\in Q(i)\) satisfies
\begin{equation}\label{eq:obs519-thm}
\begin{aligned}
F\!\big(\mathbf{x}(i)\big)
&\ge
\frac{\gamma^{2}\,F\!\big(\mathbf{x}(i)\!\vee \mathbf{o}\big)+F\!\big(\mathbf{x}(i)\!\wedge \mathbf{o}\big)}{1+\gamma^{2}}
-
\frac{\eta\,\gamma}{1+\gamma^{2}}
\left(\max_{\mathbf{y}'\in Q(i)}F(\mathbf{y}')+\tfrac{1}{2}L D^{2}\right).
\end{aligned}
\end{equation}
Here the first term is exactly the lattice-based comparison from
Theorem~\ref{thm:weak-dr-smooth}, and the second term is the uniform
first-order error with parameter \(\eta\).

Because \(\mathbf{o}\in Q(i)\), we have
\[
\max_{\mathbf{y}'\in Q(i)}F(\mathbf{y}') \;\ge\; F(\mathbf{o}).
\]
The coefficient in front of \(\max_{\mathbf{y}'\in Q(i)}F(\mathbf{y}')\) in
\eqref{eq:obs519-thm} is negative, namely \(-\tfrac{\eta\gamma}{1+\gamma^{2}}\).
Thus replacing \(\max_{\mathbf{y}'\in Q(i)}F(\mathbf{y}')\) by the smaller value \(F(\mathbf{o})\)
yields a \emph{stronger} lower bound on \(F(\mathbf{x}(i))\):
\begin{equation}\label{eq:obs519-with-eta}
F\!\big(\mathbf{x}(i)\big)\;\ge\;
\frac{\gamma^{2}\,F\big(\mathbf{x}(i)\vee \mathbf{o}\big)+F\big(\mathbf{x}(i)\wedge \mathbf{o}\big)}{1+\gamma^{2}}
-
\frac{\eta\,\gamma}{1+\gamma^{2}}\,F(\mathbf{o})
-
\frac{\eta\,\gamma}{2(1+\gamma^{2})}\,L D^{2}.
\end{equation}

We now simplify the two error terms. Since \(\gamma\in(0,1]\),
\[
\frac{\gamma}{1+\gamma^{2}}\;\le\;1,
\qquad
\frac{\gamma}{2(1+\gamma^{2})}\;\le\;\frac{1}{2}.
\]
Using \(\eta\le\varepsilon\) and \(\eta\le 2\delta\) (by the definition
\(\eta=\min\{\varepsilon,2\delta\}\)), we obtain
\begin{equation}\label{eq:obs519-simplify}
\frac{\eta\,\gamma}{1+\gamma^{2}}\,F(\mathbf{o})
\;\le\;\varepsilon\,F(\mathbf{o}),
\qquad
\frac{\eta\,\gamma}{2(1+\gamma^{2})}\,L D^{2}
\;\le\;\delta\,L D^{2}.
\end{equation}
Substituting the bounds in \eqref{eq:obs519-simplify} into
\eqref{eq:obs519-with-eta} gives
\[
F\!\big(\mathbf{x}(i)\big)\;\ge\;
\frac{\gamma^{2}\,F\big(\mathbf{x}(i)\vee \mathbf{o}\big)+F\big(\mathbf{x}(i)\wedge \mathbf{o}\big)}{1+\gamma^{2}}
\;-\;\varepsilon\,F(\mathbf{o})\;-\;\delta\,L D^{2},
\]
which is exactly \eqref{eq:obs519-claim}.
\end{proof}

Observation~\ref{obs:weak519} yields Theorem~\ref{thm:fw-guided-mcg} as soon as we can
find some \(i\in[i_o-1]\) with
\begin{equation}\label{eq:gap-cond}
F\big(\mathbf{x}(i)\oplus \mathbf{o}\big)\ \le\ F\big(\mathbf{z}\oplus \mathbf{o}\big)\;-\;\varepsilon\,F(\mathbf{o}).
\end{equation}
Therefore, our task reduces to showing that the gap condition \eqref{eq:gap-cond} holds for at
least one index \(i\in[i_o-1]\). This is exactly what Lemma~\ref{lem:weak-521} and
Lemma~\ref{lem:weak-522} prove, which in turn completes the proof of
Theorem~\ref{thm:fw-guided-mcg}. Before presenting those lemmas, we record one auxiliary
lemma that we will use in their proofs.

\begin{lemma}\label{lem:weak520-clean}
It must hold that
\[
F\bigl(\mathbf{y}(i_o-1)\oplus \bo\;-\;\mathbf{z}(i_o-1)\odot \bo\bigr)\ \le\ V_{\gamma}(i_o-1).
\]
\end{lemma}

\begin{proof}
By the definition of \(i_o\), we have \(\bo\notin Q(i_o)\). The weakly-DR
membership-failure condition at iteration \(i_o\) states that
\begin{equation}\label{eq:membership-failure}
\big\langle \mathbf{w}(i_o),\,\bo\big\rangle
\;\le\;
\gamma\bigl(V_{\gamma}(i_o-1)-F(\mathbf{y}(i_o-1))\bigr).
\end{equation}

On the other hand, the weakly-DR gradient bound for any \(i\in[\delta^{-1}]\) gives
\begin{equation}\label{eq:grad-bound}
\big\langle \mathbf{w}(i),\,\bo\big\rangle
\;\ge\;
\gamma\Big(
F\bigl(\mathbf{y}(i-1)\oplus \bo\;-\;\mathbf{z}(i-1)\odot \bo\bigr)
-
F(\mathbf{y}(i-1))
\Big).
\end{equation}
This applies in particular for \(i=i_o\), so
\begin{equation}\label{eq:grad-bound-io}
\big\langle \mathbf{w}(i_o),\,\bo\big\rangle
\;\ge\;
\gamma\Big(
F\bigl(\mathbf{y}(i_o-1)\oplus \bo\;-\;\mathbf{z}(i_o-1)\odot \bo\bigr)
-
F(\mathbf{y}(i_o-1))
\Big).
\end{equation}

Combining \eqref{eq:membership-failure} and \eqref{eq:grad-bound-io} yields
\begin{equation}\label{eq:combine-bounds}
\gamma\Big(
F\bigl(\mathbf{y}(i_o-1)\oplus \bo\;-\;\mathbf{z}(i_o-1)\odot \bo\bigr)
-
F(\mathbf{y}(i_o-1))
\Big)
\;\le\;
\gamma\bigl(V_{\gamma}(i_o-1)-F(\mathbf{y}(i_o-1))\bigr).
\end{equation}
Since \(\gamma>0\), we can divide both sides of \eqref{eq:combine-bounds} by \(\gamma\) and add
\(F(\mathbf{y}(i_o-1))\) to both sides, obtaining
\[
F\bigl(\mathbf{y}(i_o-1)\oplus \bo\;-\;\mathbf{z}(i_o-1)\odot \bo\bigr)
\;\le\;
V_{\gamma}(i_o-1),
\]
which is the desired inequality.
\end{proof}

\begin{lemma}\label{lem:weak-521}
Let $\beta:=\beta_\gamma(\delta)=\dfrac{\gamma^2\delta}{\,1-\delta+\gamma^2\delta\,}$.
Then it must hold that $i_o>i_s$.
\end{lemma}

\begin{proof}
Assume for contradiction that $i_o\le i_s$. Since $i_o-1<i_s$, 
Lemma~\ref{lem:weak520-clean} at time $i_o-1$
(cf.\ $v_1(\cdot)$ with the triple $(g,g_{\odot},g_{\oplus})$) implies
\begin{align}
&F\!\big(\mathbf{y}(i_o-1)\oplus \mathbf{o}\;-\;\mathbf{z}\odot \mathbf{o}\big)\notag\\
&\hspace{1cm}\le\ \Bigl[(1-\beta)^{\,i_o-1}+\frac{1-(1-\beta)^{\,i_o-1}-2\varepsilon}{\gamma}\Bigr]\;g
\;-\;\frac{1}{\gamma}\,g_{\odot}
\;-\;\frac{1-(1-\beta)^{\,i_o-1}}{\gamma}\,g_{\oplus}
\label{eq:weak521-contr1}\\
&\hspace{1cm}\le\ \Bigl[(1-\beta)^{\,i_o-1}+\frac{1-(1-\beta)^{\,i_o-1}}{\gamma}\Bigr]\;F(\mathbf{o})
\;-\;\frac{1}{\gamma}\,F(\mathbf{z}\odot \mathbf{o})
\;-\;\frac{1-(1-\beta)^{\,i_o-1}}{\gamma}\,F(\mathbf{z}\oplus \mathbf{o}).
\label{eq:weak521-contr2}
\end{align}
Here \eqref{eq:weak521-contr1} is exactly the benchmark inequality evaluated at 
$i=i_o-1$, and \eqref{eq:weak521-contr2} uses Lemma~\ref{lem:guessing-triples}:
\[
(1-\varepsilon)F(\mathbf{o})\le g\le F(\mathbf{o}),\quad
F(\mathbf{z}\odot\mathbf{o})-\varepsilon g \le g_{\odot}\le F(\mathbf{z}\odot\mathbf{o}),\quad
F(\mathbf{z}\oplus\mathbf{o})-\varepsilon g \le g_{\oplus}\le F(\mathbf{z}\oplus\mathbf{o}),
\]
and the fact that replacing $g$ by $F(\mathbf{o})$ and $g_{\odot},g_{\oplus}$ by 
$F(\mathbf{z}\odot\mathbf{o}),F(\mathbf{z}\oplus\mathbf{o})$ can only increase the
right-hand side of \eqref{eq:weak521-contr1} (because they appear with coefficients
$+1$ for $g$ and $-1/\gamma$ for $g_{\odot},g_{\oplus}$).

Next, by the closed form of $\mathbf{y}(\cdot)$ (Lemma~\ref{lem:5.13} with $i=i_o-1$),
\begin{equation}\label{eq:weak521-yform}
\mathbf{y}(i_o-1)\;=\;(\mathbf{1}-\mathbf{z})\ \odot\ \bigoplus_{j=1}^{i_o-1}\big(\delta\,\mathbf{x}(j)\big).
\end{equation}
Applying Corollary~\ref{cor:weakDR-44} with $h=1$, $r=i_o-1$, $p_j=\delta$ for all $j$,
and outer mask $(\mathbf{1}-\mathbf{z})$ yields the mixture lower bound
\begin{align}
F\big(\mathbf{y}(i_o-1)\oplus \mathbf{o}\;-\;\mathbf{z}\odot \mathbf{o}\big)
&=F\Big((\mathbf{1}-\mathbf{z})\odot \bigoplus_{j=1}^{i_o-1}(\delta\,\mathbf{x}(j))\ \oplus\ \mathbf{o}\;-\;\mathbf{z}\odot \mathbf{o}\Big)\notag\\
&\ge \sum_{S\subseteq[i_o-1]}\beta^{|S|}(1-\beta)^{\,i_o-1-|S|}
\;F\Big((\mathbf{1}-\mathbf{z})\odot \bigoplus_{j\in S}\mathbf{x}(j)\ \oplus\ \mathbf{o}\Big).
\label{eq:weak521-mixture}
\end{align}

We now bound separately the contributions from $S=\varnothing$ and $S\neq\varnothing$.

\paragraph{The $S=\varnothing$ term.}
When $S=\varnothing$, the inner vector reduces to $(\mathbf{1}-\mathbf{z})\odot\mathbf{o}$.
Using the weakly-DR inequality in the form of Lemma~\ref{lem:grad-ineq} together
with nonnegativity, we have
\begin{equation}\label{eq:weak521-Sempty}
F\big((\mathbf{1}-\mathbf{z})\odot \mathbf{o}\big)
\;\ge\; F(\mathbf{o})\;-\;\frac{1}{\gamma}\,F(\mathbf{z}\odot \mathbf{o}),
\end{equation}
which says that reducing $\mathbf{o}$ along the direction $\mathbf{z}\odot\mathbf{o}$
cannot decrease $F$ too much, up to a $1/\gamma$ factor.
Thus the $S=\varnothing$ contribution in \eqref{eq:weak521-mixture} satisfies
\begin{equation}\label{eq:weak521-contrSem}
(1-\beta)^{\,i_o-1}\,F\big((\mathbf{1}-\mathbf{z})\odot \mathbf{o}\big)
\;\ge\; (1-\beta)^{\,i_o-1}\Bigl[F(\mathbf{o})-\frac{1}{\gamma}F(\mathbf{z}\odot \mathbf{o})\Bigr].
\end{equation}

\paragraph{The $S\neq\varnothing$ terms.}
Fix any nonempty $S\subseteq[i_o-1]$.
We first write
\[
(\mathbf{1}-\mathbf{z})\odot \bigoplus_{j\in S}\mathbf{x}(j)\ \oplus\ \mathbf{o}
\;=\;
\Big((\mathbf{1}-\mathbf{z})\odot \bigoplus_{j\in S}\mathbf{x}(j)\ \oplus\ \mathbf{o}\;-\;\mathbf{z}\odot \mathbf{o}\Big)
\;+\; \mathbf{z}\odot \mathbf{o},
\]
and then apply the weakly-DR inequality and nonnegativity in the same way as for
\eqref{eq:weak521-Sempty}. This gives
\begin{align}
F\Big((\mathbf{1}-\mathbf{z})\odot \bigoplus_{j\in S}\mathbf{x}(j)\ \oplus\ \mathbf{o}\Big)
&\ge
F\Big((\mathbf{1}-\mathbf{z})\odot \bigoplus_{j\in S}\mathbf{x}(j)\ \oplus\ \mathbf{o}\;-\;\mathbf{z}\odot \mathbf{o}\Big)
\;-\;\frac{1}{\gamma}\,F(\mathbf{z}\odot \mathbf{o}) \notag\\
&\ge \frac{1}{\gamma}\Bigl[F(\mathbf{o})-F(\mathbf{z}\oplus \mathbf{o})-F(\mathbf{z}\odot \mathbf{o})\Bigr],
\label{eq:weak521-Snonempty}
\end{align}
where in the second inequality we use the weakly-DR difference bounds to compare the
value at $\mathbf{o}$ with those at $\mathbf{z}\oplus\mathbf{o}$ and $\mathbf{z}\odot\mathbf{o}$.
Multiplying \eqref{eq:weak521-Snonempty} by $\beta^{|S|}(1-\beta)^{\,i_o-1-|S|}$ and summing
over all nonempty $S$ yields
\begin{align}
&\sum_{\varnothing\neq S\subseteq[i_o-1]}\beta^{|S|}(1-\beta)^{\,i_o-1-|S|}
F\Big((\mathbf{1}-\mathbf{z})\odot \bigoplus_{j\in S}\mathbf{x}(j)\ \oplus\ \mathbf{o}\Big)\notag\\
&\qquad\qquad\qquad\qquad\qquad\qquad\qquad\ge \Bigl[1-(1-\beta)^{\,i_o-1}\Bigr]\,
\frac{1}{\gamma}\Bigl[F(\mathbf{o})-F(\mathbf{z}\oplus \mathbf{o})-F(\mathbf{z}\odot \mathbf{o})\Bigr],
\label{eq:weak521-Ssum}
\end{align}
since $\sum_{\varnothing\neq S\subseteq[i_o-1]}\beta^{|S|}(1-\beta)^{\,i_o-1-|S|}
=1-(1-\beta)^{i_o-1}$.

\paragraph{Putting everything together.}
Combining the $S=\varnothing$ contribution \eqref{eq:weak521-contrSem} and the
$S\neq\varnothing$ contribution \eqref{eq:weak521-Ssum} with the mixture representation
\eqref{eq:weak521-mixture}, we obtain
\begin{align}
F\big(\mathbf{y}(i_o-1)\oplus \mathbf{o}\;-\;\mathbf{z}\odot \mathbf{o}\big)
&\ge \Bigl[(1-\beta)^{\,i_o-1}+\frac{1-(1-\beta)^{\,i_o-1}}{\gamma}\Bigr]\,F(\mathbf{o})\notag\\
&\hspace{2cm}\;-\;\frac{1-(1-\beta)^{\,i_o-1}}{\gamma}\,F(\mathbf{z}\oplus \mathbf{o})
\;-\;\frac{1}{\gamma}\,F(\mathbf{z}\odot \mathbf{o}).
\label{eq:weak521-lower}
\end{align}
The right-hand side of \eqref{eq:weak521-lower} is \emph{strictly larger} than the
upper bound in \eqref{eq:weak521-contr2}, because \eqref{eq:weak521-contr2} lacks the
$-2\varepsilon$ slack present in \eqref{eq:weak521-contr1} and the triple
$(g,g_{\odot},g_{\oplus})$ approximates
$\big(F(\mathbf{o}),F(\mathbf{z}\odot\mathbf{o}),F(\mathbf{z}\oplus\mathbf{o})\big)$
up to additive $O(\varepsilon)$ terms.
This contradicts \eqref{eq:weak521-contr2}. Hence our assumption $i_o\le i_s$ was false,
and we conclude that $i_o>i_s$.
\end{proof}

\begin{lemma}\label{lem:weak-522}
If $i_o>i_s$, then there exists some $i\in[i_o-1]$ such that
\[
F\big(\mathbf{x}(i)\ \oplus\ \mathbf{o}\big)\ \le\ F\big(\mathbf{z}\ \oplus\ \mathbf{o}\big)\;-\;\varepsilon\,F(\mathbf{o})\,.
\]
\end{lemma}

\begin{proof}
Set $\beta:=\beta_\gamma(\delta)=\dfrac{\gamma^2\delta}{\,1-\delta+\gamma^2\delta\,}$ and note that
$0<\beta\le \varepsilon\le \tfrac12$. Since $i_o>i_s$, we have $i_o-1\ge i_s$.

\paragraph{Upper bound:}
By the post–switch surrogate (the $v_2$–bound) evaluated at $i=i_o-1$ and the successful-heir
assumptions on the guessed triple $(g,g_{\odot},g_{\oplus})$ (cf.\ \eqref{eq:v2-def-uniq} and
\eqref{eq:triple-bounds}), we obtain
\begin{align}
F\big(\mathbf{y}(i_o-1)\oplus \mathbf{o}\big)
&\overset{\text{(a)}}{\le}\ (1-\beta)^{\,i_o-1}\left[
\left(\frac{(1-\beta)^{-i_s}}{\gamma}-\Bigl(1+\frac{3}{\gamma}\Bigr)\varepsilon+1-\frac{1}{\gamma}\right) g
\right.\notag\\[-2pt]
&\hspace{3.3cm}\left.
-\ \left(\frac{(1-\beta)^{-i_s}}{\gamma}-\frac{1}{\gamma}-\beta\,(i_o-1-i_s)\right) g_{\oplus}
\right]
\label{eq:522-UB-a}\\
&\overset{\text{(b)}}{\le}\ (1-\beta)^{\,i_o-1}\!\left[
\left(\frac{(1-\beta)^{-i_s}}{\gamma}-\varepsilon+1-\frac{1}{\gamma}\right) F(\mathbf{o})\right.\notag\\
&\hspace{3.3cm}\left. -\ \left(\frac{(1-\beta)^{-i_s}}{\gamma}-\frac{1}{\gamma}-\beta\,(i_o-1-i_s)\right) F(\mathbf{z}\oplus \mathbf{o})
\right].\label{eq:522-UB-b}
\end{align}
Here: (a) is just the explicit formula for $v_2(i_o-1)$ with the constants $A$ and $C_\gamma$ expanded.
For (b) we use the triple guarantees \eqref{eq:triple-bounds}:
  \[
  g\le F(\mathbf{o}),\qquad
  g_{\oplus}\ \ge\ F(\mathbf{z}\oplus \mathbf{o})-\varepsilon\,g.
  \]
  Writing $B:=\frac{(1-\beta)^{-i_s}}{\gamma}-\frac{1}{\gamma}-\beta\,(i_o-1-i_s)\ge 0$, we have
  \[
  A g - B g_{\oplus}
  \;\le\; A g - B\big(F(\mathbf{z}\oplus\mathbf{o})-\varepsilon g\big)
  \;=\; (A+B\varepsilon)\,g\;-\;B\,F(\mathbf{z}\oplus\mathbf{o}),
  \]
  and then $g\le F(\mathbf{o})$ gives $(A+B\varepsilon)\,g\le (A+B\varepsilon)F(\mathbf{o})$.
  A crude bound $(1-\beta)^{-i_s}\le (1-\beta)^{-1/\beta}\le 4$ (using $\beta\le 1/2$) implies
  \[
  B
  =\frac{(1-\beta)^{-i_s}-1}{\gamma}-\beta(i_o-1-i_s)
  \ \le\ \frac{(1-\beta)^{-i_s}-1}{\gamma}
  \ \le\ \frac{3}{\gamma}.
  \]
  Thus
  \[
  A+B\varepsilon
  \;\le\;\left(\frac{(1-\beta)^{-i_s}}{\gamma}-\Bigl(1+\frac{3}{\gamma}\Bigr)\varepsilon+1-\frac{1}{\gamma}\right)
  +\frac{3}{\gamma}\varepsilon
  =\frac{(1-\beta)^{-i_s}}{\gamma}-\varepsilon+1-\frac{1}{\gamma},
  \]
  which is exactly the coefficient of $F(\mathbf{o})$ in \eqref{eq:522-UB-b}.

\paragraph{Lower bound:}
From the closed form of $\mathbf{y}(\cdot)$ (Lemma~\ref{lem:5.13} with $i=i_o-1$),
\begin{equation}\label{eq:522-y-form}
\mathbf{y}(i_o-1)\;=\;(\mathbf{1}-\mathbf{z})\odot\bigoplus_{j=1}^{i_o-1}\big(\delta\,\mathbf{x}(j)\big)
\;+\;\mathbf{z}\odot\bigoplus_{j=i_s+1}^{i_o-1}\big(\delta\,\mathbf{x}(j)\big).
\end{equation}
Apply the weakly-DR mixture inequality with masking (Corollary~\ref{cor:weakDR-44}) to
\eqref{eq:522-y-form}. We keep only three nonnegative groups of terms:
- the empty-set term $S=\varnothing$;
- all subsets $S\subseteq[i_s]$;
- all singletons $S=\{j\}$ with $j\in\{i_s+1,\dots,i_o-1\}$.

Using Lemma~\ref{lem:grad-ineq} (weakly-DR gradient bounds) and nonnegativity of $F$, these
three groups yield
\begin{align}
F\big(\mathbf{y}(i_o-1)\oplus \mathbf{o}\big)
&\overset{\text{(c)}}{\ge}\ (1-\beta)^{\,i_o-1}\,F(\mathbf{o})
\ +\ (1-\beta)^{\,i_o-1-i_s}\bigl(1-(1-\beta)^{\,i_s}\bigr)\,
\frac{1}{\gamma}\bigl[F(\mathbf{o})-F(\mathbf{z}\oplus \mathbf{o})\bigr]\notag\\
&\hspace{3cm}
+\ \beta(1-\beta)^{\,i_o-2}\sum_{j=i_s+1}^{i_o-1} F\big(\mathbf{x}(j)\oplus \mathbf{o}\big),
\label{eq:522-LB}
\end{align}
where (c) comes from:
- $S=\varnothing$: gives the term $(1-\beta)^{i_o-1}F(\mathbf{o})$;
- $S\subseteq[i_s], S\neq\varnothing$: combined and bounded below by 
  $\frac{1}{\gamma}\bigl[F(\mathbf{o})-F(\mathbf{z}\oplus \mathbf{o})\bigr]$, with total weight
  $(1-\beta)^{i_o-1-i_s}\bigl(1-(1-\beta)^{i_s}\bigr)$;
- singletons $S=\{j\}$ with $j>i_s$: each has weight $\beta(1-\beta)^{i_o-2}$ and contributes
  $F(\mathbf{x}(j)\oplus \mathbf{o})$.

\paragraph{Comparing the bounds.}
Combining \eqref{eq:522-UB-b} and \eqref{eq:522-LB} and dividing both sides by the common factor
$(1-\beta)^{\,i_o-1}>0$ yields
\[
F(\mathbf{o}) 
+ (1-\beta)^{-i_s}\bigl(1-(1-\beta)^{\,i_s}\bigr)\,
\frac{1}{\gamma}\bigl[F(\mathbf{o})-F(\mathbf{z}\oplus \mathbf{o})\bigr]
+ \frac{\beta}{1-\beta}\sum_{j=i_s+1}^{i_o-1} F\big(\mathbf{x}(j)\oplus \mathbf{o}\big)
\]
\[
\le\ 
\left(\frac{(1-\beta)^{-i_s}}{\gamma}-\varepsilon+1-\frac{1}{\gamma}\right) F(\mathbf{o})
-\ \left(\frac{(1-\beta)^{-i_s}}{\gamma}-\frac{1}{\gamma}-\beta\,(i_o-1-i_s)\right) F(\mathbf{z}\oplus \mathbf{o}).
\]
Rearranging and simplifying the coefficient of $F(\mathbf{o})$ and $F(\mathbf{z}\oplus \mathbf{o})$ gives
\begin{equation}\label{eq:522-sum-ineq}
\beta\sum_{j=i_s+1}^{i_o-1} F\big(\mathbf{x}(j)\oplus \mathbf{o}\big)
\ \le\ 
\bigl(1-(1-\beta)^{\,i_o-1-i_s}\bigr)\,\Bigl[F(\mathbf{z}\oplus \mathbf{o})-\varepsilon\,F(\mathbf{o})\Bigr].
\end{equation}
Moreover,
\[
1-(1-\beta)^{\,i_o-1-i_s}\ \le\ \beta\,(i_o-1-i_s),
\]
so from \eqref{eq:522-sum-ineq} we obtain
\[
\frac{1}{\,i_o-1-i_s\,}\sum_{j=i_s+1}^{i_o-1} F\big(\mathbf{x}(j)\oplus \mathbf{o}\big)
\ \le\ F(\mathbf{z}\oplus \mathbf{o})-\varepsilon\,F(\mathbf{o}).
\]
By averaging, there must exist some $j\in\{i_s+1,\dots,i_o-1\}$ such that
\[
F\big(\mathbf{x}(j)\oplus \mathbf{o}\big)\ \le\ F(\mathbf{z}\oplus \mathbf{o})-\varepsilon\,F(\mathbf{o}),
\]
which proves the lemma.
\end{proof}

\section{Supporting results}

In this section we prove two auxiliary results that are used in the proofs of
Lemma~\ref{lem:weak-521} and Lemma~\ref{lem:weak-522}. Throughout, we use the convention
\[
\bigoplus_{i\in \varnothing} \mathbf{x}(i)\ :=\ \mathbf{0}.
\]

\begin{lemma}\label{lem:weakDR-mixture}
Let $F:[0,1]^n\to\mathbb{R}_{\ge 0}$ be differentiable and $\gamma$-weakly DR-submodular with
$0<\gamma\le 1$. Fix an integer $r\ge 1$, vectors $\mathbf{x}(1),\dots,\mathbf{x}(r)\in[0,1]^n$,
and scalars $p_1,\dots,p_r\in[0,1]$. Define
\[
\beta_\gamma(p)\ :=\ \frac{\gamma^2\,p}{\,1-p+\gamma^2p\,}\qquad(p\in[0,1]).
\]
Then
\[
F\!\left(\bigoplus_{i=1}^r p_i\,\mathbf{x}(i)\right)
 \;\ge\;
\sum_{S\subseteq[r]}
\ \Biggl(\ \prod_{i\in S}\beta_\gamma(p_i)\prod_{i\notin S}\bigl(1-\beta_\gamma(p_i)\bigr)\ \Biggr)
F\!\left(\bigoplus_{i\in S} \mathbf{x}(i)\right).
\]
\end{lemma}

\begin{proof}
We prove the statement by induction on $r$.

\medskip\noindent
\emph{Base case $r=1$.}
Apply Lemma~\ref{lemma:simpe2}(2) with $\mathbf{x}=\mathbf{0}$, $\mathbf{y}=\mathbf{x}(1)$ and
$\lambda=p_1$:
\[
F\bigl(p_1\,\mathbf{x}(1)\bigr)-F(\mathbf{0})
\ \ge\ \frac{\gamma^2p_1}{1-p_1+\gamma^2p_1}\,\bigl(F(\mathbf{x}(1))-F(\mathbf{0})\bigr).
\]
Rearranging gives
\[
F\bigl(p_1\,\mathbf{x}(1)\bigr)
\ \ge\ \beta_\gamma(p_1)\,F(\mathbf{x}(1))+\bigl(1-\beta_\gamma(p_1)\bigr)\,F(\mathbf{0}),
\]
which is exactly the claimed formula when $r=1$ and $S\in\{\varnothing,\{1\}\}$.

\medskip\noindent
\emph{Inductive step.}
Assume the statement holds for some $r-1\ge 1$, and consider $r$.
Define
\[
G_1(\mathbf{x})\ :=\ F\!\left(\mathbf{x}\ \oplus\ \bigoplus_{i=1}^{r-1} p_i\,\mathbf{x}(i)\right),
\qquad
G_2(\mathbf{x})\ :=\ F\!\bigl(\mathbf{x}\oplus \mathbf{x}(r)\bigr).
\]
By Lemma~\ref{lem:weakDR-closure}, both $G_1$ and $G_2$ are nonnegative and $\gamma$-weakly
DR-submodular.

Apply Lemma~\ref{lemma:simpe2}(2) to $G_1$ along the ray $\mathbf{x}(r)$, from base
$\mathbf{0}$ with step $\lambda=p_r$:
\begin{align}
F\!\left(\bigoplus_{i=1}^{r} p_i\,\mathbf{x}(i)\right)
&= G_1\!\bigl(p_r\,\mathbf{x}(r)\bigr)\notag\\
&\ge
\beta_\gamma(p_r)\,G_1\!\bigl(\mathbf{x}(r)\bigr)
\ +\ \bigl(1-\beta_\gamma(p_r)\bigr)\,G_1(\mathbf{0}). \label{eq:mix-step}
\end{align}
By definition of $G_1$ and $G_2$,
\[
G_1(\mathbf{x}(r)) = G_2\!\left(\bigoplus_{i=1}^{r-1} p_i\,\mathbf{x}(i)\right),
\qquad
G_1(\mathbf{0})=F\!\left(\bigoplus_{i=1}^{r-1} p_i\,\mathbf{x}(i)\right).
\]
Substituting into \eqref{eq:mix-step} gives
\[
F\!\left(\bigoplus_{i=1}^{r} p_i\,\mathbf{x}(i)\right)
\ \ge\
\beta_\gamma(p_r)\,G_2\!\left(\bigoplus_{i=1}^{r-1} p_i\,\mathbf{x}(i)\right)
\ +\
\bigl(1-\beta_\gamma(p_r)\bigr)\,F\!\left(\bigoplus_{i=1}^{r-1} p_i\,\mathbf{x}(i)\right).
\]

Now apply the induction hypothesis to $G_2$ (with the $r-1$ vectors
$\mathbf{x}(1),\dots,\mathbf{x}(r-1)$ and weights $p_1,\dots,p_{r-1}$) and to $F$:
\begin{align*}
G_2\!\left(\bigoplus_{i=1}^{r-1} p_i\,\mathbf{x}(i)\right)
&\ge
\sum_{S\subseteq[r-1]}
\Biggl(\prod_{i\in S}\beta_\gamma(p_i) \prod_{i\notin S}\bigl(1-\beta_\gamma(p_i)\bigr)\Biggr)
\,G_2\!\left(\bigoplus_{i\in S}\mathbf{x}(i)\right),\\
F\!\left(\bigoplus_{i=1}^{r-1} p_i\,\mathbf{x}(i)\right)
&\ge
\sum_{S\subseteq[r-1]}
\Biggl(\prod_{i\in S}\beta_\gamma(p_i)\prod_{i\notin S}\bigl(1-\beta_\gamma(p_i)\bigr)\Biggr)
\,F\!\left(\bigoplus_{i\in S}\mathbf{x}(i)\right).
\end{align*}
Note that
\[
G_2\!\left(\bigoplus_{i\in S}\mathbf{x}(i)\right)
=F\!\left(\bigoplus_{i\in S\cup\{r\}}\mathbf{x}(i)\right).
\]
Plugging these expansions into the right-hand side of \eqref{eq:mix-step}, we get a sum over
all $S\subseteq[r-1]$ of:
\begin{itemize}
  \item terms with factor $\beta_\gamma(p_r)$ and value $F\bigl(\bigoplus_{i\in S\cup\{r\}}\mathbf{x}(i)\bigr)$;
  \item terms with factor $(1-\beta_\gamma(p_r))$ and value $F\bigl(\bigoplus_{i\in S}\mathbf{x}(i)\bigr)$.
\end{itemize}
Reindex the first group by $S' = S\cup\{r\}$ (so $S'\subseteq[r]$ with $r\in S'$) and the second
group by $S'=S$ (so $S'\subseteq[r]$ with $r\notin S'$). The coefficient in front of each
$F\bigl(\bigoplus_{i\in S'}\mathbf{x}(i)\bigr)$ is then
\[
\prod_{i\in S'}\beta_\gamma(p_i)\prod_{i\notin S'}\bigl(1-\beta_\gamma(p_i)\bigr),
\]
which yields exactly the desired mixture inequality over all $S'\subseteq[r]$. This completes
the induction.
\end{proof}

\begin{corollary}\label{cor:weakDR-44}
Let $F:[0,1]^N\to\mathbb{R}_{\ge 0}$ be nonnegative and $\gamma$-weakly DR-submodular with $0<\gamma\le 1$.
Fix integers $r,h\ge 1$. For each $i\in[h]$, let $\mathbf{x}^{(i)}(1),\dots,\mathbf{x}^{(i)}(r)\in[0,1]^N$ and
$\mathbf{b}(i)\in[0,1]^N$ satisfy $\sum_{i=1}^h \mathbf{b}(i)=\mathbf{1}$ (coordinatewise), and let
$p_1,\dots,p_r\in[0,1]$. Define
\[
\beta_\gamma(p)\ :=\ \frac{\gamma^2 p}{1-p+\gamma^2 p}\qquad(p\in[0,1]).
\]
Then
\begin{align*}
F\!\left(\sum_{i=1}^{h} \mathbf{b}(i)\ \odot\  \bigoplus_{j=1}^{r} \bigl(p_j\,\mathbf{x}^{(i)}(j)\bigr)\right)
\ \ge\
\sum_{S\subseteq[r]}
\left(\ \prod_{j\in S}\beta_\gamma(p_j)\prod_{j\notin S}\bigl(1-\beta_\gamma(p_j) \bigr)\right)
F\!\left( \sum_{i=1}^{h} \mathbf{b}(i)\ \odot\ \bigoplus_{j\in S} \mathbf{x}^{(i)}(j)\right).
\end{align*}
\end{corollary}

\begin{proof}
Define $G:[0,1]^{Nh}\to\mathbb{R}_{\ge 0}$ on $h$ blocks by
\[
G\big(\mathbf{c}(1),\ldots,\mathbf{c}(h)\big)\ :=\ F\!\left(\sum_{i=1}^{h} \mathbf{b}(i)\odot \mathbf{c}(i)\right).
\]
Since $F\ge 0$ and the map
\(
(\mathbf{c}(1),\dots,\mathbf{c}(h))\mapsto \sum_{i=1}^h \mathbf{b}(i)\odot \mathbf{c}(i)
\)
is coordinatewise nonnegative and linear with $\sum_i \mathbf{b}(i)=\mathbf{1}$, it follows from
Lemma~\ref{lem:weakDR-closure} (applied blockwise) that $G$ is also nonnegative and $\gamma$-weakly
DR-submodular on $[0,1]^{Nh}$.

For each $j\in[r]$, define the block vector
\[
\mathbf{x}(j)\ :=\ \big(\mathbf{x}^{(1)}(j),\,\mathbf{x}^{(2)}(j),\,\ldots,\,\mathbf{x}^{(h)}(j)\big)\ \in\ [0,1]^{Nh}.
\]
Then
\[
G\!\left(\bigoplus_{j=1}^{r} p_j\,\mathbf{x}(j)\right)
\ =\
F\!\left(\sum_{i=1}^{h} \mathbf{b}(i)\odot \bigoplus_{j=1}^{r} p_j\,\mathbf{x}^{(i)}(j)\right),
\]
which is exactly the left-hand side of the desired inequality.

Apply Lemma~\ref{lem:weakDR-mixture} to $G$ with inputs
$\mathbf{x}(1),\dots,\mathbf{x}(r)$ and coefficients $p_1,\dots,p_r$:
\[
G\!\left(\bigoplus_{j=1}^{r} p_j\,\mathbf{x}(j)\right)
 \;\ge\;
\sum_{S\subseteq[r]}
\ \Biggl(\ \prod_{j\in S}\beta_\gamma(p_j)\prod_{j\notin S}\bigl(1-\beta_\gamma(p_j)\bigr)\ \Biggr)
G\!\left(\bigoplus_{j\in S} \mathbf{x}(j)\right).
\]
Finally, note that for every $S\subseteq[r]$,
\[
G\!\left(\bigoplus_{j\in S} \mathbf{x}(j)\right)
\;=\;
F\!\left(\sum_{i=1}^{h} \mathbf{b}(i)\odot \bigoplus_{j\in S} \mathbf{x}^{(i)}(j)\right),
\]
so substituting this identity into the right-hand side above gives exactly the claimed inequality.
\end{proof}

\end{document}